 \documentclass[opre,sglanonrev]{informs5}
\RequirePackage{tgtermes}
\RequirePackage{newtxtext}
\RequirePackage{newtxmath}
\RequirePackage{bm}
\RequirePackage{endnotes}

\OneAndAHalfSpacedXII 

\usepackage{algorithm}
\usepackage{algpseudocode}
\usepackage{tikz}

\usepackage{natbib}
 \bibpunct[, ]{(}{)}{,}{a}{}{,}%
 \def\bibfont{\small}%
\EquationsNumberedThrough    

\TheoremsNumberedThrough     
\ECRepeatTheorems  %

\MANUSCRIPTNO{MOOR-0001-2024.00}

\usepackage{appendix}
\usepackage{graphicx}
\usepackage{subcaption}
\usepackage{wrapfig}
\usepackage{float}
\usepackage{multirow}
\usepackage{booktabs} 
\usepackage{multicol} 
\usepackage{makecell}
\usepackage[pdftex,colorlinks=true,urlcolor=blue,citecolor=black,anchorcolor=black,linkcolor=black]{hyperref}
\usepackage{enumitem}

\begin{document}


 \RUNAUTHOR{Li et al.} 

 \RUNTITLE{Optimal Low-Rank Stochastic Gradient Estimation for LLM Training}

\TITLE{Optimal Low-Rank Stochastic Gradient Estimation for LLM Training }

\ARTICLEAUTHORS{%

\AUTHOR{Zehao Li}
\AFF{Guanghua School of Management, Peking University, Beijing, 100871, China, zehaoli@stu.pku.edu.cn  } 

\AUTHOR{Tao Ren}
\AFF{Guanghua School of Management, Peking University, Beijing, 100871, China, rtkenny@stu.pku.edu.cn  }

\AUTHOR{Zishi Zhang}
\AFF{Guanghua School of Management, Peking University, Beijing, 100871, China, zishizhang@stu.pku.edu.cn  }

\AUTHOR{Xi Chen}
\AFF{Stern School of Business, New York University, New York, NY 10012, USA, xc13@stern.nyu.edu}

\AUTHOR{Yijie Peng}
\AFF{Guanghua School of Management, Peking University, Beijing, 100871, China, pengyijie@pku.edu.cn  }
} 

\ABSTRACT{
Large language model (LLM) training is often bottlenecked by memory consumption and the noise of stochastic gradients in extremely high-dimensional parameter spaces. Motivated by empirical evidence that many LLM gradient matrices are
effectively low-rank during training, we present an unbiased, memory-efficient, low-rank matrix estimator with the lowest variance that is
applicable across common stochastic gradient estimation paradigms.
The core idea is to project a high-dimensional stochastic gradient estimator onto a random low-dimensional subspace and lift it back, reducing memory while keeping the estimator unbiased and controlling mean-squared error via an optimally designed projection distribution, including Haar--Stiefel projections. 
The projection distribution is derived by solving a constrained functional optimization problem, yielding an optimal random projector that guides algorithm design. Empirically, the resulting low-rank gradient estimators deliver both practical memory savings and improved training behavior. In RoBERTa-large fine-tuning, our method attains the lowest peak GPU memory among compared methods (e.g., 3.83GB versus 16.7GB for full BP) while remaining competitive in accuracy; in autoregressive LLM pretraining (LLaMA-20M/60M/100M), our method outperforms the traditional methods, supporting the benefit of the proposed optimal projection strategy. 
}%


\KEYWORDS{large language model training, stochastic gradient estimation, low-rank matrix, constrained distributional optimization} 

\maketitle

%


\section{Introduction}

Large language models (LLMs) are large-scale neural networks pretrained on massive text corpora and have demonstrated strong capabilities across a wide range of tasks, including reasoning, coding, and decision support. With the growing use of instruction tuning and domain-specific adaptation, LLMs are increasingly becoming a core foundation for intelligent systems in applications ranging from everyday communication to industrial optimization \citep{huang2025orlm,dai2025assured,liang2026llm}. Their development typically involves two stages. The first is pretraining, in which a model is trained on broad unlabeled data to acquire general-purpose representations. The second is fine-tuning, where the pretrained model is further adapted to task- or domain-specific data so as to better serve a particular objective.

LLM training is inherently a stochastic optimization problem: parameter updates are driven by randomness from minibatch sampling and environment interaction. In both pretraining and fine-tuning, gradient-based optimization remains the standard approach for updating model parameters. Consequently, how to obtain stochastic gradient estimators becomes the central problem. In simulation optimization literature, classical stochastic gradient estimation methods include infinitesimal perturbation analysis (IPA) and likelihood ratio (LR) methods \citep{heidelberger1988convergence,ho1983infinitesimal,glynn1990likelihood,l1990unified,reiman1989sensitivity}. At the LLM scale, gradients are naturally matrix-valued, aligning with the weight-matrix structure of neural networks. The training bottleneck is driven by entrywise noise aggregation and optimizer-state memory across massive weight blocks, which introduces both computational and statistical challenges.

On the computational side, the memory requirements are substantial. For instance, pretraining a LLaMA 7B model from scratch with a single batch size requires at least 58 GB of memory, which makes training infeasible on consumer-level GPUs such as
 NVIDIA RTX 4090 with 24 GB of memory \citep{zhao2024galore}. To address this issue, parameter-efficient training methods have been developed, which aim to adapt large pretrained models by updating only a small subset of parameters while keeping the backbone frozen. A prominent example is low-rank adaptation (LoRA) and its variants \citep{hu2022lora,dettmers2023qlora,ye2025lola,xu2026parameter}, which inject trainable low-rank matrices into selected weight modules so that the update is constrained to a low-dimensional subspace, thereby reducing the number of parameters that are updated and the associated optimizer-state memory.

Furthermore, the variance of the gradient estimators becomes an even more critical challenge in such stochastic optimization problems. While Monte Carlo gradient estimation and variance reduction techniques have been extensively studied in the stochastic optimization literature \citep{vihola2018unbiased,cui2020variance,ye2025unified}, existing methods are predominantly developed for vector-valued gradients and do not explicitly exploit the low-rank matrix structure that often arises in LLM training. In contrast, LLM training involves optimization over high-dimensional neural weight matrices, where the variance accumulated across all entries can be substantial and interacts nontrivially with the matrix rank. 
To address this gap, our goal is to develop methods that simultaneously reduce memory overhead and control gradient variance, enabling scalable and stable training for high-dimensional stochastic neural optimization tasks.

Our motivations are twofold. A key empirical observation in recent literature is that neural network gradients tend to lie in a low-rank subspace during training. For example, \cite{zhao2024galore} identify this phenomenon as a well-established property across various architectures and demonstrate that in Transformer models, the gradient of a value matrix with dimensions $1024\times1024 $ typically exhibits only around 10 dominant eigenvalues—implying that the full gradient can be well-approximated by a low-rank matrix. A low effective rank suggests that the descent signal concentrates in a small subspace; a projection that respects this geometry can filter irrelevant noise directions, improving both memory and estimator MSE. Another motivation comes from the recent progress in parameter-efficient training methods, as mentioned before \citep{xu2026parameter}. However, existing methods mainly reduce the number of parameters that are updated. From the perspective of stochastic gradient estimation, we focus on reducing the cost and noise of gradient estimation itself for large weight matrices.

Motivated by these findings, we propose a new low-rank stochastic gradient estimation method for LLM training based on classical IPA and LR methods. Our method exploits the intrinsic low-rank structure of neural gradients to simultaneously reduce memory consumption and estimator variance—two critical challenges in large-scale stochastic optimization. Our approach randomly projects the full-rank stochastic gradient estimator onto a lower-dimensional subspace, significantly reducing computational overhead and memory consumption. At each iteration we compute the directional derivative along a random \(r\)-dimensional subspace \(V\in\mathbb R^{n\times r}\) and lift it back, reducing computation and memory from \(O(mn)\) to \(O\!\bigl(r(m+n)\bigr)\) with \(r\ll\min\{m,n\}\).  Despite the dimensionality reduction, the resulting low-rank gradient estimator effectively balances estimation accuracy with computational efficiency. 

 This method falls within the broader class of random subspace projection techniques, which perform optimization not in the full parameter space, but within a randomly selected low-dimensional subspace at each iteration. Our setting, however, differs from prior subspace methods in three important ways. First, most literature is developed for vector-valued problems. It does not exploit the low-rank structure of matrix-valued gradients that is pervasive in modern LLM training \citep{spall1992multivariate,spall1997one,ghadimi2013stochastic,duchi2015optimal,ji2019improved,kozak2023zeroth}. Second, even in the matrix-oriented subspace literature, the projection mechanism is usually taken as given: one samples a random sketch or projector according to a prescribed rule and then analyzes the resulting estimator or algorithm. The resulting analysis is not cast as a stochastic optimization problem over projection distributions \citep{vu2018random,gutman2023coordinate,ding2025new}. 
 Third, the projection rule is almost always chosen heuristically, typically using Gaussian sampling, and the question of how to optimally design the distribution over projection matrices is largely left unexplored \citep{ghadimi2013stochastic,nesterov2017random,chen2024enhancing,chen2025memory}.
 
 A central challenge in random subspace projection is how to choose an efficient distribution over projection matrices, since this choice directly affects the statistical efficiency of the resulting gradient estimator. To address this issue in a principled way, we analyze the mean squared error (MSE) of our low-rank stochastic gradient estimator and formulate the selection of the projection distribution as a constrained distributional optimization problem. The admissible projections are required to satisfy two structural conditions: they must produce a low-rank estimator and preserve weak unbiasedness, meaning that the estimator matches the true gradient up to a positive scalar factor. The low-rank structure is also essential for reducing memory cost. Importantly, the resulting MSE depends on two coupled sources of randomness: the stochasticity of data sampling and the additional randomness induced by projection. As a result, the projection distribution should be treated as a decision variable, rather than a fixed implementation choice.

We study this constrained design problem in two cases. In the instance-independent case, where no reliable information about data-induced gradient variability is available, we establish a universal optimality principle: among all admissible random projectors, the smallest achievable MSE of the information-agnostic upper bound is attained by balanced subspace distributions whose projectors are isotropic on average, meaning that they distribute projection energy evenly across parameter directions. This result provides a sharp worst-case benchmark and implies that symmetric distributions such as Haar–Stiefel sampling and uniform coordinate subspaces are not merely heuristics but are provably optimal in the information-agnostic setting.

In the instance-dependent case, where partial information about gradient variability can be estimated, we go further and solve the optimal projection design problem. We explicitly characterize an optimal sampling distribution that adapts to the eigen-structure of the underlying gradient noise, preferentially allocating projection budget to directions with larger signal-to-noise impact. This yields a principled, data-aware subspace sampling rule that can strictly reduce estimator MSE relative to standard isotropic subspace sampling, and in favorable low-effective-rank situations can match the accuracy of full-gradient estimation while using substantially less memory.

In summary, our paper makes the following key contributions.
\begin{itemize}
\item From the viewpoint of stochastic gradient estimation, we build a (weakly) unbiased low-rank stochastic gradient estimator by leveraging classical IPA and LR estimators. We then design a memory-efficient randomized subspace projection algorithm for LLM training, where gradients/updates are computed in a low-dimensional subspace and mapped back to the full parameter space. To further amortize the projection cost and improve stability, we introduce a lazy-update mechanism that reuses one sampled subspace for multiple inner steps.

\item We formulate a constrained functional optimization problem that minimizes the MSE of our low-rank gradient estimator over admissible random projectors. In the instance-independent setting, we optimize a tractable upper bound and characterize optimal random projectors that achieve weak unbiasedness and the smallest achievable MSE under this robust design. In the instance-dependent setting, we derive a sharper instance-dependent characterization, solve the optimal distribution of the random projection, and provide a practical sampling procedure that yields an unbiased, memory-efficient estimator with the lowest MSE.

\item Extensive experiments provide empirical support for both the theoretical insights and the practical effectiveness of our framework across representative LLM training regimes. In LR-based fine-tuning of RoBERTa-large on multiple benchmarks, our structured subspace designs, especially the Stiefel sampler, achieve strong and often superior accuracy over vanilla low-rank baselines while retaining the memory-efficiency advantages of low-rank training. In IPA-based autoregressive pretraining of LLaMA models at multiple scales, the proposed optimal projector consistently outperforms Gaussian subspace sampling in both training and evaluation loss. 

\end{itemize}

\smallskip
The rest of the paper is organized as follows. Section \ref{sec2} reviews the related literature, and Section \ref{sec3} introduces the necessary preliminaries. Section \ref{sec4} presents the proposed low-rank gradient estimator and the associated optimization algorithm. Section \ref{sec5} develops the theoretical analysis. Section \ref{sec6} reports the numerical results, and Section \ref{sec7} concludes. All proofs are deferred to the Appendix.

\section{Literature Review}\label{sec2}
Our work contributes to three streams of literature: stochastic gradient estimation, randomized subspace optimization, and parameter-efficient training methods.

\textbf{Stochastic gradient estimation.}
 Gradient estimation with only noisy function evaluation arises ubiquitously in stochastic optimization problems such as quantile optimization \citep{hu2022stochastic,hu2024quantile}, distributionally robust analysis \citep{fan2020distributionally,lam2024distributionally,wang2025sinkhorn}, complex systems optimization \citep{xu2023gradient},  convex discrete optimization \citep{zhang2023gradient}, adaptive data stream \citep{che2026stochastic}, simulation-based inference \citep{peng2020maximum,li2025new}, and inventory optimization \citep{wang2023large,chungunbalanced}. Classical unbiased gradient estimators, such as IPA, LR methods, and other variants \citep{ho1983infinitesimal,glynn1990likelihood,hong2009estimating,heidergott2010gradient,peng2018new}, are widely used in stochastic optimization due to their general applicability and flexibility. Reviews of gradient estimation techniques can be found in \cite{FU2006575,fu2015stochastic} and \cite{mohamed2020monte}. The above gradient estimators are all scalar- or vector-valued. In this work, we focus on matrix-valued gradient estimators under the specific setting of neural network-based optimization.
 
 Furthermore, bias and variance are typically capable of measuring the properties of an estimator. Conversely, we often minimize the variance or the MSE to design an optimal estimator. For example, \cite{rhee2015unbiased} construct finite variance unbiased estimators for a general class of SDEs with an optimal distribution for randomization.
 \cite{cui2020variance} analyze the variance of single-run unbiased stochastic gradient estimators, as well as offering insights on finding an estimator with lower variance. Variance reduction techniques such as multilevel Monte Carlo \citep{giles2008multilevel,rosenbaum2017multilevel,vihola2018unbiased,hu2023contextual} and optimal importance sampling have also been studied to design a gradient estimator and algorithm \citep{pan2020adaptive,he2024adaptive,aolaritei2025stochastic}. Bias and variance of the gradient estimators also affect the convergence rate of the stochastic optimization algorithms \citep{karimi2019non,li2024beyond,hu2025convergence}.

\textbf{Randomized subspace optimization.}
Following another line of literature, random subspace optimization has also been widely studied across different problems. An important subclass emerges in zeroth-order (ZO) optimization, where gradient information is unavailable and must be inferred from function evaluations \citep{ye2025unified,lam2024distributionally,wu2022joint,chen2020unbiased}. The classical Kiefer–Wolfowitz scheme \citep{harold1997stochastic} estimates partial derivatives by perturbing each coordinate individually, leading to computational overhead proportional to the input dimension. To alleviate this burden, more recent approaches adopt synchronized perturbations along randomly chosen directions, enabling simultaneous estimation of all gradient components from a single or a few function evaluations. These perturbation vectors are typically sampled from identically and independently distributed Gaussian due to its ease of sampling \citep{nesterov2017random,ghadimi2013stochastic}. Other effective and common methods are the simultaneous perturbation stochastic
 approximation \citep{spall1992multivariate,spall1997one}, the random coordinate method \citep{ji2019improved,lu2018accelerating}, or the uniform random vector \citep{duchi2015optimal}. 

In this sense, ZO methods can be naturally viewed as a special case of randomized subspace projections with rank one, where stochastic smoothing is applied over one-dimensional subspaces \citep{cai2022zeroth}. Extensions to higher-rank subspaces have also been studied to exploit structural information in a variety of optimization settings, including linear programming via random projections \citep{vu2018random}, proximal gradient methods with adaptive subspace sampling \citep{grishchenko2021proximal}, coordinate-descent-type methods on Riemannian manifolds \citep{gutman2023coordinate}, and orthogonally randomized subspace algorithms \citep{kozak2021stochastic,kozak2023zeroth}; related low-rank structure has also been explored in semidefinite programming \citep{ding2025new}. These works focus on deterministic optimization problems, where random subspace projection is primarily used for dimensionality reduction and algorithmic simplification. Moreover, they typically lack a variance-aware perspective and rely on heuristic subspace distributions that are not optimized for the statistical structure of matrix-valued gradient estimators, which is a central focus of our work.

\textbf{Parameter-efficient training methods. }
In both large-scale pretraining and task-specific fine-tuning, a common strategy to reduce training memory is to restrict learning to a small set of trainable parameters. LoRA is a representative approach that freezes the backbone and parameterizes weight updates by a low-rank factorization \citep{hu2022lora}. 
LoRA does not optimize the full weight matrix directly; instead, it introduces a low-rank adaptation layer to fine-tune large models with significantly fewer parameters. While this yields memory savings, it restricts optimization to a low-dimensional subspace and thus limits the reachable update space relative to full-parameter training  \citep{chen2024enhancing}. A rich family of LoRA derivatives
has been developed to make the fine-tuning more efficient in downstream tasks \citep{zhang2023adalora,dettmers2023qlora,ye2025lola,xu2026parameter}. In contrast, GaLore \citep{zhao2024galore} and its variant \citep{he2024subspace} leverage the empirical low-rank structure of neural gradients to reduce memory usage. However, they first compute the full gradient and only then project it onto a low-rank subspace via PCA, meaning that they do not reduce the cost of gradient estimation itself. \cite{chen2024enhancing,chen2025memory} offer a different perspective. Their methods can be interpreted as randomized subspace optimization in gradient-free neural network training, which fits into our framework as a special case. More broadly, other recent studies have also investigated related questions in LLM training and inference \citep{liu2024uncertainty,wang2024understanding,li2025agentgit,liang2026llm}.

\paragraph{Notations.}
We introduce essential notations used throughout this paper. For a positive integer $n$, we write
$[n]:=\{1,2,\ldots,n\}$.
For a finite set $J$, we use $|J|$ to denote its cardinality, and $\mathbf{1}_{\{\cdot\}}$ for the indicator function.
For a matrix $A\in\mathbb{R}^{m\times n}$, we write $A^\top$ for its transpose and $\mathrm{tr}(A)$ for its trace.
We use $\langle A,B\rangle := \mathrm{tr}(A^\top B)$ to denote the Frobenius inner product, and
$\|A\|_F := \sqrt{\langle A,A\rangle}=\sqrt{\mathrm{tr}(A^\top A)}$ for the Frobenius norm.
When needed, $\|A\|_2$ denotes the spectral (operator) norm, i.e., the largest singular value of $A$.
We use $I_n$ for the $n\times n$ identity matrix, and for a positive semi-definite matrix $A\succeq 0$ we write
$A=Q\mathrm{diag}(\lambda_1,\ldots,\lambda_n)Q^\top$ for its spectral decomposition with $\lambda_1\ge\cdots\ge\lambda_n\ge 0$.
For a random subset $J\subseteq [n]$, the notation $|J|=r$ means that $J$ has exactly $r$ elements.

\section{Problem Formulation: Gradient Estimation in LLM Training}\label{sec3}

In this section, we provide a unified perspective on prevalent LLM training paradigms through the lens of gradient estimation. Specifically, we formulate these paradigms within a stochastic optimization framework and classify existing approaches into two families—IPA and LR—according to how the gradients are estimated. 
We also discuss the respective strengths and limitations of these two families in the context of LLM training, thereby motivating the development of our new low-rank approach in the next section, which can be applied to improve both families.

Specifically, LLM training can be cast as the following stochastic
optimization problem:
\begin{equation}\label{eq:llm-train}
\min_{\Theta\in\mathbb{R}^{m\times n}:\ \mathrm{NN}_\Theta\in\mathcal{H}}
\ \mathbb{E}_{\xi\sim p(\cdot;\Theta)}\big[F(\xi,\Theta)\big],
\end{equation}
where $\mathrm{NN}_\Theta$ denotes a neural network (in modern LLMs, typically a Transformer) parameterized by $\Theta\in\mathbb{R}^{m\times n}$,
$\mathcal{H}$ is a hypothesis class of admissible networks, $\xi$ collects the randomness in training
(e.g., minibatch sampling, dropout, or environment randomness),
and $F(\xi,\Theta)$ denotes the training loss induced by the model $\mathrm{NN}_\Theta$ under $\xi$.
This formulation is general in that the randomness $\xi\sim p(\cdot;\Theta)$ may depend on the trainable parameter $\Theta$, resulting in a parameter-dependent sampling distribution. In practice, the parameter $\Theta$ is a high-dimensional tensor that aggregates all trainable weights across all layers of the neural network. Without loss of generality, we focus on a single matrix block (e.g., one layer) and treat the remaining parameters as fixed throughout this work.

A standard way to solve \eqref{eq:llm-train} is gradient-based optimization, in
which {the central challenge is to construct an estimator $\hat{g}(\Theta)$ of the true gradient $$g(\Theta):=\nabla_\Theta \mathbb{E}_{\xi\sim p(\cdot;\Theta)}\big[F(\xi,\Theta)\big].$$ 

We begin by introducing the first gradient estimation family, namely \textit{infinitesimal perturbation analysis} (IPA).
When the randomness distribution does not depend on $\Theta$ (i.e., $p(\xi;\Theta)=p(\xi)$) and
$F(\xi,\Theta)$ is differentiable in $\Theta$, the IPA estimator, given by
\begin{equation}\label{eq:ipa:1}
    \widehat g_{\mathrm{IPA}}(\xi,\Theta)=\nabla_\Theta F(\xi,\Theta), 
\end{equation}
is unbiased under standard dominated-convergence type conditions \citep{FU2006575,fu2015stochastic}, i.e., $ \mathbb{E}[\widehat g_{\mathrm{IPA}}(\xi,\Theta)]=\nabla_\Theta \mathbb{E}[F(\xi,\Theta)]$.
In neural network training, backpropagation (BP) provides an efficient computational implementation of the IPA estimator, which computes the pathwise gradient 
$\nabla_\Theta F(\xi,\Theta)$ by recursively applying the chain rule along the computational graph of $\mathrm{NN}_\Theta$.
Stochastic gradient descent (SGD) combined with BP constitutes the dominant paradigm in modern LLM training and falls within the IPA family \citep{peng2022new}. In these settings, the randomness $\xi$—for example, from minibatch sampling or dropout—is exogenous to the high-dimensional $\Theta$ and is typically of relatively low dimension. 
Consequently, the IPA gradient estimator 
$\nabla_\Theta F(\xi,\Theta)$ typically enjoys relatively low variance \citep{cui2020variance}.
The main limitation of IPA-based methods lies in their substantial memory requirement, as computing $\nabla_\Theta F(\xi,\Theta)$ via BP requires storing the entire computational path of the neural network \citep{ren2025zeroth}. 
In addition, IPA is not directly applicable when the sampling distribution depends on $\Theta$ or involves discrete randomness, unless such randomness admits a suitable reparameterization \citep{mohamed2020monte}. We illustrate this with the following example.

\begin{example}[IPA gradient estimator for a one-layer ReLU network]\label{ex:ipa}
Consider training a one-layer network $\mathrm{NN}_\Theta(x):=\psi(\Theta^\top x),$
with ReLU activation function $\psi(z):=\max\{0,z\}$ and model parameter $\Theta\in\mathbb R^{m\times n}$.
The training sample $\xi=(x,y)\in\mathbb R^m\times\mathbb R^n$ is drawn from a distribution
$p(\xi)$.
The training loss is squared loss $F(\xi,\Theta):=\frac12\big\|\mathrm{NN}_\Theta(x)-y\big\|^2$.
Then, on $\{\Theta^\top x\neq 0\}$, the IPA gradient estimator defined in \eqref{eq:ipa:1} is given by
$$\widehat g_{\mathrm{IPA}}(\xi,\Theta)
=\nabla_\Theta F(\xi,\Theta)
= x\Big(\big(\psi(\Theta^\top x)-y\big)\odot \mathbf 1\{\Theta^\top x>0\}\Big)^\top,$$
where $\odot$ denotes the elementwise product of two vectors in $\mathbb R^n$.
\end{example}

A second classical family of gradient estimators is the \emph{likelihood-ratio} (LR) estimator.
It becomes relevant in LLM training whenever the distribution of the randomness $\xi\sim p(\cdot;\Theta)$ depends on $\Theta$. If $p(\xi;\Theta)$ is differentiable in $\Theta$ and the loss function does not depend on $\Theta$
explicitly, i.e., $F(\xi,\Theta)=F(\xi)$, the LR estimator is given by
\begin{equation}\label{eq:lr:1}
    \widehat g_{\mathrm{LR}}(\xi,\Theta)=(F(\xi)-b)\,\nabla_\Theta\log p(\xi;\Theta),
\end{equation}
which is unbiased, with an optional constant $b$ for variance reduction \citep{greensmith2004variance,mohamed2020monte}. LR can also be used when $\Theta$ appears in $F(\xi,\Theta)$ rather than in $p(\xi;\Theta)$, by a
suitable change of variables that makes $\Theta$ enter $p(\xi;\Theta)$.
Many methods in the LLM training literature belong to the LR family.
For example, in reinforcement-learning-based LLM post-training, the LLM itself generates the random sample
$\xi$ (such as a token sequence), so that its sampling distribution depends on the model parameter
$\Theta$. In the reinforcement learning literature, the LR estimator \eqref{eq:lr:1} is commonly
referred to as the {policy gradient} (or {REINFORCE}) estimator \citep{lin2025reusing}. Another LR-family method commonly used in LLM training is zeroth-order (ZO) optimization \citep{ye2025unified,wu2022joint}.
ZO methods inject a random perturbation (e.g., Gaussian) into $\Theta$ and,
by a change-of-variables argument, induce a $\Theta$-dependent density that again
enables an LR estimator. In this way, ZO constructs a stochastic gradient estimator
using only function evaluations of the black-box objective, without requiring
backpropagation; see Example~\ref{ex:zo}.


\begin{example}[Special case of LR Gradient Estimator: ZO optimization]\label{ex:zo}
Consider the problem \eqref{eq:llm-train}.
ZO methods introduce an auxiliary random perturbation $\sigma Z$ into the parameter
$\Theta$, where $Z\sim\mathcal N(0,I_{mn})$ and $\sigma>0$ is a prescribed
perturbation scale. The objective can then be written as $$\mathbb E_{\xi,Z}\big[F(\xi,\Theta+\sigma Z)\big]=\mathbb E_{\xi,Z'}\big[F(\xi,Z')\big],$$
where $Z'=\Theta+\sigma Z \sim \mathcal N(\Theta,\sigma^2 I_{mn})$.
Under this change of variables, $\Theta$ enters through the randomness,
and the LR estimator \eqref{eq:lr:1} can be applied, which yields
the standard one-point
estimator
\[
\widehat g_{\mathrm{LR\text{-}1pt}}(\xi,Z;\Theta)
=
F(\xi,\Theta+\sigma Z)\,\frac{Z}{\sigma}.
\]
In practice, an antithetic two-point form is often used to reduce variance,
leading to
\[
\widehat g_{\mathrm{LR\text{-}2pt}}(\xi,Z;\Theta)
=
\frac{F(\xi,\Theta+\sigma Z)-F(\xi,\Theta-\sigma Z)}{2\sigma} Z.
\]
Two canonical instances are finite-difference methods and
Simultaneous Perturbation Stochastic Approximation (SPSA), obtained by choosing
different distributions for $Z$; see
\cite{spall1992multivariate,spall1997one} and \cite{scheinberg2022finite} for further details.
\end{example}

In general, unlike IPA, LR does not require differentiating through the model or the sampling path, and thus naturally applies to discrete random variables.
Moreover, since LR methods, especially ZO methods, do not require storing gradients or intermediate activations for BP, it is relatively
memory-efficient. Its main drawback, however, is typically a higher estimator variance.

\section{Low-Rank Gradient Estimation with Lazy Updates}\label{sec4}

As discussed in the previous section, we introduced two major
gradient-estimation paradigms used in LLM training: IPA and LR.
In this section, we propose a low-rank gradient estimator,
together with a lazy-update gradient descent algorithm, applicable to both the IPA and LR families.
Our goal is to substantially reduce memory usage while
maintaining low estimation variance.
The key motivation stems from a widely observed phenomenon
in LLM training, as optimization progresses, the gradient matrices $\nabla_\Theta F(\xi,\Theta)$ often become effectively low-rank
\citep{zhao2024galore}.
As a result, despite their high ambient dimension, full gradient matrices can often be well approximated by
low-rank representations.
Motivated by this observation, it is unnecessary to estimate
or store every entry of the full gradient matrix.
Instead, one can construct low-rank gradient estimators that
capture the dominant descent directions while substantially
reducing both memory requirements and estimation variance.
Importantly, our method does not rely on any explicit low-rank
assumption on the true gradient.

\subsection{Design of Low-Rank Stochastic Gradient Estimators}
Our goal is to build a valid gradient estimator while avoiding the cost of forming and storing a full gradient matrix. Validity is
captured by unbiasedness conditions, which ensure that the estimator points in the correct descent
direction on average.
To make this precise, we distinguish two levels of unbiasedness.

\begin{definition}\label{def1}
A gradient estimator $\hat{g}(\Theta)$ is called \textbf{weakly unbiased} if $\mathbb{E}[\hat{g}(\Theta)] = c \cdot g(\Theta)$ for some constant scalar $c > 0$; it is called \textbf{strongly unbiased} if $c = 1$.
\end{definition}

Weak unbiasedness means that the estimator differs from the true gradient only by a positive scalar
factor $c$ in expectation. In high dimensions, this is often sufficient because it preserves the
direction of the gradient: the expected update still points along true gradient $ g(\Theta)$, up to
rescaling. We introduce this notion because LLM training typically cares about descent
directions rather than absolute magnitudes, and a constant factor can usually be absorbed into the
stepsize or learned by an adaptive optimizer. Strong unbiasedness is the special case $c=1$, so
weak unbiasedness strictly generalizes the classical notion. Under standard regularity conditions,
the classical IPA estimator \eqref{eq:ipa:1} and the classical LR estimator \eqref{eq:lr:1} are
strongly unbiased \citep{mohamed2020monte}.

We first propose two new low-rank gradient estimators, corresponding to classical IPA and LR, respectively. The key idea is to avoid forming the full
$m\times n$ gradient by optimizing only within a randomly chosen rank-$r$ subspace.
Specifically, as illustrated in the first step of Figure~\ref{fig:lowrank_schematic}, we draw a random projection matrix $V\in\mathbb R^{n\times r}$ with $r\ll n$ and introduce
an auxiliary variable $B\in\mathbb R^{m\times r}$.
We reparameterize the model parameter as
$\Theta \mapsto \Theta + B V^\top$ and compute derivatives with respect to
the low-dimensional variable $B$.
The resulting gradient in $B$ is then lifted back to an $m\times n$ matrix
by multiplying $V^\top$ to update the full parameter $\Theta$, as illustrated
in Figure~\ref{fig:lowrank_schematic}.
This construction keeps every update low-rank by design,
which leads to the following two low-rank gradient estimators.
\begin{definition}[Low-rank Gradient Estimators]
Given a random projection matrix $V\in\mathbb R^{n\times r}$ with low rank $r\ll \min\{m,n\}$ and an auxiliary variable $B\in\mathbb R^{m\times r}$, we define the following low-rank stochastic gradient estimators.

\begin{enumerate}
    \item[$\bullet$] \textbf{LowRank-IPA gradient estimator:} 
\begin{equation}\label{eq_lr_ipa}
     \hat{g}_{\text{LowRank\text{-}IPA}}(\xi,V,\Theta):=\nabla_BF(\xi,\Theta+BV^{\top})\bigg|_{B=\textbf{0}}V^{\top},
 \end{equation}
  where
       $\xi\sim p(\xi)$. 
       \item[$\bullet$] \textbf{LowRank-LR gradient estimator:} \begin{equation}\label{eq_lr_lr}
    \hat{g}_{\text{LowRank\text{-}LR}}(\xi,V,\Theta):=F(\xi)\nabla_B{\log p(\xi;\Theta+BV^{\top})}\bigg|_{B=\textbf{0}}V^{\top},
       \end{equation}
      where $\xi\sim p(\xi;\Theta)$.
\end{enumerate}
\end{definition}

\begin{figure}[htbp]
    \centering
    \includegraphics[width=0.9\textwidth]{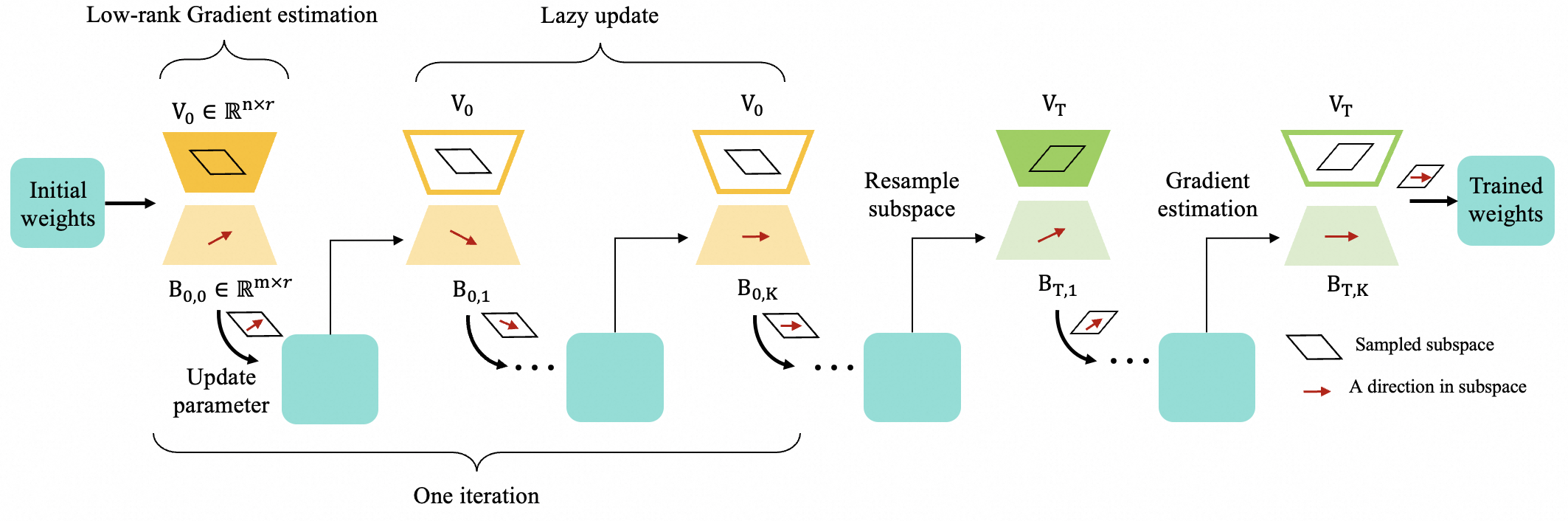}
   
\caption{Illustration of the proposed low-rank gradient estimator and the lazy-update
gradient descent framework. The gradient estimator with respect to $\Theta$ is computed via
a rank-$r$ reparameterization using an auxiliary variable $B$ and a randomly sampled projection
matrix $V$ (yellow trapezoids), and then lifted back to the original parameter
space by multiplying $V^\top$ (blue block). In the lazy-update scheme, the same projection direction
$V_0$ is reused for $K$ inner steps before switching to a new projection direction $V_1$ (green trapezoids).}\label{fig:lowrank_schematic}
\end{figure}

Intuitively, using the proposed gradient estimators can be viewed as a form of randomized subspace optimization.
At each step, we randomly select a low-dimensional subspace $V \in \mathbb{R}^{n\times r}$ from the decision space $\mathbb{R}^{m\times n}$, where $r \ll \min\{m,n\}$. Then we compute the directional derivative of the function $F$ along this subspace $V$. The gradient is evaluated at $B=0$ to ensure that the estimator corresponds to the gradient at the current value of $\Theta$. This way, the optimization is essentially conducted along a much lower-dimensional ($m\times r$) subspace.

The selection, or potentially the optimization, of the random projection matrix $V$ is central to achieving both memory savings and variance reduction. In addition to being of low rank, $V$
 must satisfy specific constraints to ensure the (weak) unbiasedness of the estimator as mentioned in Definition \ref{def1}. We introduce the following admissible class of distributions for the random matrix $V$.
 \begin{definition}\label{def3}
     We define the \textbf{space of admissible projection distributions} $\mathcal{D}$ as
\begin{equation}\label{dist}
    \mathcal{D} := \{\mathbb{P}\in \mathcal{P}(\mathbb{R}^{n\times r}):V\in\mathbb{R}^{n\times r}, \ \ \mathbb{E}_{V\sim\mathbb{P}}[VV^\top]=cI_n\},
\end{equation}
where \(\mathcal{P}(\mathbb{R}^{n\times r})\) denotes the set of all probability measures on \(\mathbb{R}^{n\times r}\), the first condition $V\in\mathbb{R}^{n\times r}$ is a low-rank constraint, and the second condition is an isotropy condition that ensures weak unbiasedness. We define the subspace projection matrix $P$ as $P:=VV^\top$.
 \end{definition}
 
The next theorem proves the unbiasedness of our low-rank stochastic gradient estimators if we sample the random projection $V$ from a distribution in $\mathcal{D}$. 
\begin{theorem}[Unbiasedness of the low-rank estimators]\label{thm:unbiasedness}
Fix a parameter block $\Theta\in\mathbb R^{m\times n}$ and let $V\in\mathbb R^{n\times r}$ be a random
projection matrix independent of the data randomness $\xi$.
The random subspace projection $V$ is sampled from a distribution in $\mathcal{D}$ in Equation \eqref{dist}. Assume the standard requirements for the unbiasedness of IPA and LR estimators hold \citep{mohamed2020monte}, then the LowRank-IPA gradient estimator \eqref{eq_lr_ipa} and LowRank-LR gradient estimator \eqref{eq_lr_lr} are weakly unbiased:
\begin{equation*}
    \mathbb{E}_{\xi,V}[\hat{g}_{\mathrm{LowRank\text{-}IPA}}(\xi,V,\Theta)] = \mathbb{E}_{\xi,V}[\hat{g}_{\mathrm{LowRank\text{-}LR}}(\xi,V,\Theta)] = c \cdot g(\Theta).
\end{equation*}
In particular, if $c=1$ then $\hat g_{\mathrm{LowRank\text{-}IPA}}$ and $\hat g_{\mathrm{LowRank\text{-}LR}}$  are strongly unbiased.
\end{theorem}

Theorem \ref{thm:unbiasedness} isolates the only requirement imposed on the random subspace: the average projector
$\mathbb E[VV^\top]$ must be proportional to the identity matrix. This condition ensures that, on average,
the projection does not favor any coordinate direction, so the low-rank update preserves the correct
descent direction up to a scalar factor. As a result, the design problem of choosing $V$
is cleanly separated from the underlying gradient-estimation family (IPA or LR): once the classical
estimator is unbiased, projecting it through an isotropic rank-$r$ subspace keeps it unbiased
(weakly for general $c$, strongly when $c=1$). Definition \ref{def3} provides the distributional constraints for the admissible projection distributions. In the next section, we will choose an optimal distribution from the admissible distributional class $\mathcal{D}$.

To provide a more concrete understanding of the proposed low-rank estimators, 
we instantiate them for a standard multi-layer feedforward neural network in the following Example~\ref{ex3}. 
This example derives the low-rank counterparts of the earlier IPA Example~\ref{ex:ipa} and the LR Example~\ref{ex:zo} (which focuses on the ZO special case). 
The same derivation extends readily to other neural network architectures 
(e.g., RNN, CNN, Transformers).

\begin{example}[Low-rank gradient estimators of feed-forward neural networks]\label{ex3}
Consider an $L$-layer feedforward neural network defined recursively by $z_\ell:=W_\ell x_{\ell-1}$ and $x_\ell := \psi_\ell(z_\ell)$ for $\ell=1,\dots,L$, where $W_\ell$ is the weight matrix at layer $\ell$, $\psi_\ell(\cdot)$ is the activation function, and $x_\ell$ denotes the activation at layer $\ell$.
We focus on training the weight matrix at the $i$-th layer, denoted by $\Theta := W_i \in \mathbb R^{m\times n}$. The network output is $\textsf{NN}_\Theta(x_0) = \psi_L\big(W_L \cdots \psi_{i+1}(W_{i+1} \cdot \psi_i(\Theta x_{i-1})) \cdots \big)$, where $x_0$ is the input and $x_{i-1}\in\mathbb{R}^n$. The final loss function is $F(\textsf{NN}_\Theta(x_0),\xi)$.

\begin{itemize}
    \item[(i)] LowRank-IPA estimator: given $V\in\mathbb R^{n\times r}$, 
\begin{equation}\label{eq:lr-ipa2}
\widehat g_{\mathrm{LowRank\text{-}IPA}}(\xi,V,\Theta)
=
\nabla_{z_i} F(\textsf{NN}_\Theta(x_0),\xi)\,(x_{i-1}^\top V)\,V^\top
\in \mathbb R^{m\times n},
\end{equation}
where $z_i := \Theta x_{i-1}\in\mathbb R^{m}$ and $\nabla_{z_i}F(\textsf{NN}_\Theta(x_0),\xi)$ is obtained via backpropagation, applying the chain rule to propagate the gradient from the final output layer backward to $z_i$.

\item[(ii)]  LowRank-LR estimator (ZO case):  following the ZO construction in Example~\ref{ex:zo}, we introduce a rank-$r$ perturbation $\sigma ZV^\top$ to $\Theta$, where $V\in\mathbb R^{n\times r}$, $Z\sim\mathcal N(0,I_{mr})$, and $\sigma>0$. The resulting low-rank one-point and two-point ZO estimator are given by
\begin{equation*}\label{eq:lr-zo-1pt}
\widehat g_{\mathrm{LowRank\text{-}LR\text{-}1pt}}(\xi,Z,V;\Theta)
=
F\big(\textsf{NN}_{\Theta+\sigma ZV^\top}(x_0),\xi\big)\,\frac{Z}{\sigma}\,V^\top
\in \mathbb R^{m\times n},
\end{equation*}
\begin{equation*}\label{eq:lr-zo-2pt}
\widehat g_{\mathrm{LowRank\text{-}LR\text{-}2pt}}(\xi,Z,V;\Theta)
=
\frac{
F\big(\textsf{NN}_{\Theta+\sigma ZV^\top}(x_0),\xi\big)
-
F\big(\textsf{NN}_{\Theta-\sigma ZV^\top}(x_0),\xi\big)
}{2\sigma}\,ZV^\top
\in \mathbb R^{m\times n}.
\end{equation*}

\end{itemize}

\end{example}

\subsection{Gradient Descent with Lazy Update}
 In this subsection, we introduce another essential strategy that complements the low-rank estimator: the \emph{lazy update}.
The one-step low-rank estimators in~\eqref{eq_lr_ipa}--\eqref{eq_lr_lr} resample a fresh random
subspace $V_t$ at every iteration. While conceptually simple, repeatedly sampling $V_t$
introduces additional randomness on top of the intrinsic stochasticity of the optimization
problem itself. Such frequent resampling may lead to excessive variance and memory overhead.

To address this issue, we adopt a \textit{lazy update} strategy implemented through a
two-level (outer--inner) stochastic gradient descent framework. An illustration of the lazy-update strategy is given in Figure~\ref{fig:lowrank_schematic}. At the beginning of each outer
iteration $t$, we sample one projection matrix $V_t$ and keep it fixed for the next $K$ inner
steps. During these $K$ steps, we compute low-rank gradients using the same $V_t$ and update the
parameters along the inner iterates
$\Theta_{(t,0)}, \Theta_{(t,1)}, \ldots, \Theta_{(t,K)}$.
Specifically, at inner step $k=0,1,\ldots,K-1$, let $\xi_{(t,k)}$ denote the randomness. Using the low-rank estimator in~\eqref{eq_lr_ipa} and \eqref{eq_lr_lr},
the inner-loop gradient descent of $\Theta$ takes the following form:
$$\Theta_{(t,k+1)}=
\Theta_{(t,k)}
-
\alpha_t
\nabla_B F(\xi_{(t,k)}, \Theta_{(t,k)} + B V_t^\top)\big|_{B=0} V_t^\top,$$
$$\Theta_{(t,k+1)}=
\Theta_{(t,k)}
-
\alpha_t
F(\xi_{(t,k)})
\nabla_B \log p(\xi_{(t,k)};\Theta_{(t,k)} + B V_t^\top)\big|_{B=0}
 V_t^\top,$$
 where the first one represents using the IPA gradient estimator and the second one represents using the LR gradient estimator.
After completing the $K$ inner updates, we set $\Theta_{t+1} = \Theta_{(t,K)}$,
and proceed to the next outer iteration by sampling a new projection matrix
$V_{t+1}$.  Intuitively, $K$ controls an exploration--exploitation tradeoff: a larger $K$ allows more
thorough optimization inside a sampled subspace, while a smaller $K$ switches subspaces more frequently,
encouraging exploration of different subspaces.

\begin{algorithm}[htbp]
\small
   \caption{(Low-rank gradient descent with lazy update)}
   \label{alg_lazyuodate}
   \begin{algorithmic}[1]
   \State Input: loss function $F(\xi,\Theta)$, initial iterate $\Theta_0$,  number of outer iterations $T$, the step-sizes $\alpha_t$.
   \For {$t \text{ in } 0: T-1$}
   \State Set $B_{(t,0)} = 0$, sample a low-rank matrix $V_t$ (will be introduced in Section \ref{sec5}) as the random subspace.
   \For {$k \text{ in } 0: K-1$}
   \State   Update the iteration for $\Theta$:   $\Theta_{(t,k)}=\Theta_t + B_{(t,k)}V_t^{\top}.$
   \State Fetch a batch of data $\xi_{(t,k)}$ and update the low-rank subspace gradient estimator $B_{(t,k)}$ by IPA or LR method:
   \begin{equation}[LowRank-IPA]\label{lazy_lowrank_ipa}
    \quad  B_{(t,k+1)}  = B_{(t,k)} - \alpha_t \nabla_B F(\xi_{(t,k)},\Theta_t+BV_t^{\top})\big|_{B=B_{(t,k)}}. 
   \end{equation}
   \begin{equation}[LowRank-LR]\label{lazy_lowrank_lr}
       \quad B_{(t,k+1)}  =B_{(t,k)}-\alpha_tF(\xi_{(t,k)})\nabla_B{\log p(\xi_{(t,k)};\Theta_t+BV_t^{\top})}\big|_{B=B_{(t,k)}}.
   \end{equation}
   \EndFor
   \State Update the iteration for $\Theta$: 
    $\Theta_{t+1} = \Theta_t + B_{(t,K)} V_t^{\top}$.

   \EndFor
   \State Output: $\Theta_{T}$.
   \end{algorithmic}
\end{algorithm}


In implementation, rather than directly updating the full parameter $\Theta$,
we adopt an equivalent formulation that optimizes a low-dimensional auxiliary
variable $B \in \mathbb{R}^{m \times r}$ within the subspace spanned by $V_t$,
and then maps the accumulated increment back to the original full parameter space. The pseudocode is presented in Algorithm~\ref{alg_lazyuodate}. At the beginning of each outer
iteration $t$, we draw a rank-$r$ projection matrix $V_t$ and initialize
the auxiliary variable by setting $B_{(t,0)} = \mathbf{0}$.
The inner loop maintains the parameterization
\begin{equation*}\label{eq:repara}
    \Theta_{(t,k)} = \Theta_t + B_{(t,k)} V_t^\top,
\qquad k = 0,1,\ldots,K-1,
\end{equation*}
so that all inner updates modify $\Theta_t$ only through the fixed projection
matrix $V_t$.
At inner step $k$, we update $B_{(t,k)}$ using either the IPA-type or
LR-type gradient descent rule (see \eqref{lazy_lowrank_ipa} and \eqref{lazy_lowrank_lr} in Algorithm~\ref{alg_lazyuodate}),
i.e., we perform a stochastic descent step with respect to $B$ while
evaluating the objective at the lifted parameter $\Theta_{(t,k)}$. Essentially, with $V_t$ and $\Theta_t$ fixed, the inner loop performs
stochastic gradient descent on the following low-dimensional subproblem:
\begin{equation}\label{eq_subproblem}
\min_{B \in \mathbb{R}^{m\times r}}
\;\mathbb{E}\!\left[
F\big(\xi, \Theta_t + B V_t^\top\big)
\right].
\end{equation}
After $K$ inner steps, the algorithm performs a {single} outer update of the full parameter by
\[
\Theta_{t+1}=\Theta_t + B_{(t,K)}V_t^\top,
\]
where $ B_{(t,K)} = -\alpha_t\sum_{s=0}^{K-1} \nabla_B F(\xi_{(t,s)},\Theta_t+BV_t^{\top})|_{B=B_{(t,s)}}$ can be written as a sum of $K$ low-rank increments taken inside the low-rank subspace spanned by $V_t$ as Equations \eqref{lazy_lowrank_ipa} or \eqref{lazy_lowrank_lr}. A new projection matrix $V_{t+1}$ is then sampled for the next outer iteration.

Combining with the lazy-update strategy, our proposed low-rank estimators offer significant memory savings in LLM training.  
The total memory cost can be decomposed into three components: memory for optimizer states (e.g., the first- and second-moment estimates in Adam-type optimizers \citep{kingma2014adam}), memory for gradients, and memory for activations.
First, the reduced dimensionality of the subproblem~\eqref{eq_subproblem} substantially decreases the memory required for both optimizer states and gradients.
Specifically, the memory needed for storing the gradient and the corresponding optimizer states is reduced from the original $m \times n$ to $m \times r$, which is significantly more efficient when $r \ll n$.
Second, the memory required to store intermediate activations is also greatly reduced.
Take the IPA case in Equation~\eqref{eq:lr-ipa2} as an example. 
In standard backpropagation, computing the gradient with respect to the weight matrix
requires storing the activation $x_{i-1}\in\mathbb{R}^{n}$ during the forward pass process.
In contrast, the low-rank estimator only needs to store the projected activation
$x_{i-1}^\top V \in \mathbb{R}^{r}$, which has a much lower dimension and thus leads to a substantial reduction in memory consumption.

\section{The Optimal Distribution for Random Subspace Projection}\label{sec5}

In this section, we optimize the distribution of $V$ over the admissible class $\mathcal{D}$, to minimize the MSE of the gradient estimator. This defines an infinite-dimensional optimization problem over a space of probability distributions, which is generally challenging. Our analysis is general and applies irrespective of the gradient estimation
family, including both IPA and LR methods.

We use the MSE to evaluate the quality of our gradient estimator. The estimator involves two independent sources of randomness: the data-driven noise, represented by $\xi$, and the randomness introduced by the subspace projection, represented by $V$ (with the subspace projection matrix defined as $P=VV^\top$). Accordingly, the MSE admits a natural decomposition into three components: (i) the intrinsic variance of the underlying IPA/LR estimator, (ii) the additional variance induced by the random projection, and (iii) a scalar bias term, which also originates from the random projection: since our estimator is only weakly unbiased, even a scalar bias increases the MSE.

\begin{proposition}\label{propMSE}
    The MSE of our low-rank gradient estimators $\hat g_{\mathrm{LowRank-IPA/LR}}$ defined in Equations \eqref{eq_lr_ipa} and \eqref{eq_lr_lr} satisfies
    \begin{equation}\label{MSE}
    \begin{aligned}
        \mathrm{MSE} &:=  \mathbb{E}_V\mathbb{E}_{\xi}[\Vert \hat g_{\mathrm{LowRank\text{-}IPA/LR}}(\xi,\Theta)-g(\Theta)\Vert_{F}^2] =  \underbrace{\operatorname{tr}\!\bigl(\Sigma_{\xi}\,\mathbb{E}[P^{2}]\bigr)}
              _{\text{IPA/LR variance}}
+\underbrace{\operatorname{tr}\!\bigl(\Sigma_{\Theta}\,
            \mathbb{E}[P^{2}-c^2I_n]\bigr)}
              _{\text{random projection variance}}+
  \underbrace{(1-c)^2\operatorname{tr}\Sigma_{\Theta}}_{\text{scalar bias}},
    \end{aligned}
    \end{equation}
    where $\Sigma_{\xi}:=\mathbb{E}_{\xi}[(\hat g_{\mathrm{IPA/LR}}(\xi,\Theta)- g(\Theta))^{\top}(\hat g_{\mathrm{IPA/LR}}(\xi,\Theta)- g(\Theta))]$, $\Sigma_{\Theta}:= g(\Theta)^{\top} g(\Theta)$, and $\hat g_{\mathrm{IPA/LR}}$ is classical IPA/LR estimators defined in Equations \eqref{eq:ipa:1} and \eqref{eq:lr:1}. 
\end{proposition}


In Proposition \ref{propMSE}, $\operatorname{tr}(\Sigma_{\xi})$ is the intrinsic statistical variance of the
underlying IPA/LR estimator before projection, while
$\operatorname{tr}(\Sigma_{\Theta})=\|g(\Theta)\|_F^2$ measures the squared Frobenius norm of the
true gradient $g(\Theta)$.
The decomposition in~\eqref{MSE} separates three effects.
The first term shows that $\Sigma_\xi$ is weighted by the second moment of the random projector, which is exactly where the projection distribution matters. For fixed intrinsic noise $\Sigma_\xi$, the data-induced variance scales with the second moment of random projection.
The second term is the additional variance
purely due to the randomness of the projection.
The third term is a scalar bias penalty:
since the estimator is only weakly unbiased with $\mathbb E[P]=cI_n$, any mismatch $c\neq 1$
shrinks or inflates the gradient on average and increases the MSE proportionally to $\|g(\Theta)\|_F^2$.

It is convenient to regroup the three terms as a single quadratic form.
Using $(P-I_n)^2=P^2-2P+I_n$ and $\mathbb E[P]=cI_n$, we obtain
\[
\mathrm{MSE}
=\operatorname{tr}\!\Big(\Sigma_{\xi}\,\mathbb E[P^2]\Big)
+\operatorname{tr}\!\Big(\Sigma_{\Theta}\,\mathbb E[(P-I_n)^2]\Big)
=\operatorname{tr}\!\Big((\Sigma_{\xi}+\Sigma_{\Theta})\,\mathbb E[P^2]\Big)
+(1-2c)\operatorname{tr}(\Sigma_{\Theta}).
\]
This representation highlights the central design objective: for fixed $c$ and rank $r$, minimizing the MSE
amounts to minimizing the weighted second moment $\operatorname{tr}\big((\Sigma_{\xi}+\Sigma_{\Theta})\,\mathbb E[P^2]\big)$.
In particular, $\Sigma:=\Sigma_{\xi}+\Sigma_{\Theta}$ acts as an instance weight that trades off
how much we care about suppressing data noise versus preserving the signal $g(\Theta)$, and
$\mathbb E[P^2]$ is the sole place where the projection distribution enters.
In the sequel, we study both the instance-dependent setting (when $\Sigma$ is known or estimable)
and the instance-independent setting (when $\Sigma$ is unknown and we minimize an upper bound),
with $\operatorname{tr}(\Sigma\,\mathbb E[P^2])$ as the key term to control.

\subsection{Instance-Independent Optimization}
We first consider the instance-independent setting, where no prior information is available about \(
\Sigma=\Sigma_{\xi}+\Sigma_{\Theta}.
\)
In this case, it is natural to optimize a worst-case instance-independent upper bound of the MSE. Using the spectral–Frobenius inequality \citep{fang1994inequalities}, we have
\begin{equation}\label{upper}
    \text{tr}(\Sigma\mathbb{E}P^2)\le \Vert\Sigma\Vert_2\text{tr}(\mathbb{E}P^2).
\end{equation}
Therefore, if $\Sigma$ is unknown, a principled surrogate is to minimize
$\operatorname{tr}(\mathbb E[P^2])$, which depends only on the projection law and
controls the MSE uniformly over all problem instances with the same spectral scale
$\|\Sigma\|_2$. This leads to the following distributional optimization problem:
\begin{equation}\label{OPT1}
    \min_{\mathcal{D}} \text{tr}(\mathbb{E}_V[P^2]), \quad s.t. \ P=VV^{\top},\ \mathbb{E}P=cI_n.
\end{equation}
The constraint $P=VV^\top$ enforces the rank-$r$ structure, while $\mathbb E[P]=cI_n$
imposes (weak) unbiasedness by making the projection isotropic in expectation. Theorem~\ref{thm2} characterizes the optimal solution for the random projection matrix $P$.
\begin{theorem}[Optimal instance-independent low-rank random projector]\label{thm2}
    The solution of the constrained distributional optimization problem (\ref{OPT1}) is 
    \begin{equation*}
        \min_{\mathcal{D}} \mathrm{tr}(\mathbb{E}_V[P^2]) = \frac{n^2c^2}{r},
    \end{equation*}
and the equality holds if and only if the distribution of $V$ satisfies $V^{\top}V=\frac{cn}{r}I_r$ almost surely. 
 \end{theorem}

This result provides a clean geometric interpretation: among all rank-$r$ random projectors
that are isotropic in expectation, the best projection is the one that spreads the projection weight
as evenly as possible across the $r$ directions. It requires
\(
V^\top V=\frac{cn}{r}I_r,
\)
meaning that the $r$ columns of $V$ are orthogonal and have the same length.
Equivalently, $P$ has rank $r$ and all of its nonzero eigenvalues are equal, so no direction
inside the chosen subspace is favored over another. As a result, the second-moment term
$\operatorname{tr}(\mathbb E[P^2])$ is minimized, which in turn minimizes the MSE upper bound. 
A similar projector structure appears in \citet{kozak2023zeroth} for vector-space optimization,
but its optimality was not explicitly characterized there. In our setting, we show that the same
structure is exactly the optimal choice for stochastic low-rank projections in matrix space
when $\Sigma$ is unknown.

Algorithms~\ref{alg:haar_stiefel_boxed}--\ref{alg:coord_axis_boxed} provide two practical ways to sample such an optimal random projection matrix $V\in\mathbb R^{n\times r}$.
Both algorithms follow the same high-level template: we first generate an $n\times r$ matrix $U$ whose
columns form an orthonormal set that satisfy the optimality condition in Theorem \ref{thm2}, and then rescale it by
$\alpha=\sqrt{cn/r}$ to obtain $V=\alpha U$. This rescaling ensures that the average projection strength
matches the weak-unbiasedness requirement $\mathbb E[VV^\top]=cI_n$.

\begin{algorithm}[htbp]
\small
\caption{Haar--Stiefel sampler}
\label{alg:haar_stiefel_boxed}
\hrule\vspace{4pt}
\begin{algorithmic}[1]
\Require integers $n>r$, constant $c>0$, $\alpha=\sqrt{cn/r}$.
\Ensure $V\in\mathbb{R}^{n\times r}$ and $P\in\mathbb{R}^{n\times n}$ with $P=VV^\top$.
\State Sample $G\in\mathbb{R}^{n\times r}$ with i.i.d.\ Gaussian entries $G_{ij}\sim\mathcal{N}(0,1)$.
\State Compute the thin QR factorization $G=QR$, where $Q\in\mathbb{R}^{n\times r}$ has orthonormal columns
      and $R\in\mathbb{R}^{r\times r}$ is upper triangular.
\State Remove the QR sign ambiguity:
      set $D\gets \operatorname{diag}(\operatorname{sgn}(\operatorname{diag}(R)))$, and let $U\gets QD$.
      \Comment{$U$ is Haar–Stiefel on $\operatorname{St}(n,r)$}
\State Set $V\gets \alpha U$, $P\gets VV^\top$.
\State \Return $(V,P)$.
\end{algorithmic}
\end{algorithm}

In the Haar--Stiefel sampler
(Algorithm~\ref{alg:haar_stiefel_boxed}), we aim to draw a uniformly random orthonormal frame
in $\mathbb R^n$. Here an $r$-frame in $\mathbb R^n$ simply means an ordered collection of $r$ vectors
$(u_1,\ldots,u_r)$.
An orthonormal frame is a frame whose columns are orthonormal, i.e.,
$u_i^\top u_j=\delta_{ij}$ for all $i,j$, or equivalently $U^\top U=I_r$.
The set of all orthonormal $r$-frames forms the so-called Stiefel manifold
$\mathrm{St}(n,r)=\{U\in\mathbb R^{n\times r}:U^\top U=I_r\}$; see \cite{chikuse2003statistics} and \cite{stewart1980efficient} for details. Sampling from the Haar measure means the distribution is invariant under rotations:
for any orthogonal matrix $Q$, $U$ and $QU$ have the same law. This invariance is the precise
mathematical notion of being ``uniform'' on $\mathrm{St}(n,r)$, and it implies that the resulting
$r$-dimensional column space of $U$ is uniformly random among all $r$-dimensional subspaces of
$\mathbb R^n$. Operationally, the algorithm implements this by drawing a Gaussian matrix $G$ and then
orthonormalizing its columns via a thin QR factorization. The QR step can be viewed as a projection of
$G$ onto the Stiefel manifold: it keeps the directional information (the span and orientation)
while removing the scale information. The additional diagonal sign correction makes the output
exactly Haar–Stiefel, rather than merely having orthonormal columns. Intuitively, this sampler explores
all directions in $\mathbb R^n$ in a perfectly symmetric way, which is why it is a natural choice when
no instance-specific information about $\Sigma$ is available.



\begin{algorithm}[htbp]
\small
\caption{Coordinate--axis sampler}
\label{alg:coord_axis_boxed}
\hrule\vspace{4pt}
\begin{algorithmic}[1]
\Require integers $n>r$, constant $c>0$, $\alpha=\sqrt{cn/r}$.
\Ensure $V\in\mathbb{R}^{n\times r}$ and $P\in\mathbb{R}^{n\times n}$ with $P=VV^\top$.
\State Sample a subset $J\subset\{1,\ldots,n\}$ uniformly without replacement with $|J|=r$.
\State Order $J$ as $(j_1,\ldots,j_r)$.
\State Construct $U\in\mathbb{R}^{n\times r}$ by setting $U_{:,k}\gets e_{j_k}$ for $k=1,\ldots,r$.
      \Comment{$U^\top U=I_r$ and $UU^\top$ is a coordinate projector}
\State Set $V\gets \alpha U$.
\State Set $P\gets VV^\top$ (equivalently, $P_{jj}=\alpha^2$ if $j\in J$, and $P_{jj}=0$ otherwise).
\State \Return $(V,P)$.
\end{algorithmic}
\end{algorithm}

The coordinate--axis sampler (Algorithm~\ref{alg:coord_axis_boxed}) provides a discrete isotropic alternative.
It selects $r$ coordinates uniformly without replacement, forms $U$ by stacking the corresponding
standard basis vectors, and then rescales. Equivalently, this method randomly chooses $r$ columns of
the identity matrix and applies the same scaling. 

Both methods treat all directions equally, and within the
chosen rank-$r$ subspace, they allocate the same weight to each direction. This is exactly what prevents
the projection from introducing extra variability beyond what is unavoidable under the constraints. The following proposition guarantees that both sampling rules in Algorithms \ref{alg:haar_stiefel_boxed} and \ref{alg:coord_axis_boxed} satisfy the admissibility constraints
(low-rank and $\mathbb E[VV^\top]=cI_n$) and also meet the optimality condition characterized in
Theorem~\ref{thm2}.

\begin{proposition}[Constructions]\label{prop:isotropic}
For the above sampling methods in Algorithms \ref{alg:haar_stiefel_boxed} and \ref{alg:coord_axis_boxed}, the following identities hold almost surely: 
\begin{equation*}
    P=VV^{\top},\quad \mathbb{E}P = cI_n,\quad V^{\top}V=\frac{cn}{r}I_r.
\end{equation*}
Consequently, the induced low-rank estimator attains the minimum value of the instance-independent surrogate $\mathrm{tr}(\mathbb{E}[P^2])$, and hence the smallest uniform upper bound within the admissible class.
\end{proposition}

Thus, both constructions in Algorithms \ref{alg:haar_stiefel_boxed} and \ref{alg:coord_axis_boxed} meet the two design constraints in Definition \ref{def3} and also satisfy the optimal condition in
Theorem~\ref{thm2}. In other words, if we sample $V$ using either the Haar--Stiefel rule
or the coordinate--axis rule, then the resulting projector $P=VV^\top$ achieves the smallest possible value of $\operatorname{tr}(\mathbb E[P^2])$
among all rank-$r$ choices with $\mathbb E[P]=cI_n$. Therefore, these two sampling schemes
solve the instance-independent problem~\eqref{OPT1} and yield the projection law that minimizes the worst-case upper bound on the estimator MSE.

Moreover, plugging the optimal value $\operatorname{tr}(\mathbb E[P^2])=n^2c^2/r$ into the bound~\eqref{upper}
gives a simple uniform control of the MSE.
\begin{equation}\label{MSE5.1}
    \mathbb{E}_V\mathbb{E}_{\xi}[\Vert \hat g_{\mathrm{LowRank\text{-}IPA/LR}}-g(\Theta)\Vert_{F}^2]\le \frac{c^2n}{r} \Vert \Sigma_{\xi}\Vert_2 + (1-2c+\frac{c^2n}{r})\Vert\Sigma_{\Theta}\Vert_2.
\end{equation}
This bound makes the role of $(r,c)$ explicit. 
The rank $r$ balances memory and estimation error. Using a smaller $r$ reduces the storage cost, but it also increases the MSE because the projection is coarser. 
In contrast, $c$ balances bias and variance by determining the average scaling of the projected gradient through $\mathbb E[P]=cI_n$.
Choosing $c<1$ reduces the variance constants (since $\mathbb E[P^2]$ scales with $c^2$) but introduces a
nonzero scalar bias $(1-c)g(\Theta)$, which contributes the MSE term.
The following remark makes these effects explicit by comparing our estimator with standard baselines, showing the theoretical improvement of our method.
\begin{remark}\label{remark1}
   We have the following two baselines:
\begin{itemize}
    \item Full-rank estimator \citep{mohamed2020monte}. If we  use the original full-rank estimator  $\hat g_{\mathrm{IPA/LR}}$ without randomized subspace optimization, the MSE is $\text{MSE}_\text{F}=\mathbb{E}_{\xi}[\Vert \hat g_{\mathrm{IPA/LR}}-g(\Theta)\Vert_F^2]=\text{tr}(\Sigma_{\xi})$. It is unbiased but suffers from high memory requirements because the whole gradient matrix needs to be stored. 
    \item Gaussian low-rank estimator \citep{chen2024enhancing,he2024subspace}. This means that the matrix-valued gradient estimator is projected onto a random matrix whose entries are i.i.d. standard Gaussian variables. It is unbiased and memory-efficient. However, it suffers from high MSE since it does not satisfy the optimal condition in Theorem \ref{thm2}.
    The MSE is
    \begin{equation*}
      \text{MSE}_\text{G} = \frac{n+r+1}{r}\text{tr}(\Sigma_{\xi}) + \frac{n+1}{r}\text{tr}(\Sigma_{\Theta}). 
    \end{equation*}
\end{itemize} 
\end{remark}

Our proposed optimal gradient estimator is also memory-efficient and weakly unbiased by the two constraints. The hyperparameter $c$ determines the amount of gradient information retained. A larger $c$ implies stronger directional fidelity but higher noise. Note that as the optimization progresses, the norm of the gradient matrix $\Sigma_{\Theta}$ will decrease to zero. In the weak unbiased setting, we can choose a relatively small $c$ to lower the variance of the gradient estimator. For instance, when $c=\frac{r}{n}$, we have $\text{MSE}=\frac{r}{n} \Vert \Sigma_{\xi}\Vert_2 + (1-\frac{r}{n})\Vert\Sigma_{\Theta}\Vert_2$ in Equation \eqref{MSE5.1}. As optimization proceeds, $\|\Sigma_{\Theta}\|_{2}\!\to0$ and the estimator enjoys an \(\mathcal{O}(r/n)\) variance reduction by sacrificing the strong unbiasedness compared to the full-rank estimator. This is the minimum of the uniform upper bound of the MSE for our estimator. 
Under the strong unbiasedness condition, we have $\text{MSE}=\frac{n}{r} \Vert \Sigma_{\xi}\Vert_2 + (\frac{n}{r}-1)\Vert\Sigma_{\Theta}\Vert_2$, which is not necessarily smaller than $\text{MSE}_\text{F}$ but smaller than $\text{MSE}_\text{G}$. We will address this issue in the next section.

\subsection{Instance-Dependent Optimization}
We now move beyond the worst-case design in the previous subsection and consider the
information-aware setting, where the second-moment matrix
\(
\Sigma:=\Sigma_{\xi}+\Sigma_{\Theta}
\)
is available, or can be roughly estimated from a small set of warm-up samples. Recall that $\Sigma$ summarizes both the noise level of the
underlying IPA/LR estimator (through $\Sigma_\xi$) and the strength of the current gradient signal
(through $\Sigma_\Theta$). In this case, it is no longer optimal to treat all directions in
$\mathbb R^n$ equally. Instead, the projection should allocate more mass to directions where $\Sigma$
is large, and less mass to directions where $\Sigma$ is small.

This leads to the following instance-dependent distributional optimization problem:
\begin{equation}\label{instanceind}
\min_{\mathcal{D}}\ \Phi(P)
:= \operatorname{tr}\!\big(\mathbb{E}_V[\Sigma P^2]\big),
\qquad \text{s.t.}\quad
P=VV^{\top},\ \ \mathbb{E}[P]=cI_n.
\end{equation}
The constraints are the same as before. We still require $P$ to be rank $r$ and weakly unbiased in
expectation. The difference is that the objective is now weighted by $\Sigma$, so different
eigenspaces of $\Sigma$ are no longer interchangeable. In the instance-independent setting (\ref{OPT1}), the optimizer is blind to the spectrum of $\Sigma$, so every coordinate is treated symmetrically; the problem reduces to minimizing $\text{tr}(\mathbb{E}P^2)$ subject to rank-r and isotropy conditions, and the minimizers are those distribution whose projector has equal non-zero eigenvalues, yielding the MSE floor $\frac{c^2n^2}{r}$.  When $\Sigma$ is not proportional to the
identity, a symmetric (isotropic) law such as Haar--Stiefel wastes projection budget on directions
that contribute little to the objective. Therefore, the optimal distribution is generally
anisotropic and adapts to the spectrum of $\Sigma$.

Our next main theorem characterizes the solution of~\eqref{instanceind} and provides an optimality condition to achieve the minimum.

\begin{theorem}[Optimal instance-dependent low-rank random projector]\label{thm:SigmaP2}
Define the spectral decomposition of a positive semi-definite matrix $\Sigma$ as
$\Sigma = Q\,\mathrm{diag}(\sigma_1,\ldots,\sigma_n)\,Q^\top$,
with
$\sigma_1\ge \cdots \ge \sigma_n\ge 0$.
For the instance-dependent problem \eqref{instanceind}, the optimal value equals
\begin{equation}\label{eq:Phi-star-general-v2}
\Phi_{\min}
=
c^{2}\!\sum_{i:\pi_i^{\star}=1}\sigma_i
+
\frac{c^{2}}{r-t}
\Bigl(\sum_{i:\pi_i^{\star}<1}\sqrt{\sigma_i}\Bigr)^{2},
\end{equation}
where $\pi^\star=(\pi_1^\star,\ldots,\pi_n^\star)\in[0,1]^n$ and $t$ satisfy
\begin{equation}\label{eq:pi-star-v2}
t:=\#\{i:\pi_i^{\star}=1\},
\qquad
\pi_i^{\star}
=
\min\!\Bigl\{
1,\ 
(r-t)\sqrt{\sigma_i}\Big/\!\!\sum_{j:\pi_j^{\star}<1}\sqrt{\sigma_j}
\Bigr\},
\quad i=1,\ldots,n.
\end{equation}
 Moreover, a distribution in $\mathcal{D}$ is optimal if it satisfies the following conditions:
\begin{equation}\label{eq:opt-cond-v2}
\mathbb E[P]=cI_n,
\qquad
\mathbb E[Q^\top P^2 Q]=c^2\,\mathrm{diag}\!\Bigl(\frac{1}{\pi_1^\star},\ldots,\frac{1}{\pi_n^\star}\Bigr),
\qquad
0<\pi_i^\star\le 1,\ \ \sum_{i=1}^n \pi_i^\star = r .
\end{equation}

\end{theorem}

The instance-independent design in Theorem \ref{thm2} treats every direction the same, because it does not know which
directions matter more. Once $\Sigma$ is available, the objective $\operatorname{tr}(\Sigma\,\mathbb E[P^2])$ in Theorem \ref{thm:SigmaP2}
assigns different weights $\sigma_i$ to different eigendirections. The optimal rule therefore puts
more sampling effort on directions with larger $\sigma_i$. When $\sigma_i$ is large, the rule saturates at
$\pi_i^\star=1$, meaning that this direction should be included in every sampled subspace.
For the remaining directions, the probabilities are proportional to $\sqrt{\sigma_i}$, so larger
$\sigma_i$ still means more frequent inclusion. This produces a strict improvement over isotropic
sampling whenever the spectrum is non-flat, because it avoids spending projection budget on
directions that have little weight in $\Sigma$.

From a technical perspective, the main difficulty is that \eqref{instanceind} is a constrained
distributional optimization problem. The decision variable is not a matrix but a law over
random projectors $P=VV^\top$, and it must satisfy both a low-rank constraint (almost surely) and an
isotropy constraint (in expectation). Our core contribution is to reduce this infinite-dimensional problem to a
finite convex program. After rotating to the eigenbasis of $\Sigma$, one can lower bound the objective by a function depending only on the diagonal directional statistics $\pi_i=\Pr(i\ \text{is selected})$, and this bound is tight via a diagonal projector construction. This yields a diagonal
form where the law matters only through inclusion probabilities $\pi_i$, leading to a
 convex program with an explicit KKT solution \eqref{eq:pi-star-v2}. A tightness construction then shows
the bound is achievable, giving the closed-form optimum \eqref{eq:Phi-star-general-v2}.

Finally, the characterization  not only identifies the
optimal value but also tells us what to sample by a clean structural description of optimality in \eqref{eq:opt-cond-v2}.
The first moment condition $\mathbb E[P]=cI_n$ enforces weak unbiasedness, while the second moment
condition pins down exactly how large $\mathbb E[P^2]$ must be along each eigen-direction. The next algorithm implements this sampling rule by
drawing an eigen-direction subset with marginal probabilities $\pi^\star$ and then rescaling the
selected directions to satisfy $\mathbb E[P]=cI_n$.
\begin{algorithm}[htbp]
\small
\caption{Instance-dependent sampling of an optimal low-rank projector}
\label{alg:sigma_dependent_sampler}
\hrule\vspace{4pt}
\begin{algorithmic}[1]
\Require $\Sigma$, rank $r<n$, scalar $c>0$.
\Ensure $V\in\mathbb R^{n\times r}$ and $P=VV^\top$.

\State Compute the spectral decomposition $\Sigma=Q\,\mathrm{diag}(\sigma_1,\ldots,\sigma_n)Q^\top$
      with $\sigma_1\ge\cdots\ge\sigma_n\ge 0$.
\State Compute the optimal inclusion probabilities $\pi^\star$ by \eqref{eq:pi-star-v2}
      (equivalently, find $t$ and normalize the remaining mass so that $\sum_i \pi_i^\star=r$).
\State Sample a random subset $J\subset\{1,\ldots,n\}$ of fixed size $|J|=r$ such that
      \[
      \Pr(i\in J)=\pi_i^\star,\qquad i=1,\ldots,n,
      \]
      using any fixed-size unequal-probability design (e.g., conditional Poisson, Sampford, or Tillé’s elimination).
\State For each $i\in J$, set the weight $\mu_i\gets c/\pi_i^\star$.
\State Form $V\gets Q_J\,\mathrm{diag}\!\big(\sqrt{\mu_i}\big)_{i\in J}$, where $Q_J$ collects the columns $\{q_i\}_{i\in J}$.
\State Set $P\gets VV^\top$ and return $(V,P)$.
\end{algorithmic}
\end{algorithm}

Algorithm~\ref{alg:sigma_dependent_sampler} implements the optimal design in Theorem~\ref{thm:SigmaP2} in a
constructive way. It first diagonalizes $\Sigma=Q\,\mathrm{diag}(\sigma)Q^\top$ to work in the eigenbasis, where
each direction $q_i$ carries weight $\sigma_i$ in the objective. It then computes the target inclusion
probabilities $\pi_i^\star$ from \eqref{eq:pi-star-v2} and samples a fixed-size subset $J$ with $|J|=r$
such that $\Pr(i\in J)=\pi_i^\star$. Because $0<\pi_i^{\star}\le1$ and $\sum_i\pi_i^{\star}=r$,
standard sampling theory 
guarantees the existence of a fixed–cardinality $\pi$–ps design
with inclusion probabilities $\pi_i^{\star}$.
Algorithms such as Sampford \citep{sampford1967sampling}, conditional Poisson \citep{hajek1964asymptotic,deville1998unequal}, or Tillé’s
sequential elimination \citep{bondesson2008list} provide explicit constructions. Finally, it assigns
weights $\mu_i=c/\pi_i^\star$ to the selected directions and forms
$P=\sum_{i\in J}\mu_i q_i q_i^\top=VV^\top$. The rescaling by $1/\pi_i^\star$ is what guarantees the isotropy
constraint $\mathbb E[P]=cI_n$ even though the subset is sampled unevenly. The following
Proposition~\ref{prop:sampler_optimal} shows that this construction satisfies the optimality conditions
\eqref{eq:opt-cond-v2}, and therefore realizes the optimal distribution characterized in
Theorem~\ref{thm:SigmaP2}.

\begin{proposition}[Constructions]\label{prop:sampler_optimal}
Let $(V,P)$ be generated by Algorithm~\ref{alg:sigma_dependent_sampler}. Then $P=VV^\top$ has rank at most $r$ and
$\mathbb E[P]=cI_n.$
Moreover, in the eigenbasis of $\Sigma$,
\[
\mathbb E\!\big[Q^\top P^2 Q\big]
=
c^2\,\mathrm{diag}\!\Bigl(\frac{1}{\pi_1^\star},\ldots,\frac{1}{\pi_n^\star}\Bigr),
\]
so the distribution produced by the algorithm satisfies the optimality conditions
\eqref{eq:opt-cond-v2} and therefore attains $\Phi_{\min}$ in Theorem~\ref{thm:SigmaP2}.
\end{proposition}

Proposition \ref{prop:sampler_optimal} guarantees the efficiency of Algorithm \ref{alg:sigma_dependent_sampler}.
We sample directions more often when $\sigma_i$ is large, but we also rescale each selected direction
by $c/\pi_i^\star$. This rescaling exactly cancels the uneven sampling so that the estimator remains
(weakly) unbiased in the sense $\mathbb E[P]=cI_n$. At the same time, the remaining degrees of freedom
are used to reduce the weighted second moment $\operatorname{tr}(\Sigma\,\mathbb E[P^2])$.
In short, the algorithm adapts to $\Sigma$ through where it samples (via $\pi_i^\star$) while
preserving unbiasedness through how it rescales (via $\mu_i=c/\pi_i^\star$).

The following Proposition \ref{thm3} shows that after projection, the variance term can reach the full-rank estimator under our
optimal sampling policy. Therefore, we can obtain a random projection gradient estimator that requires less memory, and can simultaneously achieve strong unbiasedness and an MSE as small as that of the full-rank baseline under a meaningful condition. In particular, the minimal MSE under the optimal
instance-dependent projector takes the explicit form
  \begin{equation*}
      \text{MSE}=\text{tr}((\Sigma_{\xi}+\Sigma_{\Theta})\mathbb{E}P^2)+(1-2c)\text{tr}(\Sigma_{\Theta})
    =c^2\sum_{i:\pi_i^{\star}=1}\sigma_i+\frac{c^{2}}{r-t}\bigg(\sum_{i:\pi_i^{\star}<1}\sqrt{\sigma_i}\bigg)^2 + (1-2c)\text{tr}(\Sigma_{\Theta}).
\end{equation*}
\begin{proposition}[Theoretical improvement with optimal projection]\label{thm3}
Under the strong unbiasedness condition, the MSE of the low-rank random projection gradient estimator is no larger than the full-rank estimator when $\operatorname{rank}(\Sigma)\le r$. When $c=1$, the following result holds: 
\begin{equation*}
       \text{MSE}_{\min} \le \mathrm{tr}(\Sigma_{\xi}).
      \end{equation*}
\end{proposition}

This condition $\mathrm{rank}(\Sigma)\le r$ has a clear interpretation. It says that, in the eigenbasis of
$\Sigma$, only at most $r$ directions matter. When this holds, an $r$-dimensional subspace can capture all
nonzero eigendirections of $\Sigma$, so projection does not throw away any direction that the MSE objective
actually cares about. The remaining randomness is then only the intrinsic data noise in $\xi$, but the
optimal sampler avoids amplifying it unnecessarily. In this case, projection can only help: it keeps the
directions that dominate the MSE, filters out irrelevant ones, and reduces the memory from $m\times n$ to
$m\times r$.

The additional requirement $c=1$ 
eliminates the scalar bias term $(1-c)^2\operatorname{tr}\Sigma_{\Theta}$ and isolates the variance effect of
projection, making the comparison to the full-rank baseline clean. In short, when the effective dimension is
within the projection budget, and we enforce strong unbiasedness, the optimal projection policy can reduce
variance while saving memory, without incurring an MSE penalty; see Remark \ref{remark1}.

\section{Numerical Experiments}\label{sec6}
In this section, we present experiments ranging from controlled toy settings to large-scale LLM training, aiming to validate the theoretical results and the practical effectiveness of our framework. In Section \ref{sec6.1}, we begin with a toy example that serves as a sanity check of the theory. In Section \ref{sec6.2.1}, we move to LLM fine-tuning under realistic memory and compute constraints, showing that our method performs well in practice. In Section \ref{sec6.2.2}, we further study LowRank-IPA in large-scale pretraining and show that the improvements persist in long-horizon training regimes. Overall, these experiments provide end-to-end empirical support for our framework and demonstrate a favorable accuracy–efficiency trade-off over standard random or heuristic subspace choices. 

All experiments are conducted on a machine equipped with 4 RTX Pro 6000 GPUs. The \textbf{Gaussian LowRank-LR/IPA} is the vanilla LowRank-LR/IPA gradient estimation method, which samples projection matrices $V$ from the Gaussian distribution; the \textbf{Stiefel/Coordinate LowRank-LR/IPA} is the proposed instance-independent gradient estimation method, which samples random orthonormal frames from Algorithm \ref{alg:haar_stiefel_boxed} or Algorithm \ref{alg:coord_axis_boxed}. The \textbf{Dependent LowRank-LR/IPA} is the proposed instance-dependent gradient estimation method following  Algorithm \ref{alg:sigma_dependent_sampler}.

\subsection{Toy Example}\label{sec6.1}


In the toy example, we validate our gradient-estimation framework on a unified quadratic matrix regression objective
\begin{equation}\label{eq:toy_obj}
\min_{W\in\mathbb{R}^{m\times n}} f(W),\qquad 
f(W):=\mathbb{E}_{A\sim \mathcal N(\mu^\top,\Sigma)}\Big[\tfrac12\|A W B - C\|_F^2\Big],
\end{equation}
where \(A\in\mathbb{R}^{1\times m}\) is a Gaussian random vector, \(B\in\mathbb{R}^{n\times o}\) and \(C\in\mathbb{R}^{1\times o}\) are fixed random matrices with i.i.d.\ standard normal entries, and \(W\in\mathbb{R}^{m\times n}\) is the decision variable.
We set \(m=n=100\) and \(o=30\).
This controlled quadratic objective admits a closed-form gradient,
\[
\nabla_W f(W)=(\Sigma+\mu\mu^\top)\,W(BB^\top)-\mu\,(CB^\top),
\]
which allows us to directly quantify the bias--variance trade-off of different subspace-sampling strategies.
Within the same problem \eqref{eq:toy_obj}, we evaluate two classes of gradient estimators: a LowRank-LR estimator, which treats the randomness in \(A\) as the source of stochasticity and constructs an unbiased score-function-type gradient estimate, and a LowRank-IPA estimator, which leverages differentiability of the sample-path loss \(\tfrac12\|A W B - C\|_F^2\) in \(W\) to obtain a pathwise gradient estimate. 

\begin{figure}[h]
\centering
\subfloat[$c=0.1$]{
\includegraphics[trim=0cm 0.5cm 0cm 0cm, width=0.3\textwidth]{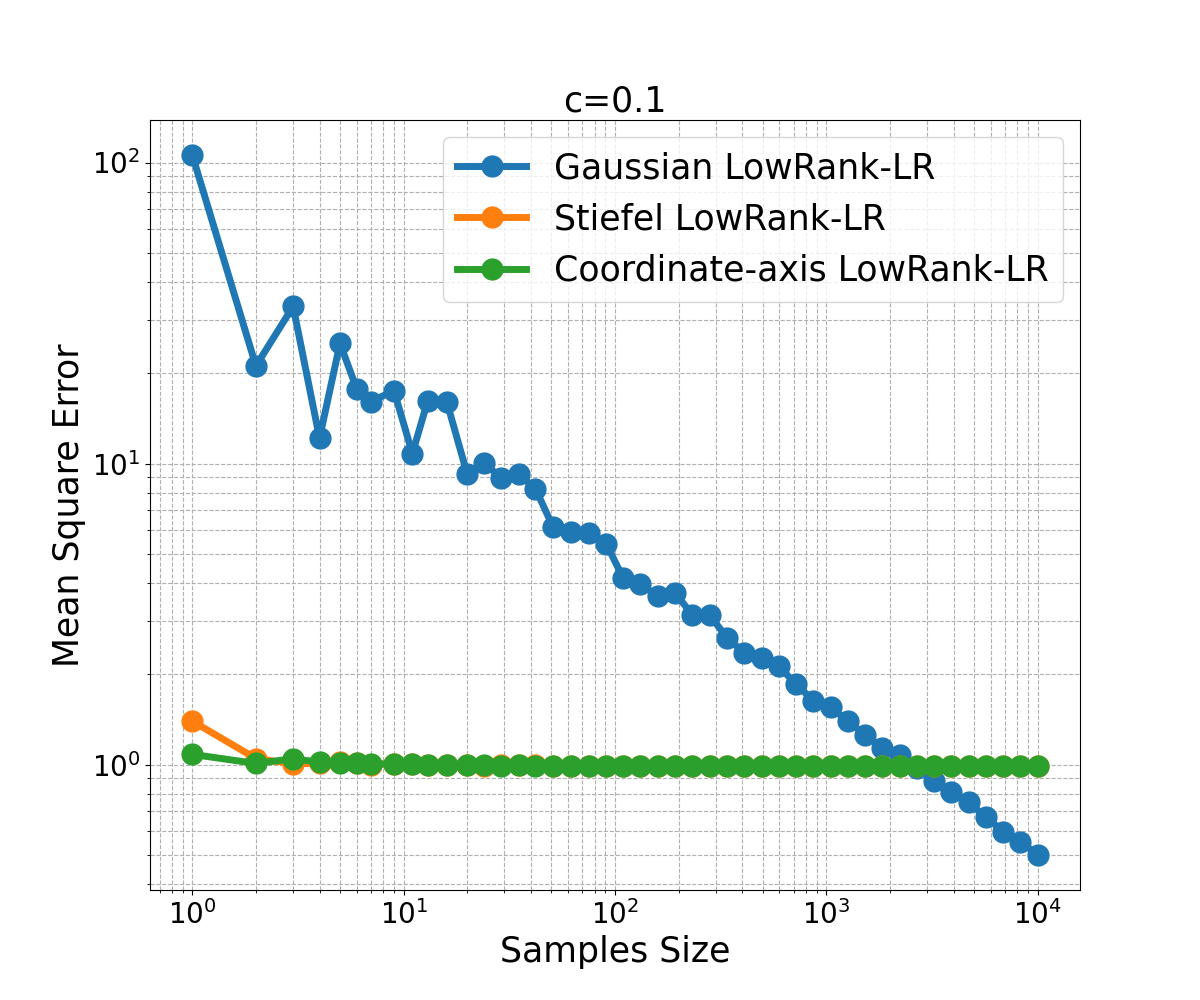}}
\subfloat[$c=0.3$]{
\includegraphics[trim=0cm 0.5cm 0cm 0cm, width=0.3\textwidth]{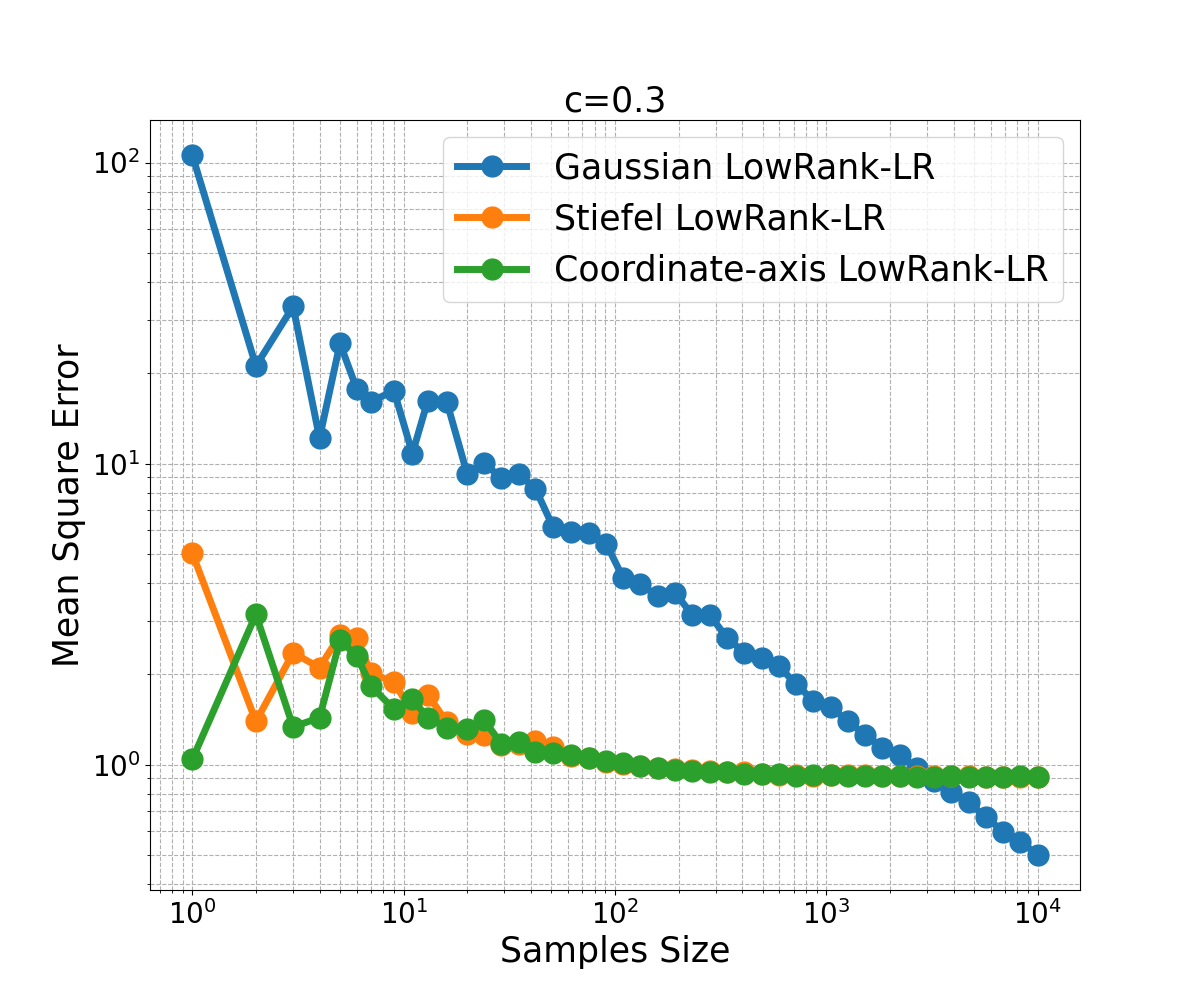}}
\subfloat[$c=0.5$]{
\includegraphics[trim=0cm 0.5cm 0cm 0cm, width=0.3\textwidth]{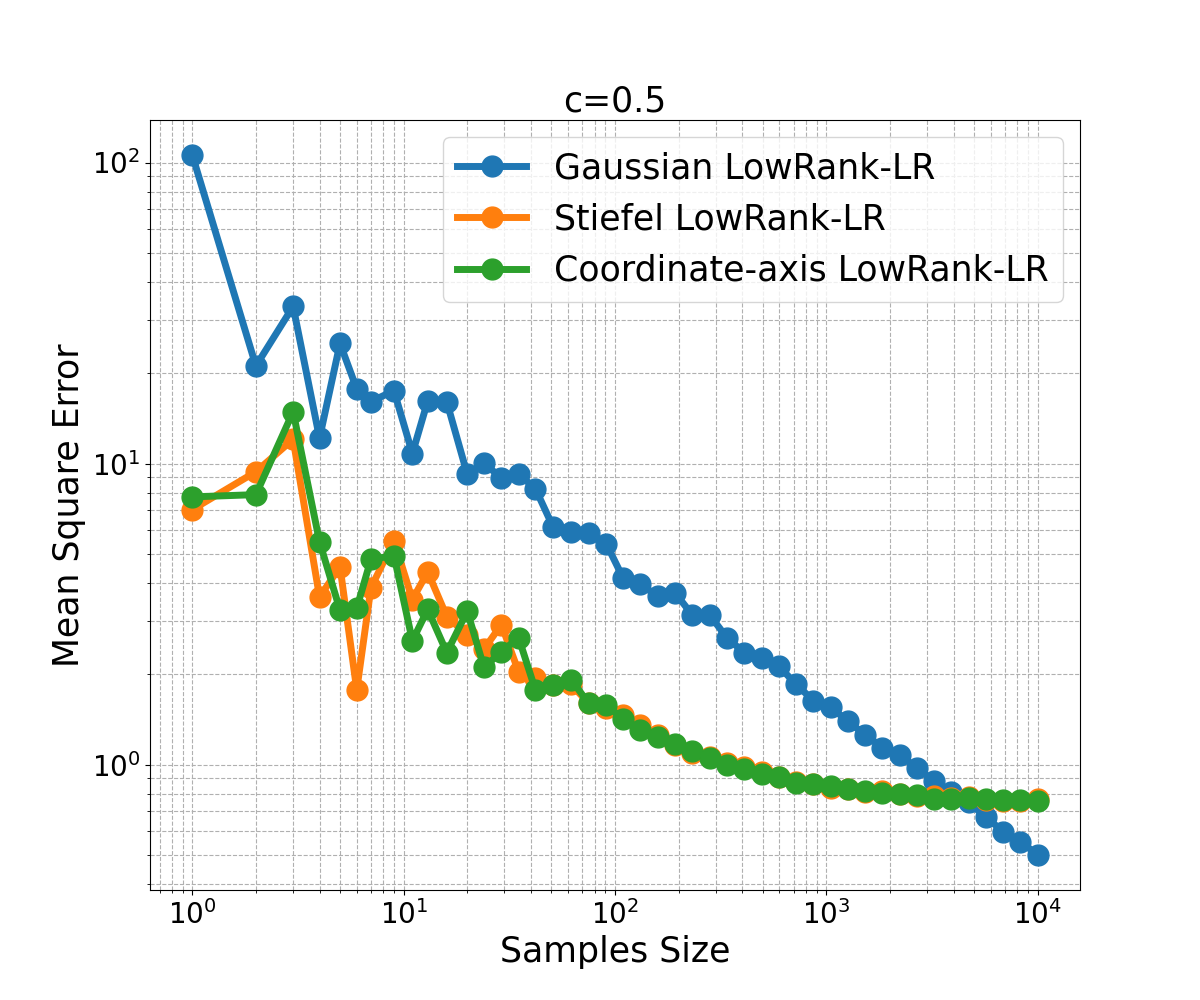}}

\subfloat[$c=0.6$]{
\includegraphics[trim=0cm 0.5cm 0cm 0cm, width=0.3\textwidth]{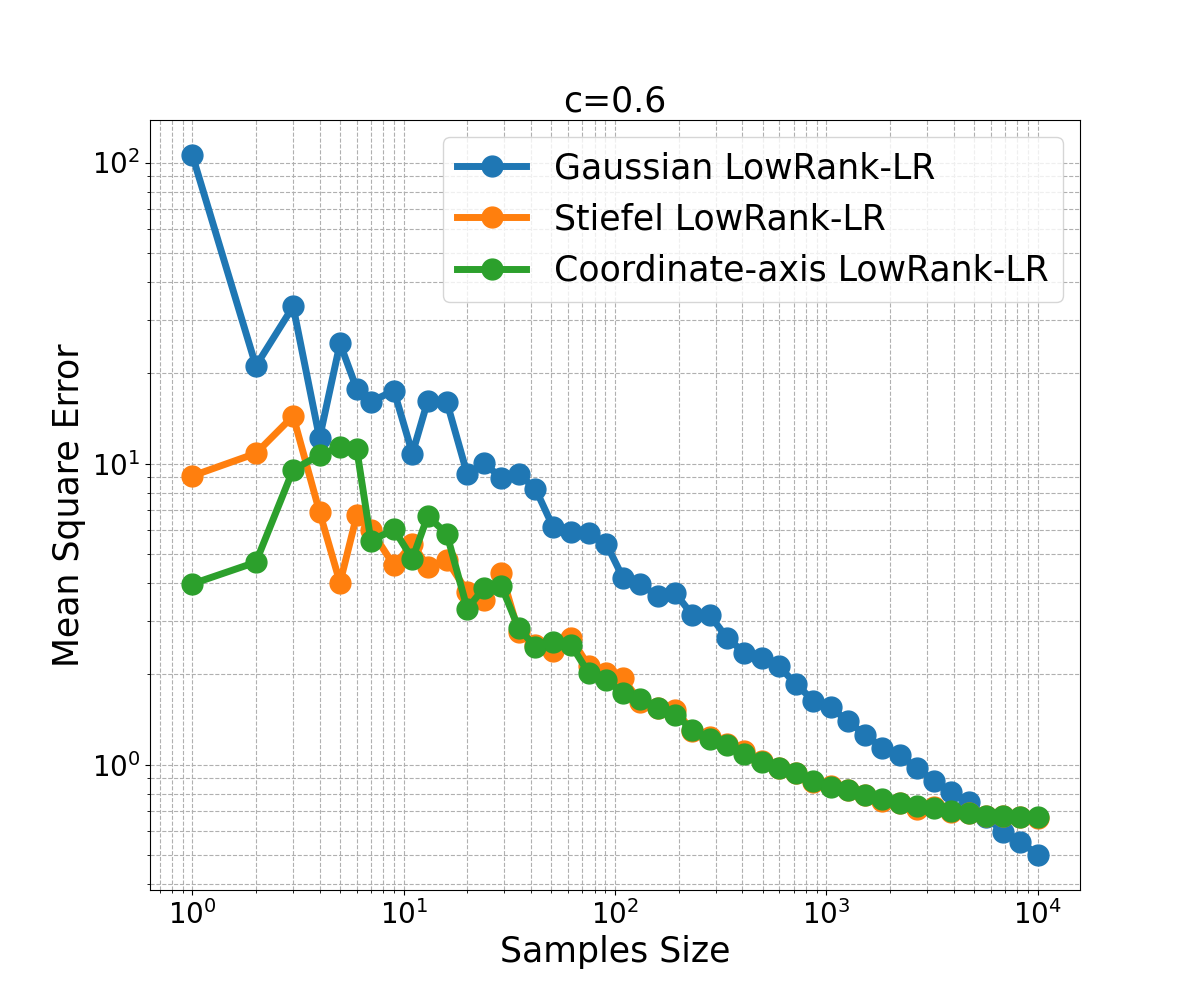}}
\subfloat[$c=0.8$]{
\includegraphics[trim=0cm 0.5cm 0cm 0cm, width=0.3\textwidth]{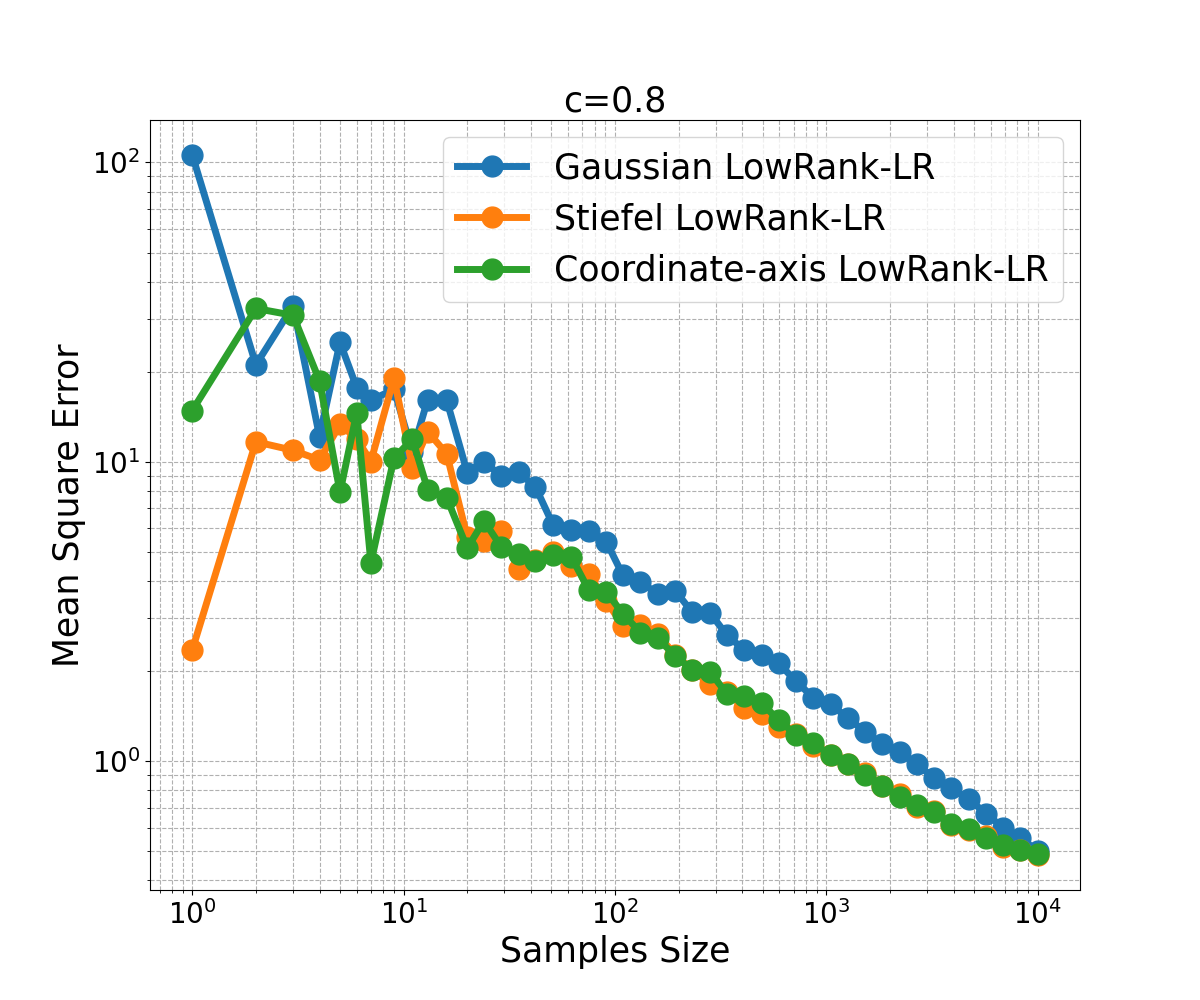}}
\subfloat[$c=1.0$]{
\includegraphics[trim=0cm 0.5cm 0cm 0cm, width=0.3\textwidth]{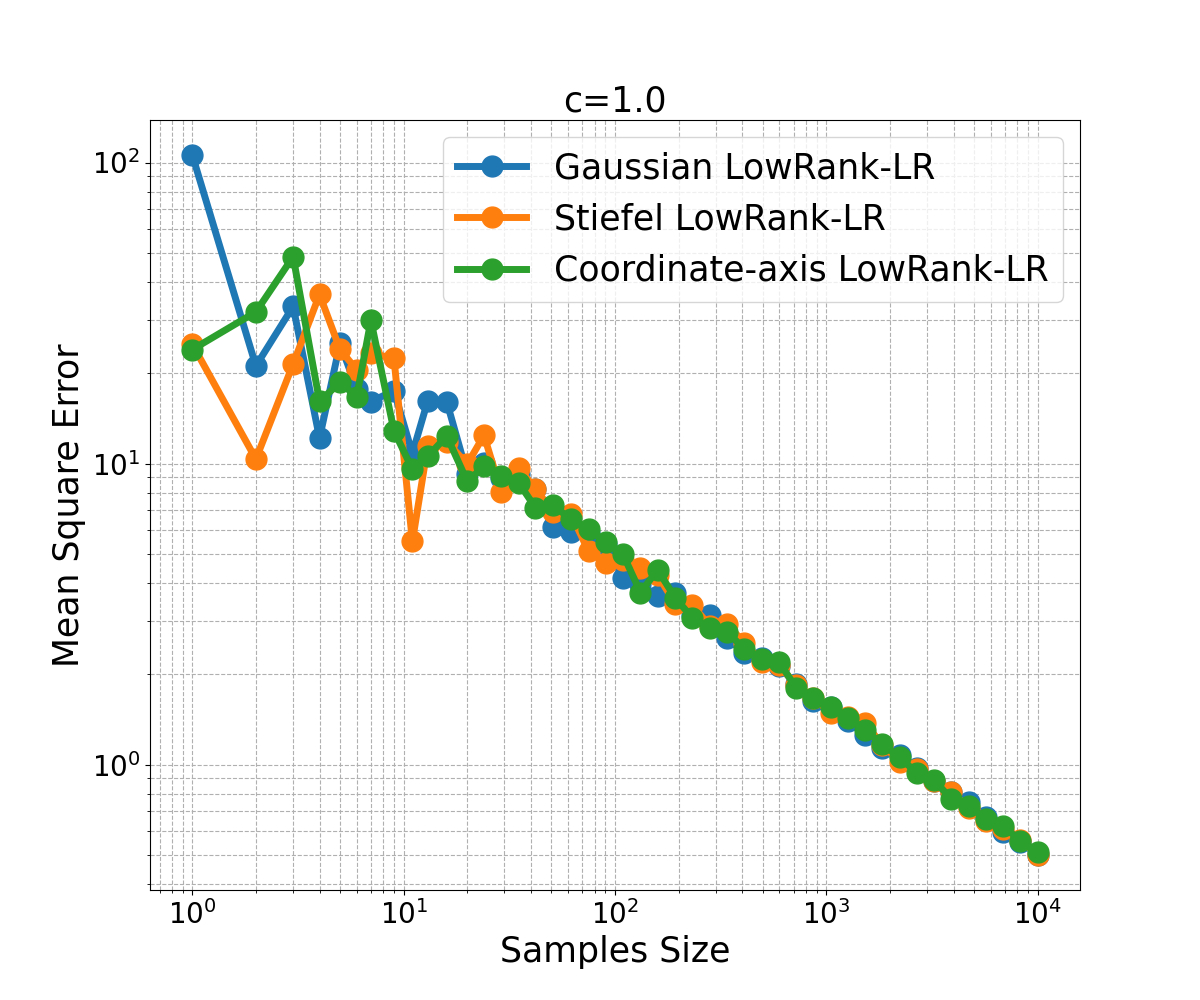}}
\vspace{-0.2cm}
\caption{\textcolor{black}{MSE versus samples plot of independent Likelihood Ratio estimator}}
\vspace{-0.3cm}
\label{fig:indpt_LR}
\end{figure}

\begin{figure}[h]
\centering
\subfloat[$c=0.1$]{
\includegraphics[trim=0cm 0.5cm 0cm 0cm, width=0.3\textwidth]{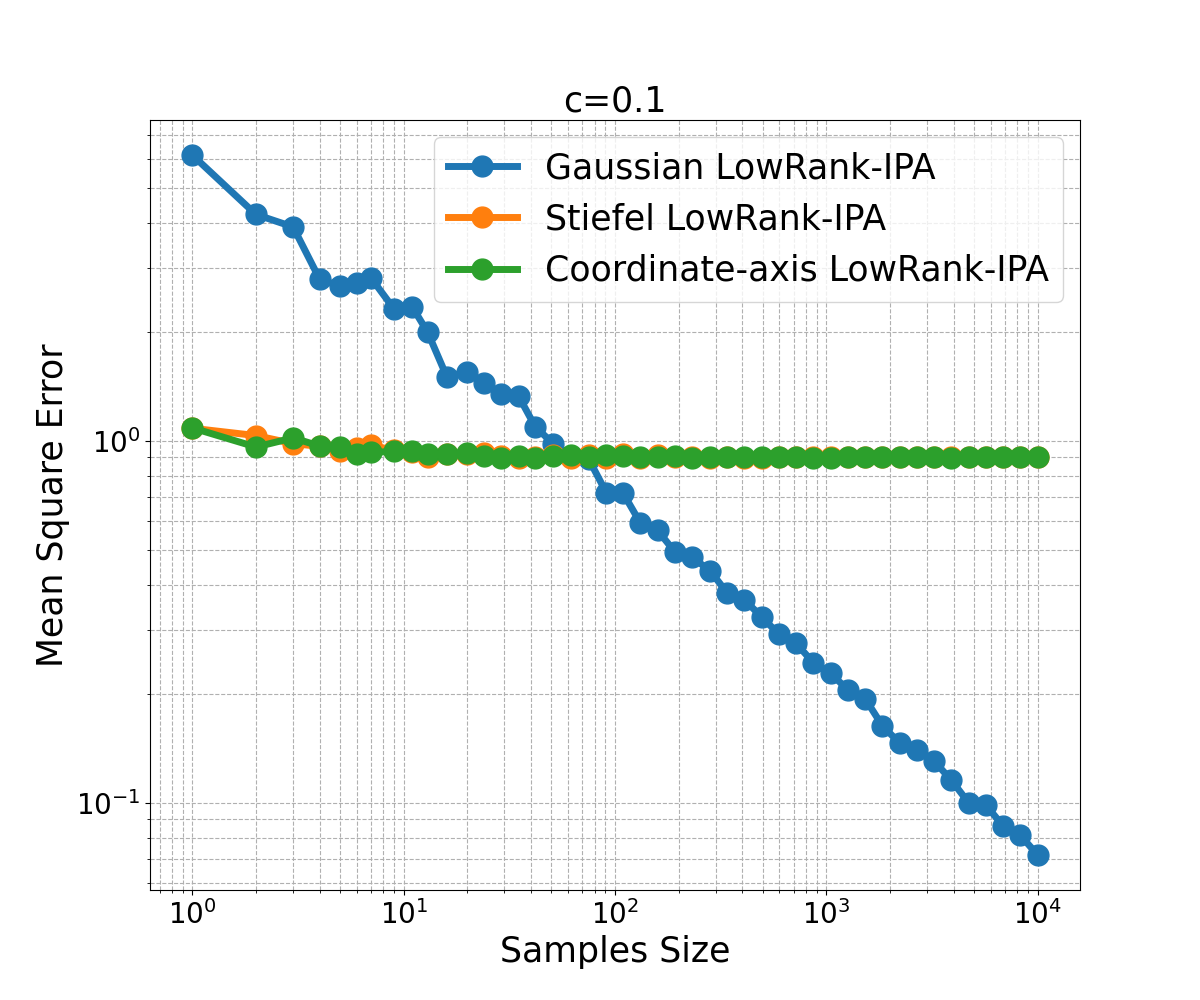}}
\subfloat[$c=0.3$]{
\includegraphics[trim=0cm 0.5cm 0cm 0cm, width=0.3\textwidth]{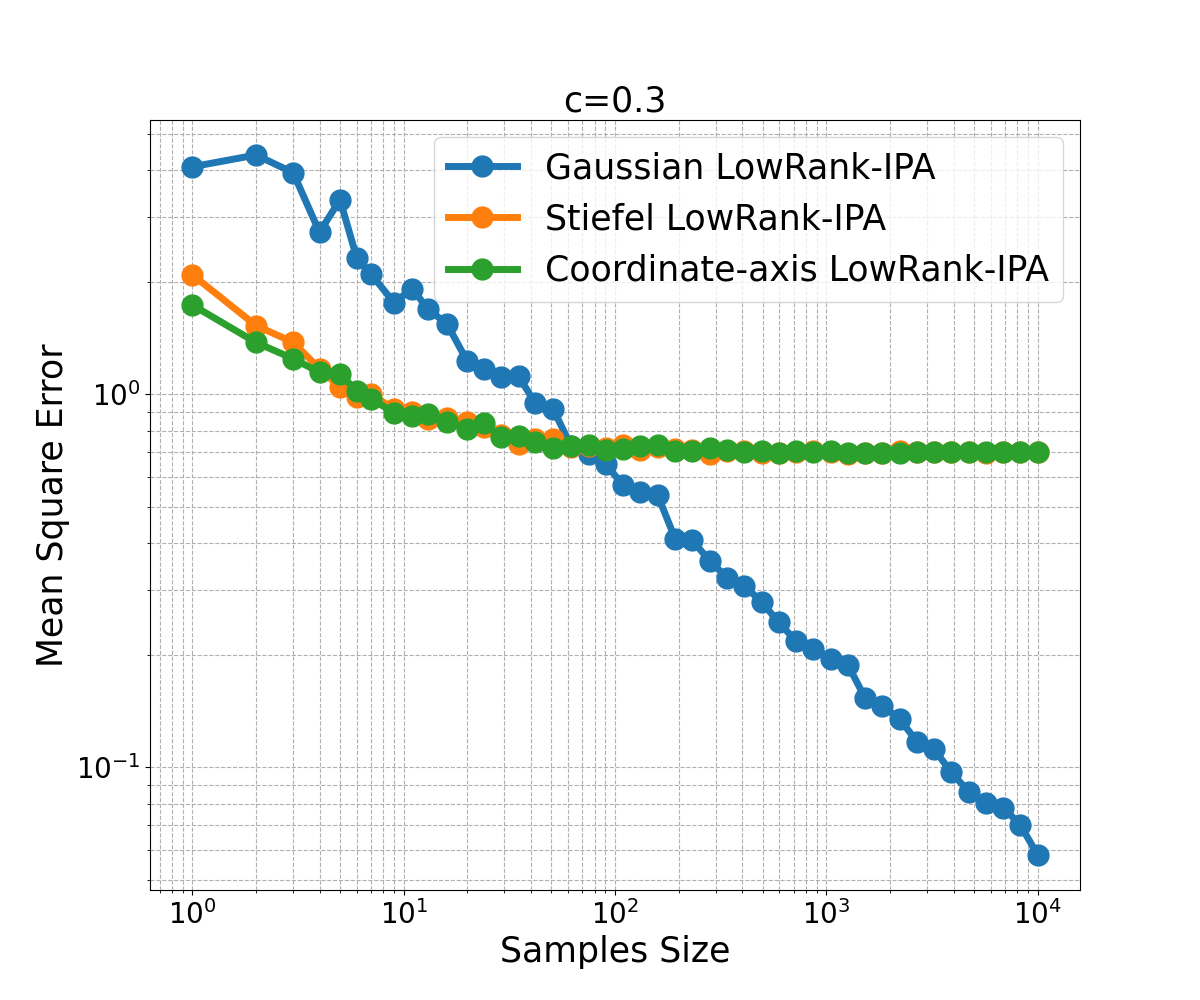}}
\subfloat[$c=0.5$]{
\includegraphics[trim=0cm 0.5cm 0cm 0cm, width=0.3\textwidth]{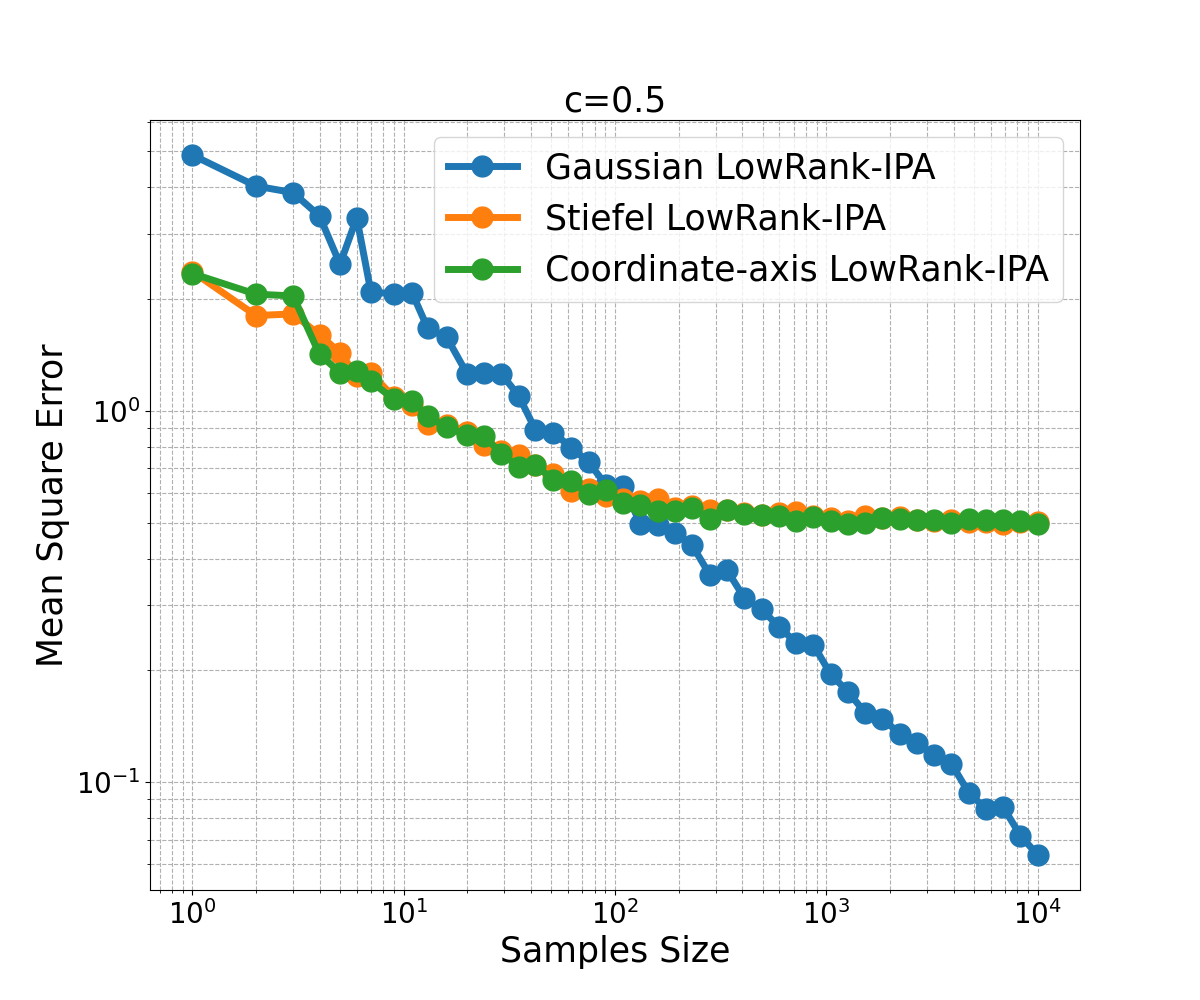}}

\subfloat[$c=0.6$]{
\includegraphics[trim=0cm 0.5cm 0cm 0cm, width=0.3\textwidth]{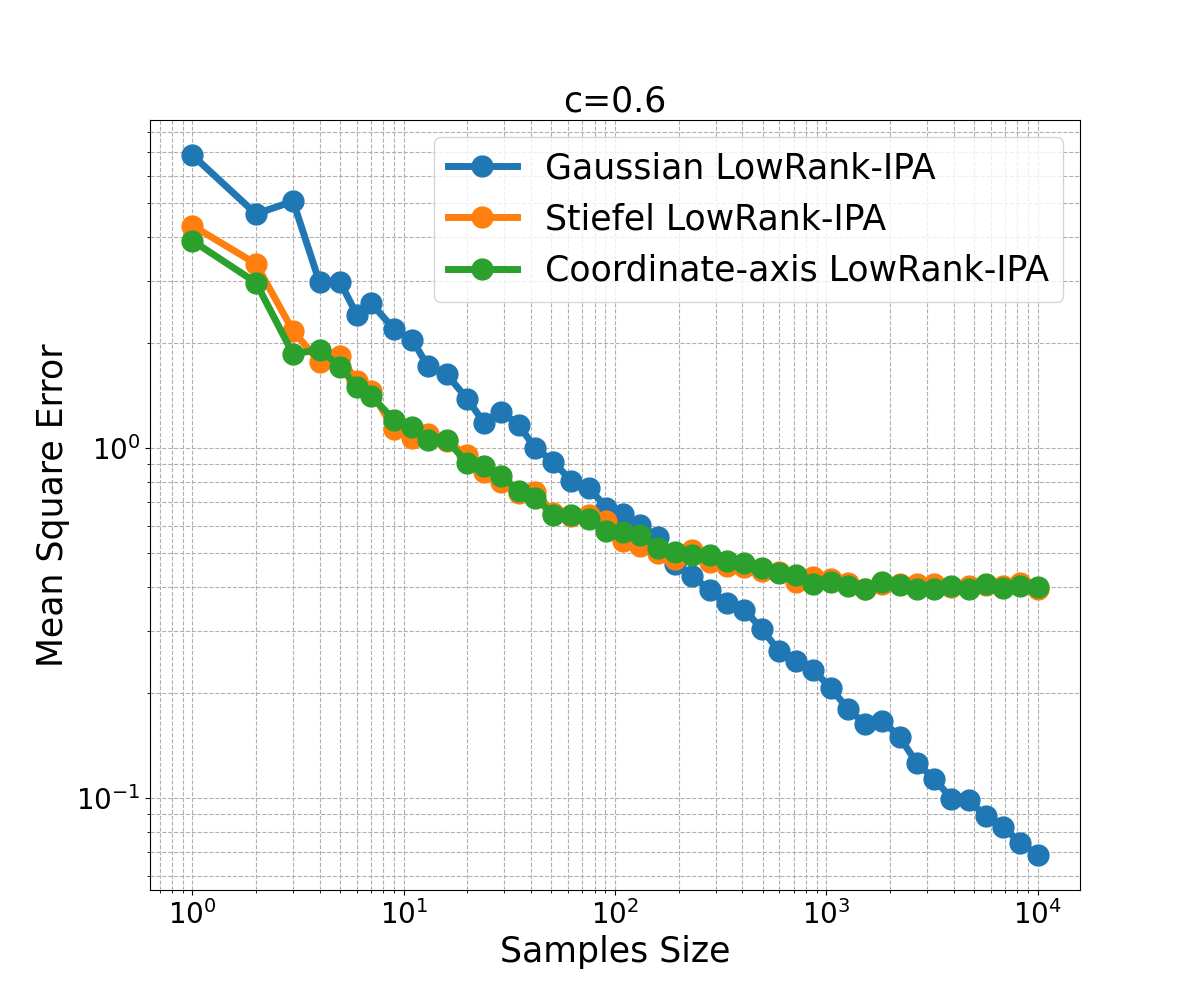}}
\subfloat[$c=0.8$]{
\includegraphics[trim=0cm 0.5cm 0cm 0cm, width=0.3\textwidth]{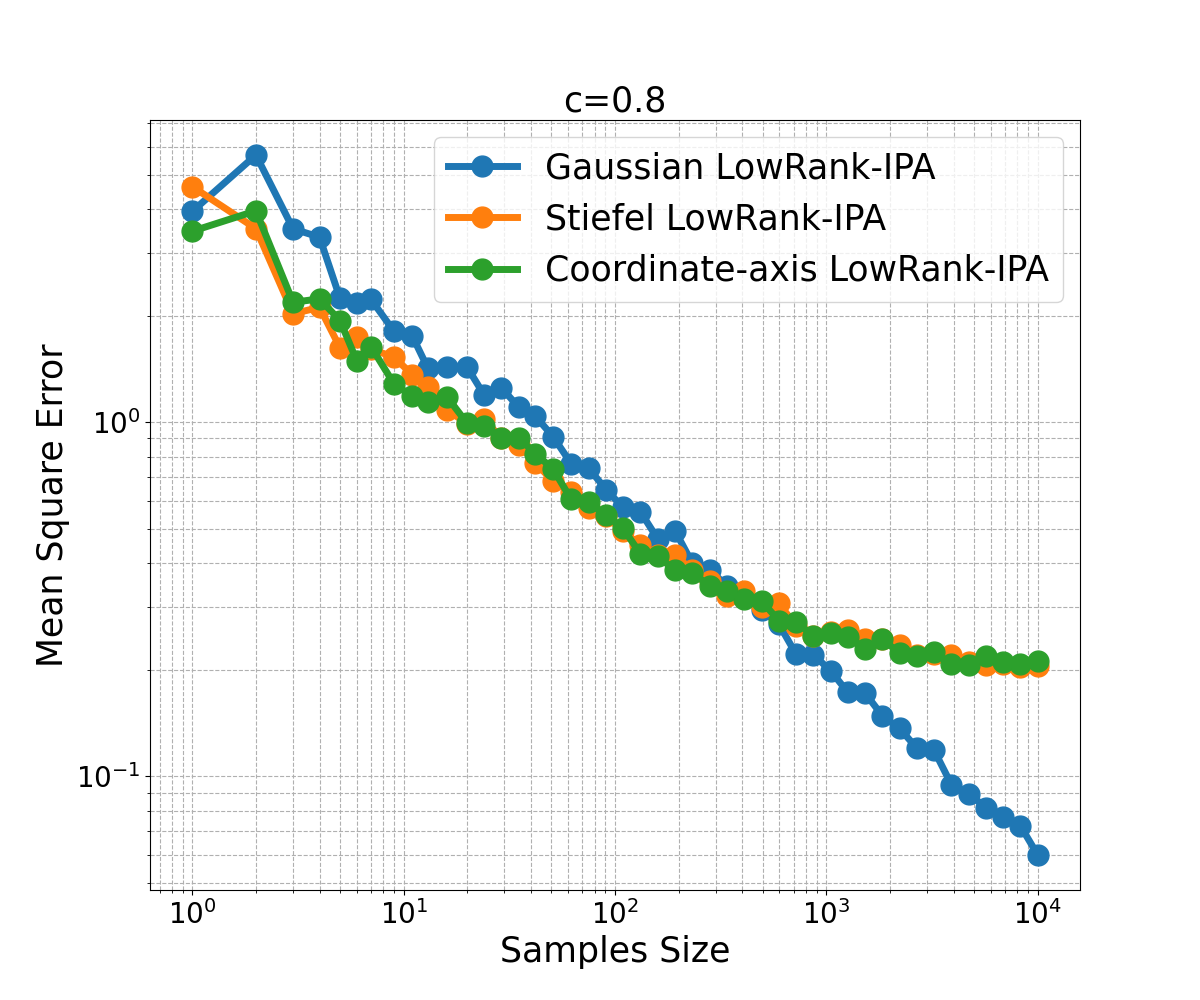}}
\subfloat[$c=1.0$]{
\includegraphics[trim=0cm 0.5cm 0cm 0cm, width=0.3\textwidth]{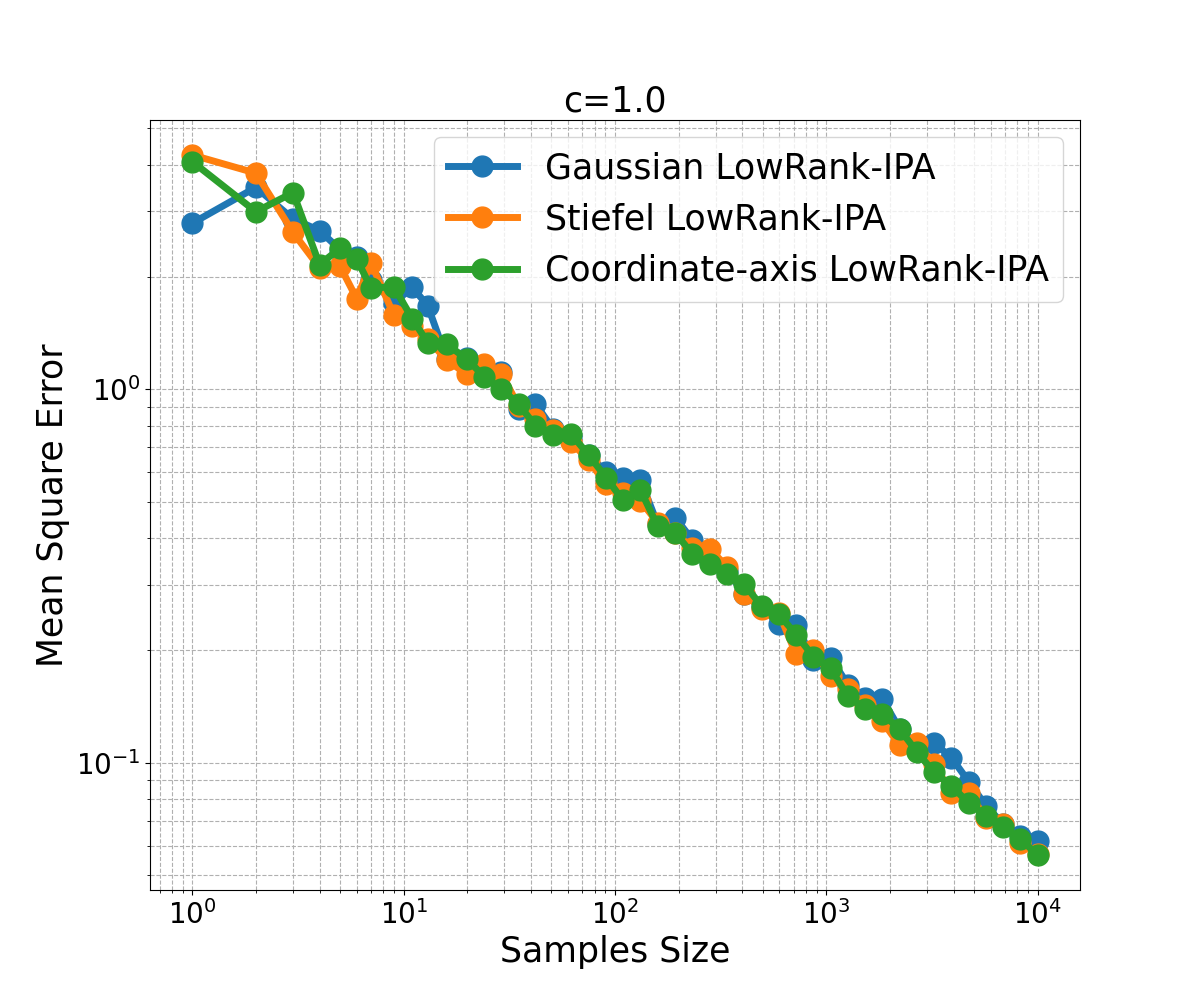}}
\vspace{-0.2cm}
\caption{\textcolor{black}{MSE versus samples plot of independent Infinitesimal Perturbation Analysis estimator}}
\vspace{-0.3cm}
\label{fig:indpt_IPA}
\end{figure}

We first investigate the independent setting to validate the bias-variance trade-off associated with the LowRank-LR and LowRank-IPA gradient estimators. Corresponding results are presented in Figure \ref{fig:indpt_LR} and Figure \ref{fig:indpt_IPA}, where we report the MSE across varying sample sizes and values of the scaling parameter $c$.

A clear bias-variance trade-off emerges as a function of $c$. When $c$ is set near $1.0$, the overall MSE is variance-dominated, meaning the estimator is unbiased and variance contributes to the MSE. Ideally, the MSE should decay toward zero as the sample size increases. In contrast, when $c$ approaches $0.1$, the MSE becomes bias-dominated, with bias being the primary source of estimation error. The independent low-rank estimator has a smaller MSE than the Gaussian one when the sample sizes are small. When the sample size increases, the independent estimator fails to decrease its MSE because of the bias. Specifically, increasing $c$ reduces estimation bias but amplifies variance, while decreasing $c$ suppresses variance at the cost of increased bias. The results validate Theorem \ref{thm2} and Remark \ref{remark1}, showing that our methods decrease the variance of the one-shot gradient estimator.

\begin{figure}[h]
\centering
\subfloat[$c=0.1$]{
\includegraphics[trim=0cm 0.5cm 0cm 0cm, width=0.3\textwidth]{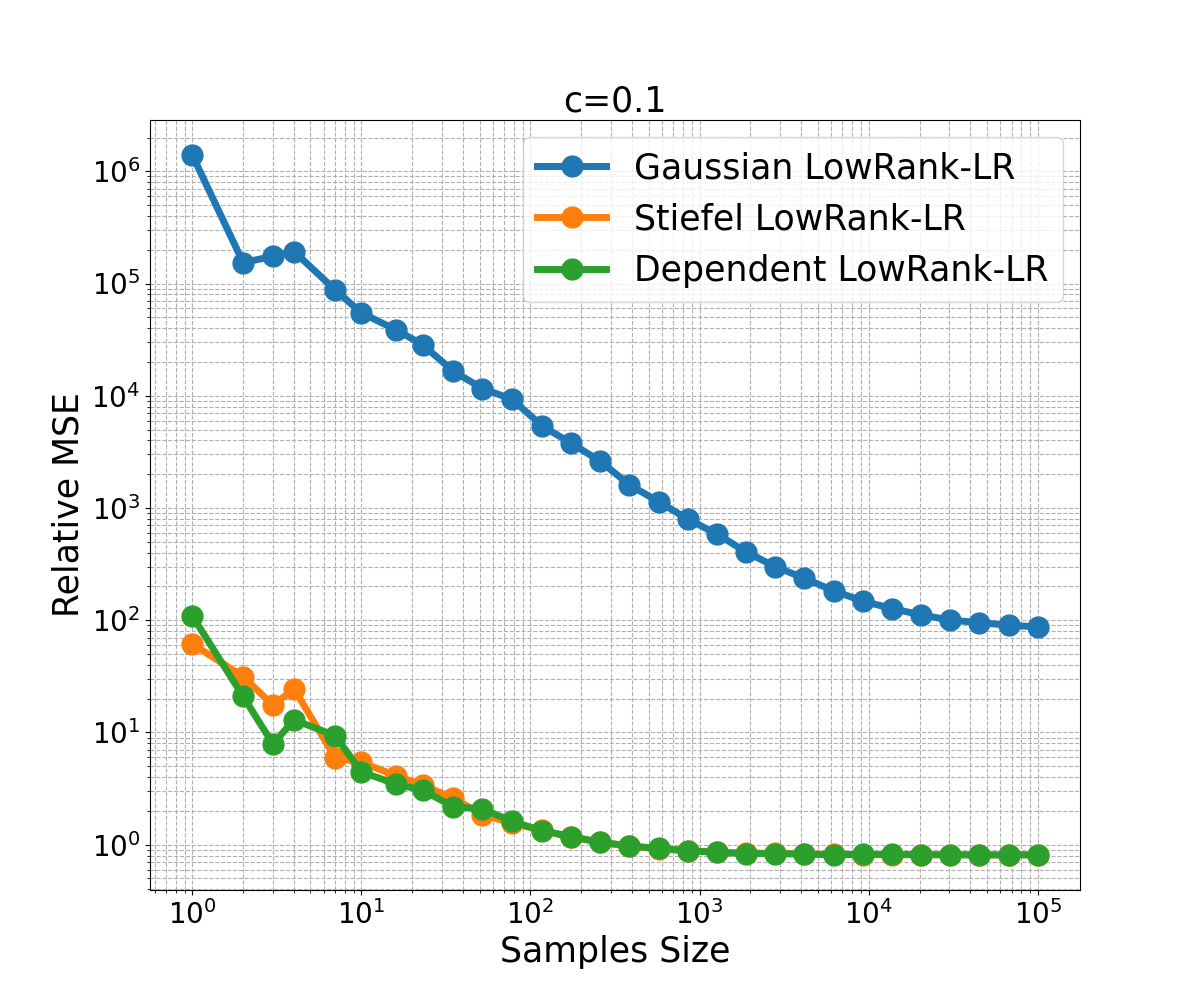}}
\subfloat[$c=0.3$]{
\includegraphics[trim=0cm 0.5cm 0cm 0cm, width=0.3\textwidth]{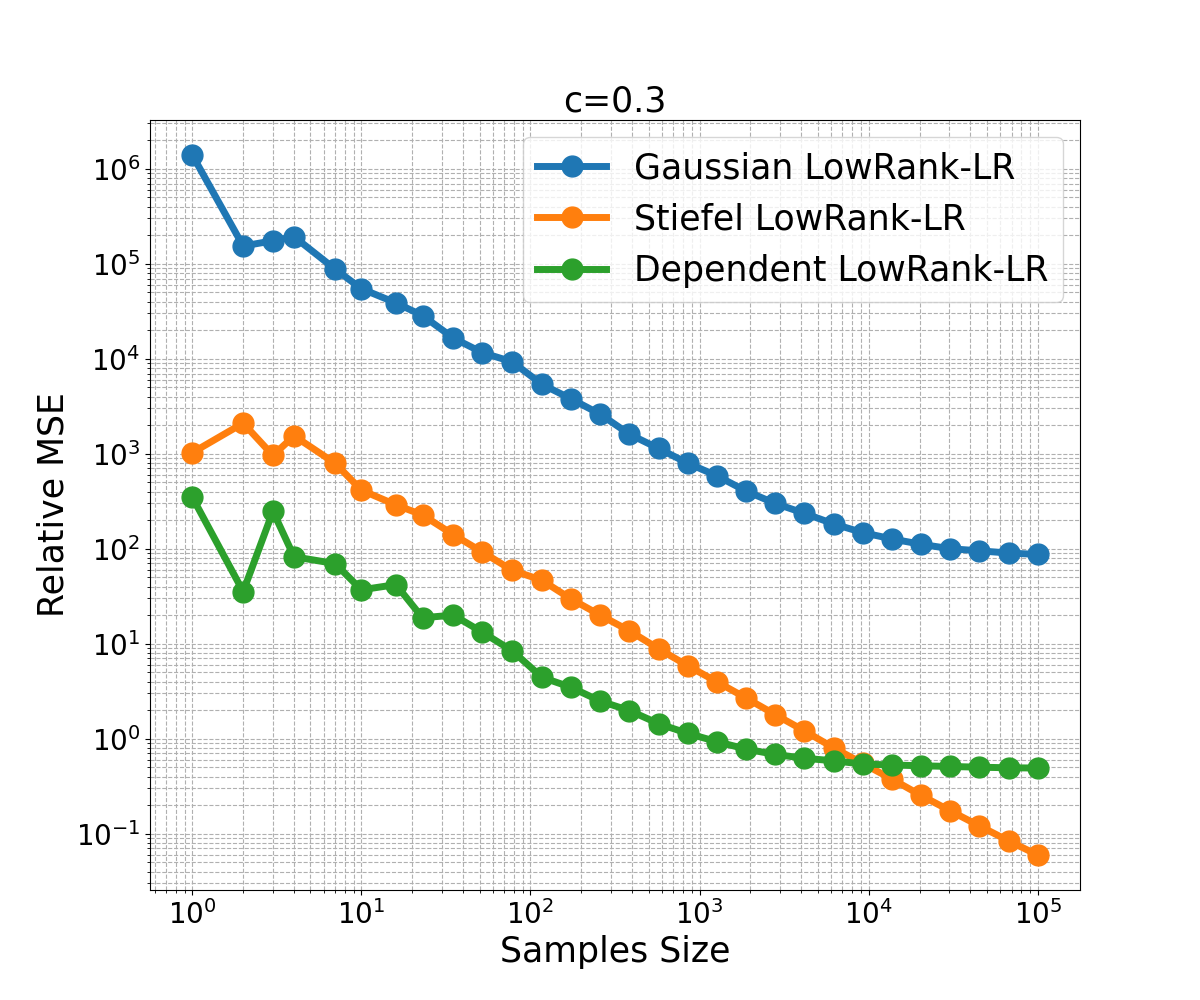}}
\subfloat[$c=0.5$]{
\includegraphics[trim=0cm 0.5cm 0cm 0cm, width=0.3\textwidth]{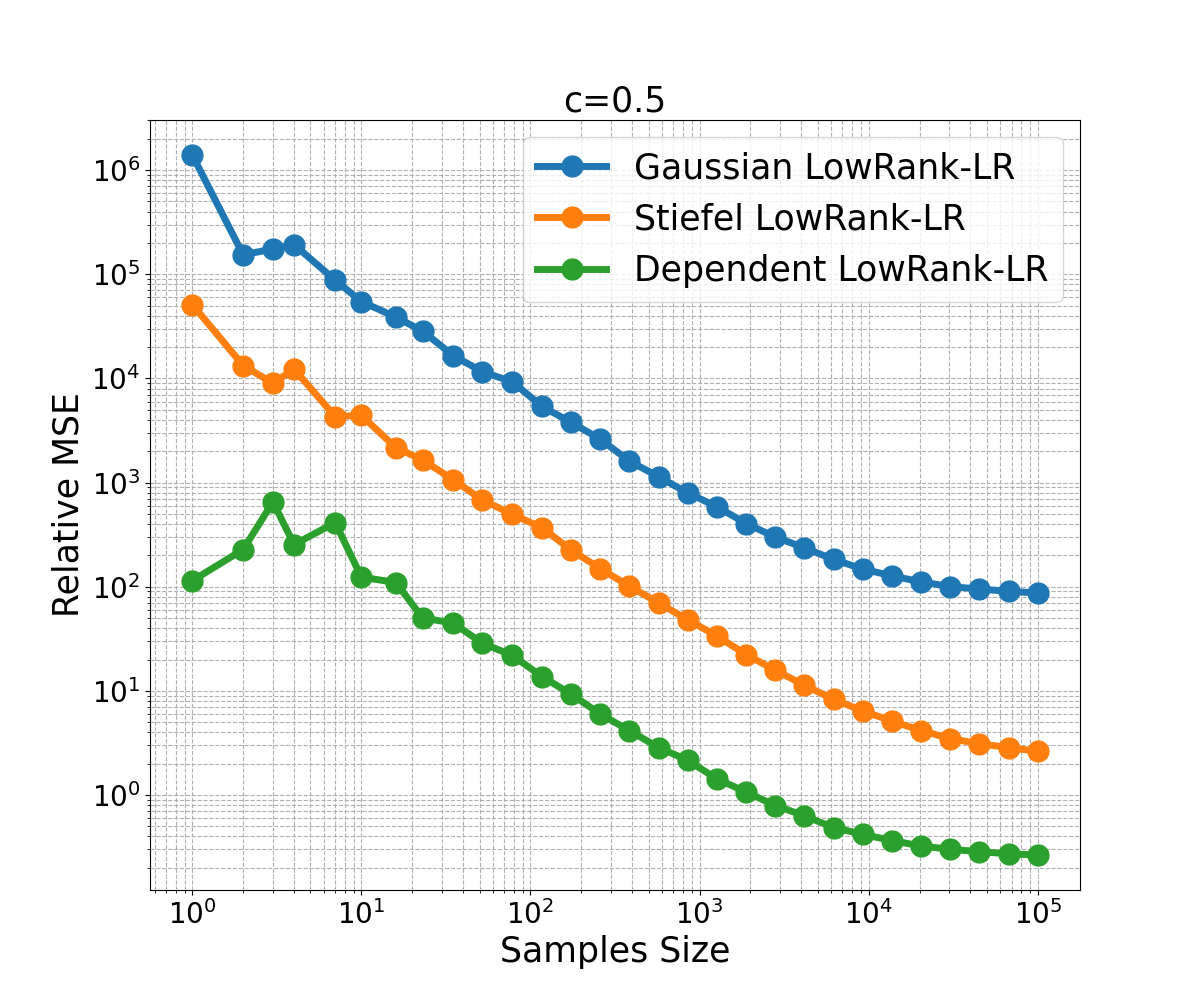}}

\subfloat[$c=0.6$]{
\includegraphics[trim=0cm 0.5cm 0cm 0cm, width=0.3\textwidth]{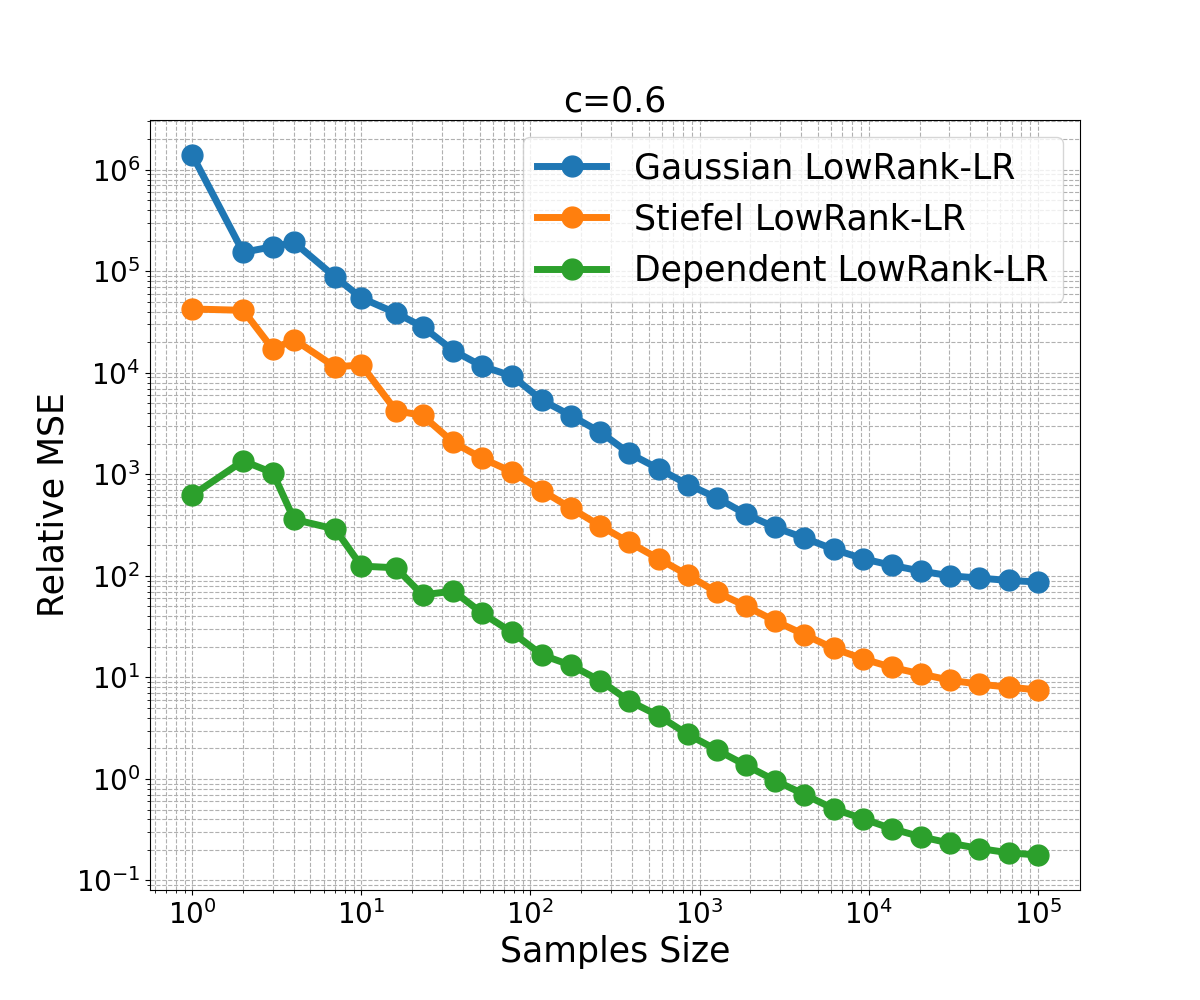}}
\subfloat[$c=0.8$]{
\includegraphics[trim=0cm 0.5cm 0cm 0cm, width=0.3\textwidth]{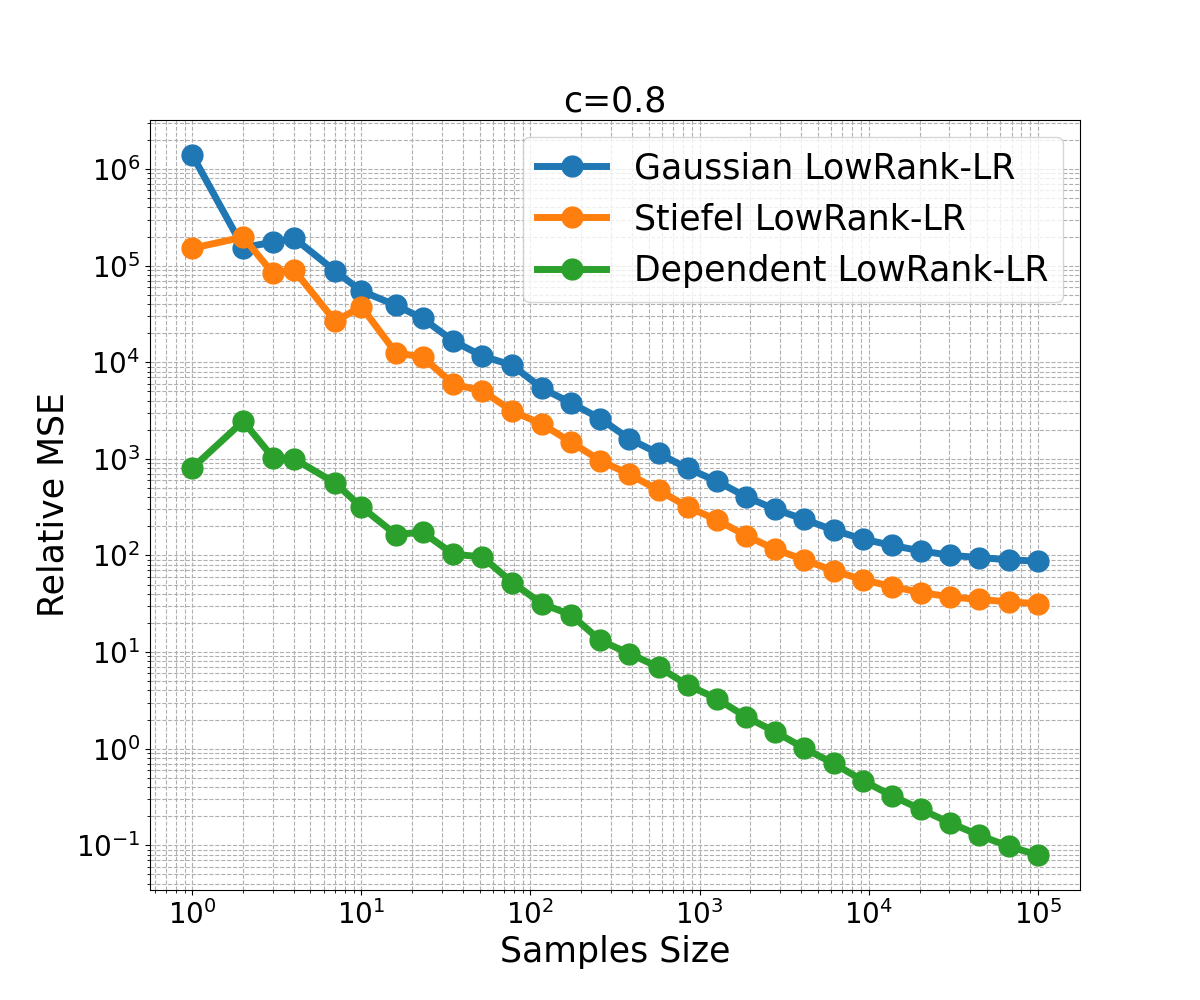}}
\subfloat[$c=1.0$]{
\includegraphics[trim=0cm 0.5cm 0cm 0cm, width=0.3\textwidth]{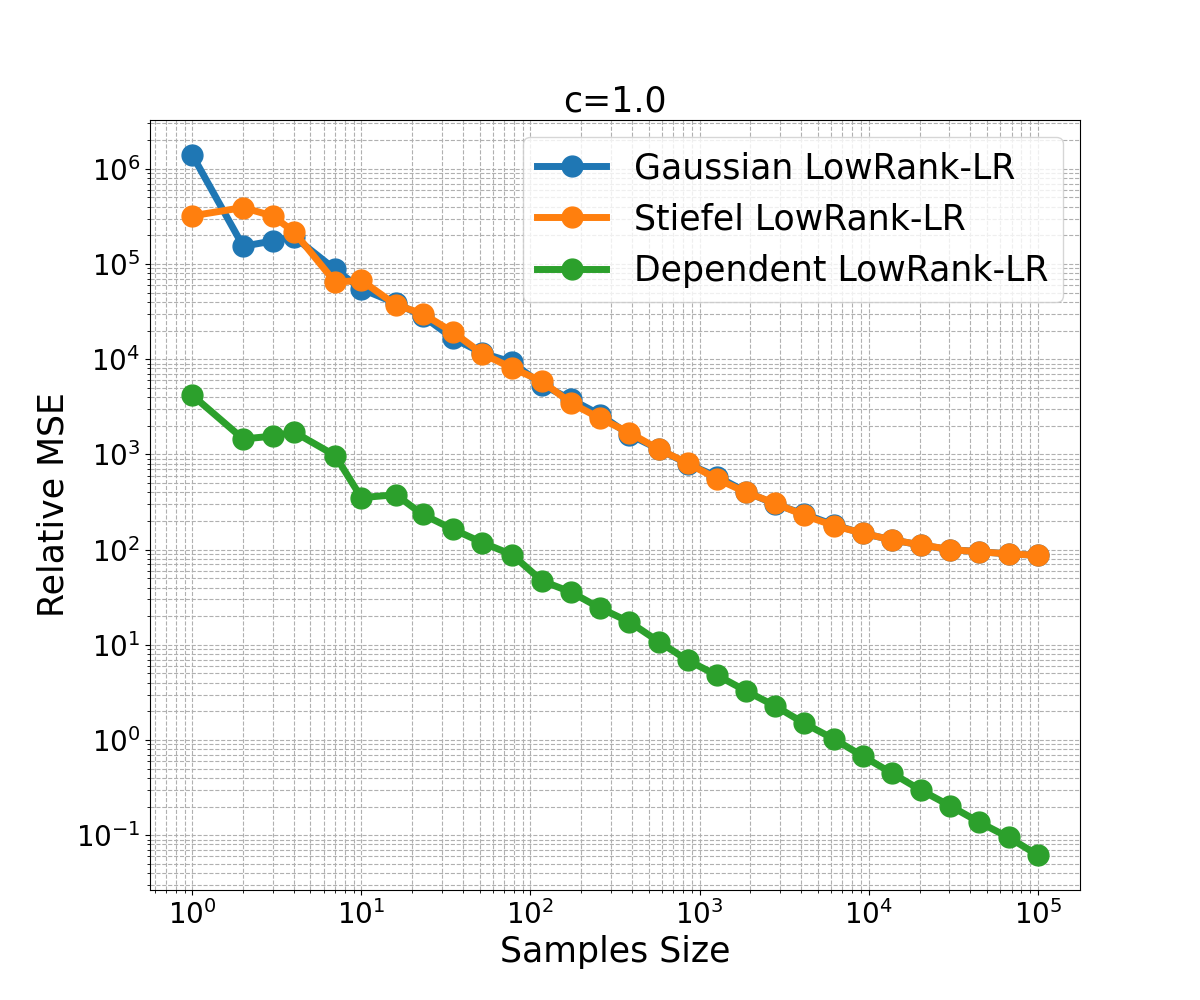}}
\vspace{-0.2cm}
\caption{\textcolor{black}{MSE versus samples plot of dependent Likelihood Ratio (LR) estimator}}
\vspace{-0.3cm}
\label{fig:dpt_LR}
\end{figure}

\begin{figure}[h]
\centering
\subfloat[$c=0.1$]{
\includegraphics[trim=0cm 0.5cm 0cm 0cm, width=0.3\textwidth]{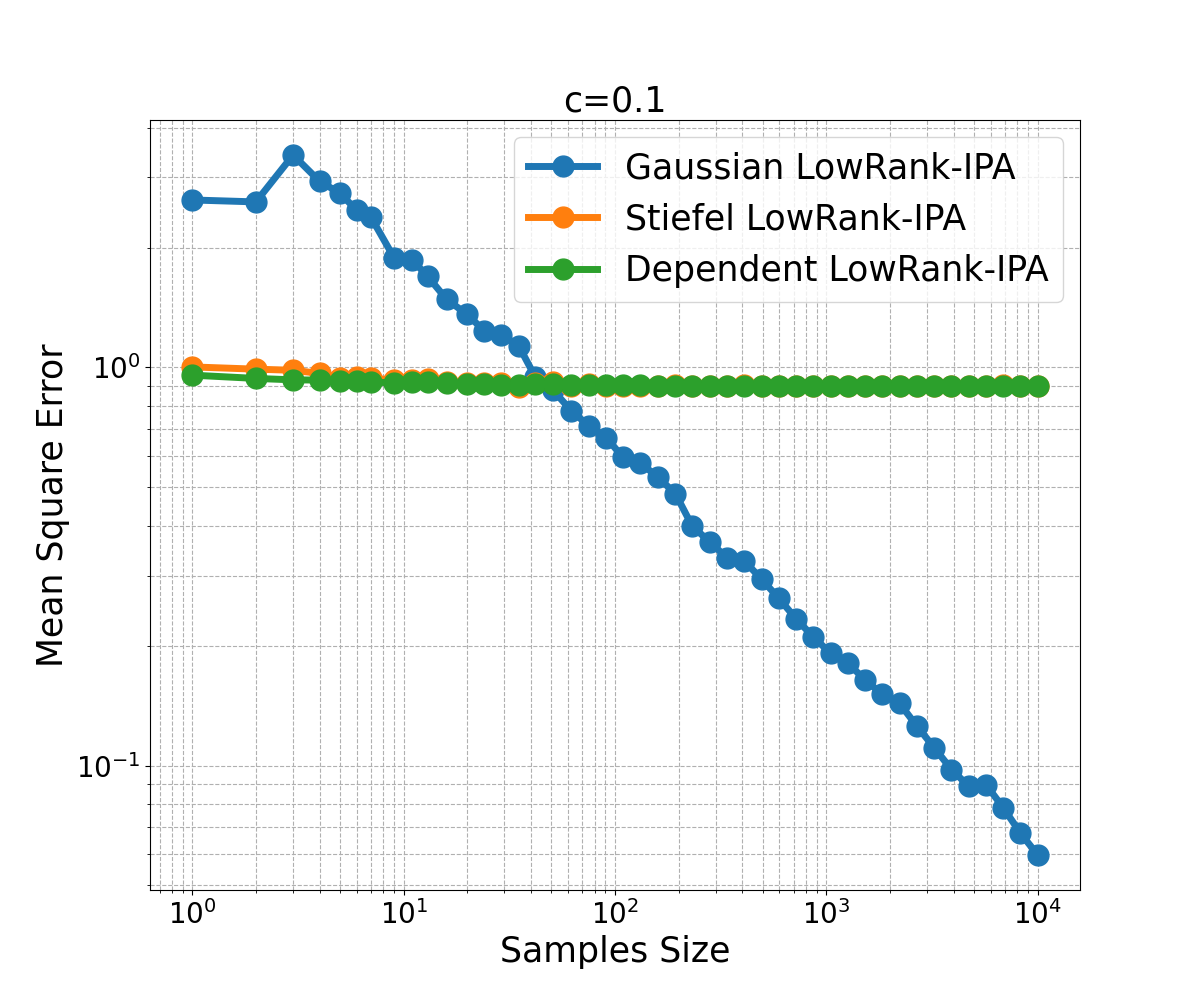}}
\subfloat[$c=0.3$]{
\includegraphics[trim=0cm 0.5cm 0cm 0cm, width=0.3\textwidth]{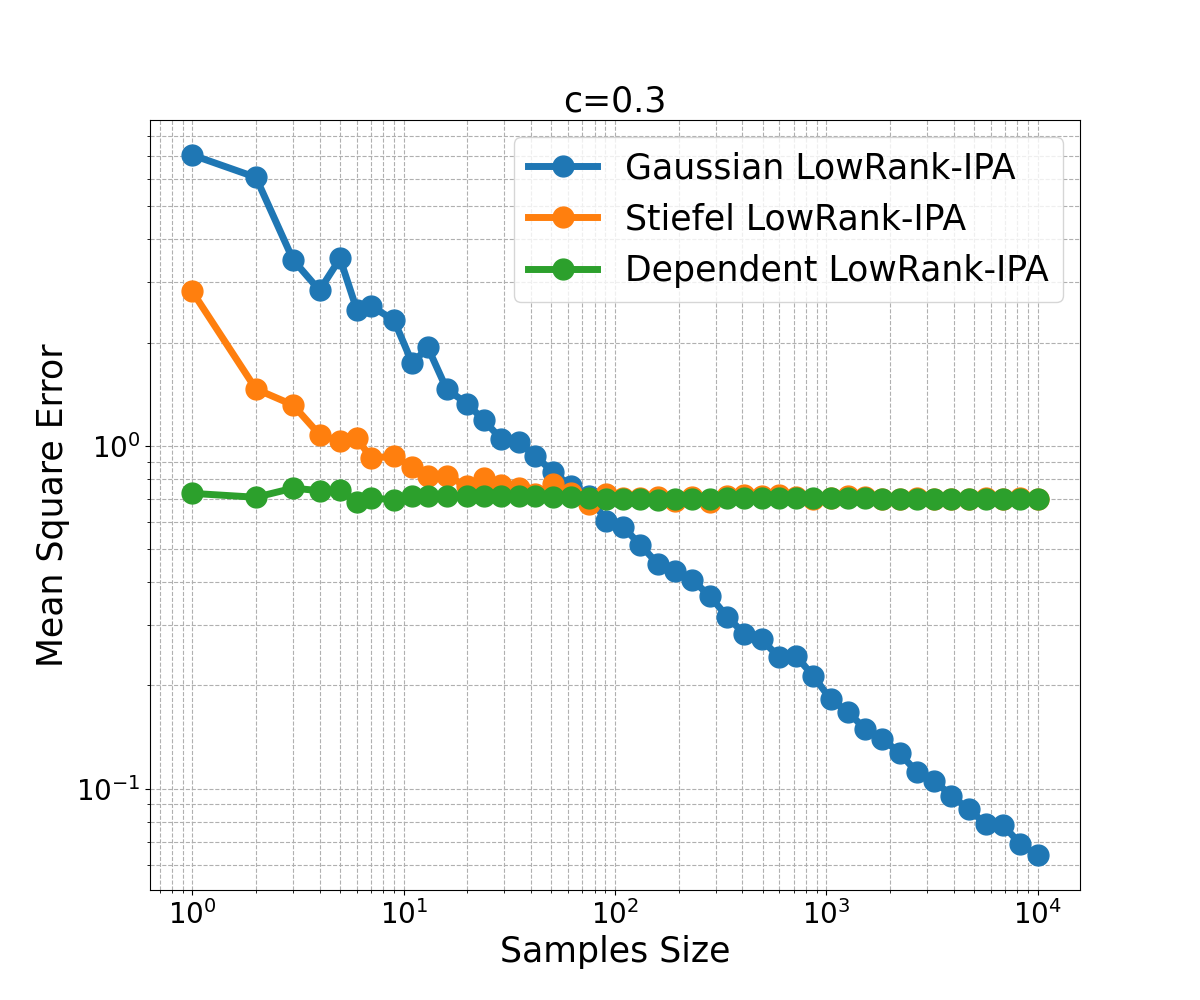}}
\subfloat[$c=0.5$]{
\includegraphics[trim=0cm 0.5cm 0cm 0cm, width=0.3\textwidth]{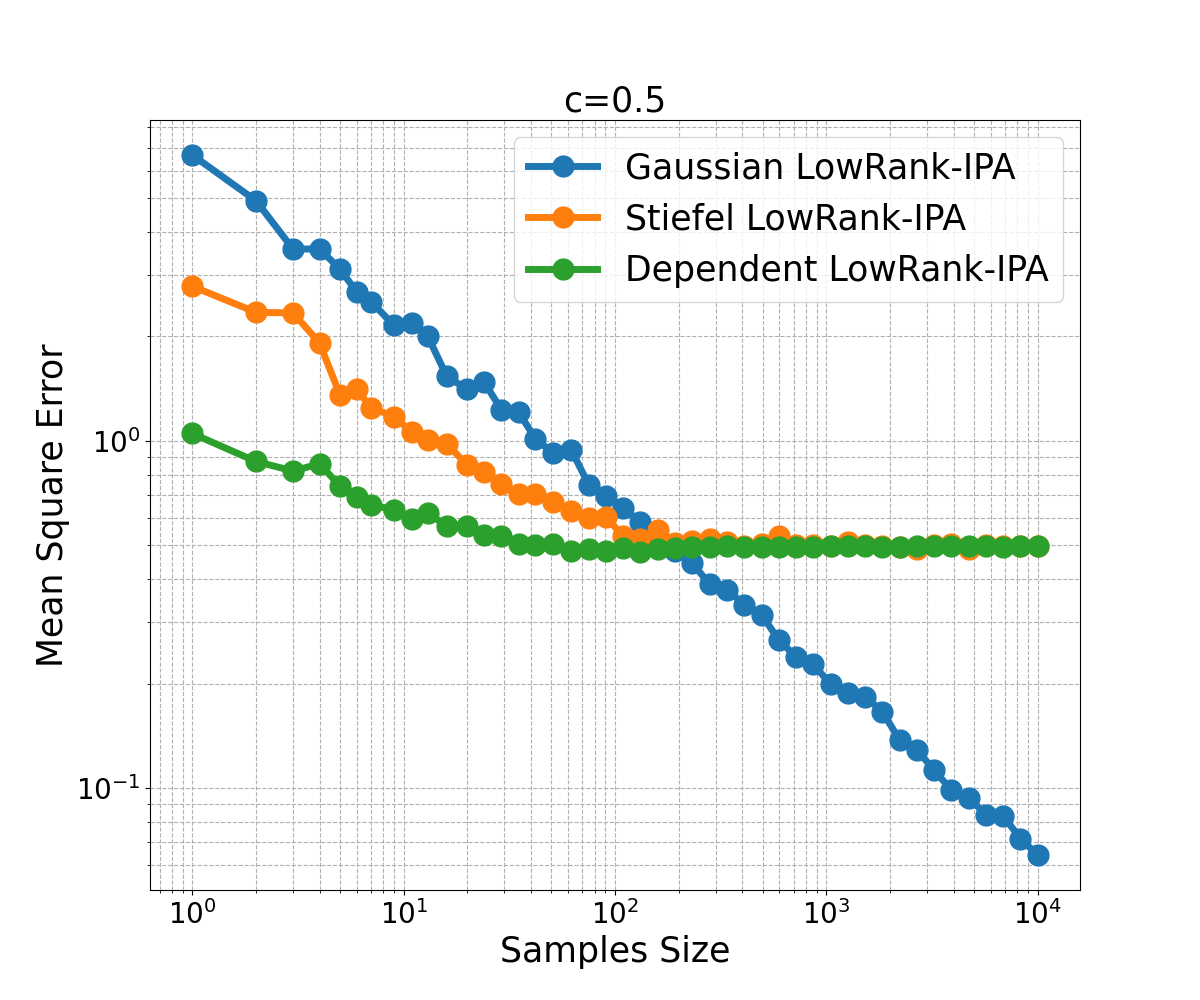}}

\subfloat[$c=0.6$]{
\includegraphics[trim=0cm 0.5cm 0cm 0cm, width=0.3\textwidth]{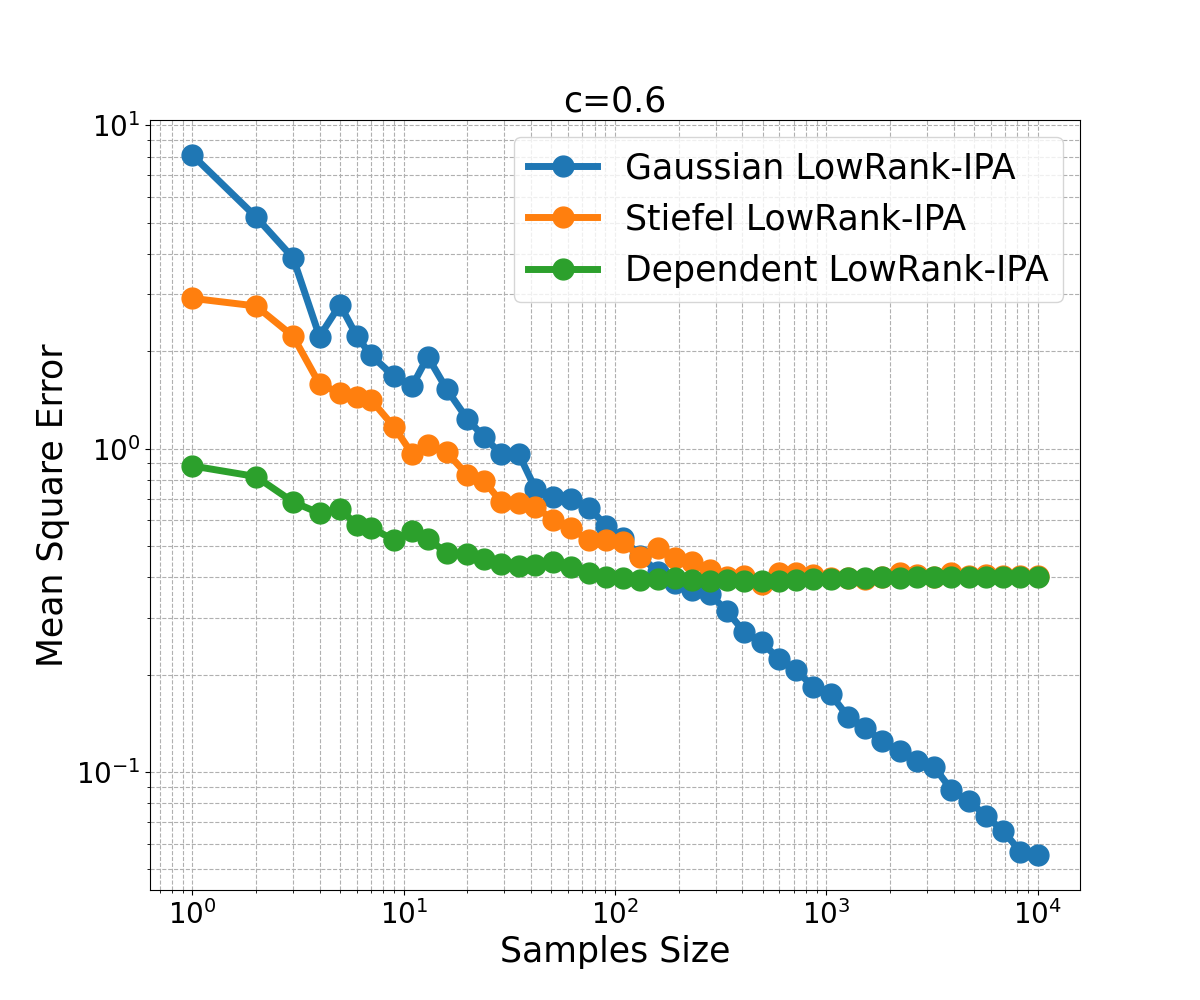}}
\subfloat[$c=0.8$]{
\includegraphics[trim=0cm 0.5cm 0cm 0cm, width=0.3\textwidth]{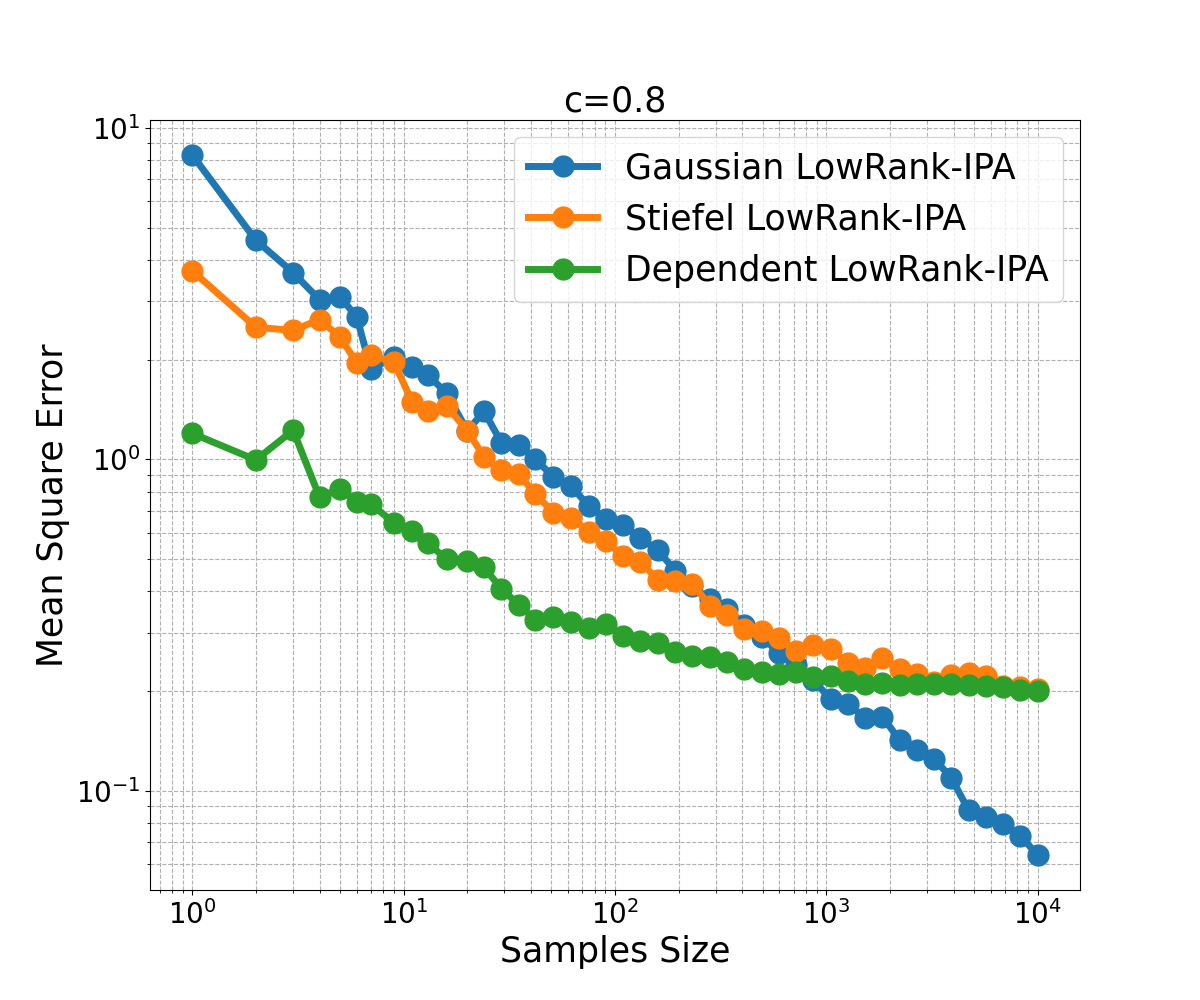}}
\subfloat[$c=1.0$]{
\includegraphics[trim=0cm 0.5cm 0cm 0cm, width=0.3\textwidth]{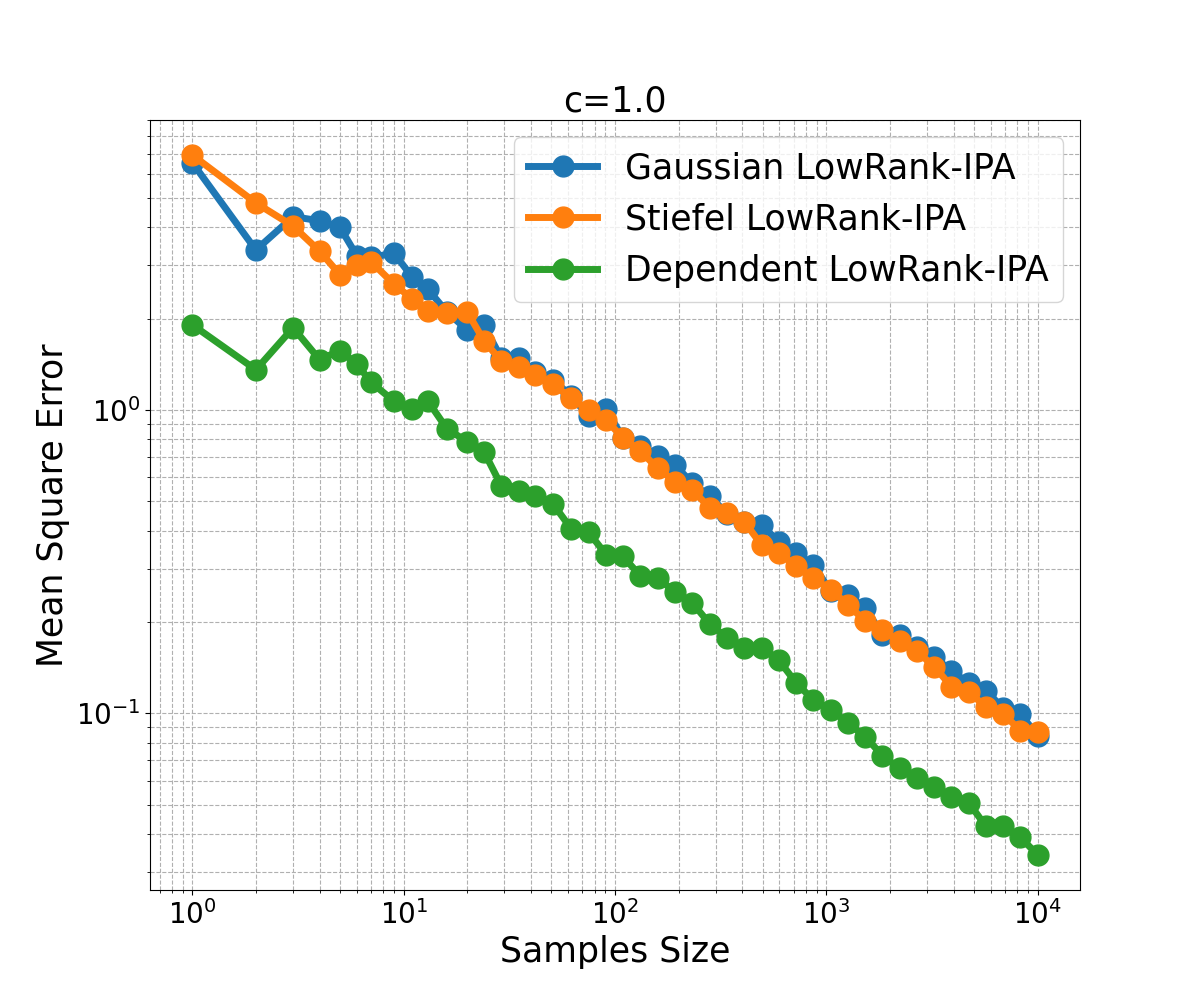}}
\vspace{-0.2cm}
\caption{\textcolor{black}{MSE versus samples plot of dependent Infinitesimal Perturbation Analysis (IPA) estimator}}
\vspace{-0.3cm}
\label{fig:dpt_IPA}
\end{figure}

We next study the dependent gradient estimator setting. This example allows MSE of the gradient estimator to be characterized explicitly, since the true gradient is available in closed form. We therefore implement the dependent subspace sampling rule in Proposition \ref{prop:sampler_optimal} for the low-rank gradient estimator and use it to provide a clear empirical validation of the theory.

In Figures \ref{fig:dpt_LR} and \ref{fig:dpt_IPA}, we report the MSE as the sample size increases. Across both the LR and IPA settings, the dependent low-rank estimator consistently achieves lower MSE than its instance-independent counterpart. In turn, both structured low-rank designs substantially outperform the traditional Gaussian baseline. This advantage is particularly clear in Figure \ref{fig:dpt_LR}, where the dependent LowRank-LR estimator remains uniformly below both the independent structured samplers and the vanilla Gaussian LowRank-LR across all values of $c$. Overall, these results show that incorporating instance-dependent information into the subspace design leads to a clear additional gain beyond isotropic sampling, while both improve markedly over standard Gaussian subspace sampling. This empirical pattern is fully consistent with Theorem \ref{thm:SigmaP2}, which predicts that the dependent low-rank estimator attains the minimum MSE within the admissible class.

\subsection{LLM Training}\label{sec6.2}
\subsubsection{Fine-tuning LLM  with likelihood ratio estimator.}\label{sec6.2.1}
We conduct fine-tuning on \textbf{RoBERTa-large} (355M parameters) and evaluate its performance across six standard classification datasets: \textbf{SST-2} (2 classes), \textbf{SST-5} (5 classes), \textbf{SNLI} (3 classes), \textbf{MNLI} (3 classes), \textbf{RTE} (2 classes), and \textbf{TREC} (6 classes). Our baselines include zero-shot in-context learning, full fine-tuning with the Adam optimizer (essentially IPA gradient estimation, we denote this baseline as Vanilla IPA), LowRank-IPA, and several LowRank-LR variants. All experiments use consistent hyperparameters: batch size of $64$, learning rate of $1 \times 10^{-6}$, lazy update interval of $50$ steps, and rank $4$ for low-rank perturbations. 

As shown in Table \ref{tab:zo_acc}, Vanilla IPA achieves the strongest performance overall, which is expected since it updates the full parameter space and therefore has access to the most complete gradient information. This accuracy advantage, however, comes at a much higher memory cost, as also reflected in Table \ref{tab:mem}. Among the low-rank LR methods, our proposed Stiefel LowRank-LR delivers the strongest overall performance, attaining the highest accuracy on SST-5, SNLI, RTE, and TREC. This suggests that our orthogonally structured subspace design can identify informative descent directions more effectively, leading to better optimization and stronger downstream accuracy under the same low-rank query budget. This advantage is also reflected in Figure \ref{fig:convergence_zo}, where Stiefel sampling generally exhibits more favorable training dynamics than Gaussian sampling.

Compared with the existing Gaussian LowRank-LR baseline, both of our structured subspace designs show clear advantages. Gaussian LowRank-LR consistently improves over Vanilla LR, confirming the benefit of low-rank subspace perturbation itself, but it is generally less competitive than our Stiefel design, with the gap being most visible on SST-5, SNLI, RTE, and TREC. Our Coordinate LowRank-LR achieves the best result on SST-2 and remains competitive on several other tasks, although its performance is less uniform than that of Stiefel across the full benchmark suite. Overall, these results show that while low-rank subspace optimization already improves over the vanilla LR baseline, our proposed structured subspace designs further strengthen its effectiveness, with Stiefel providing the most robust gains over the Gaussian alternative.

\begin{figure}[h]
\centering
\subfloat[SST-2]{
\includegraphics[trim=0cm 0.5cm 0cm 0cm, width=0.3\textwidth]{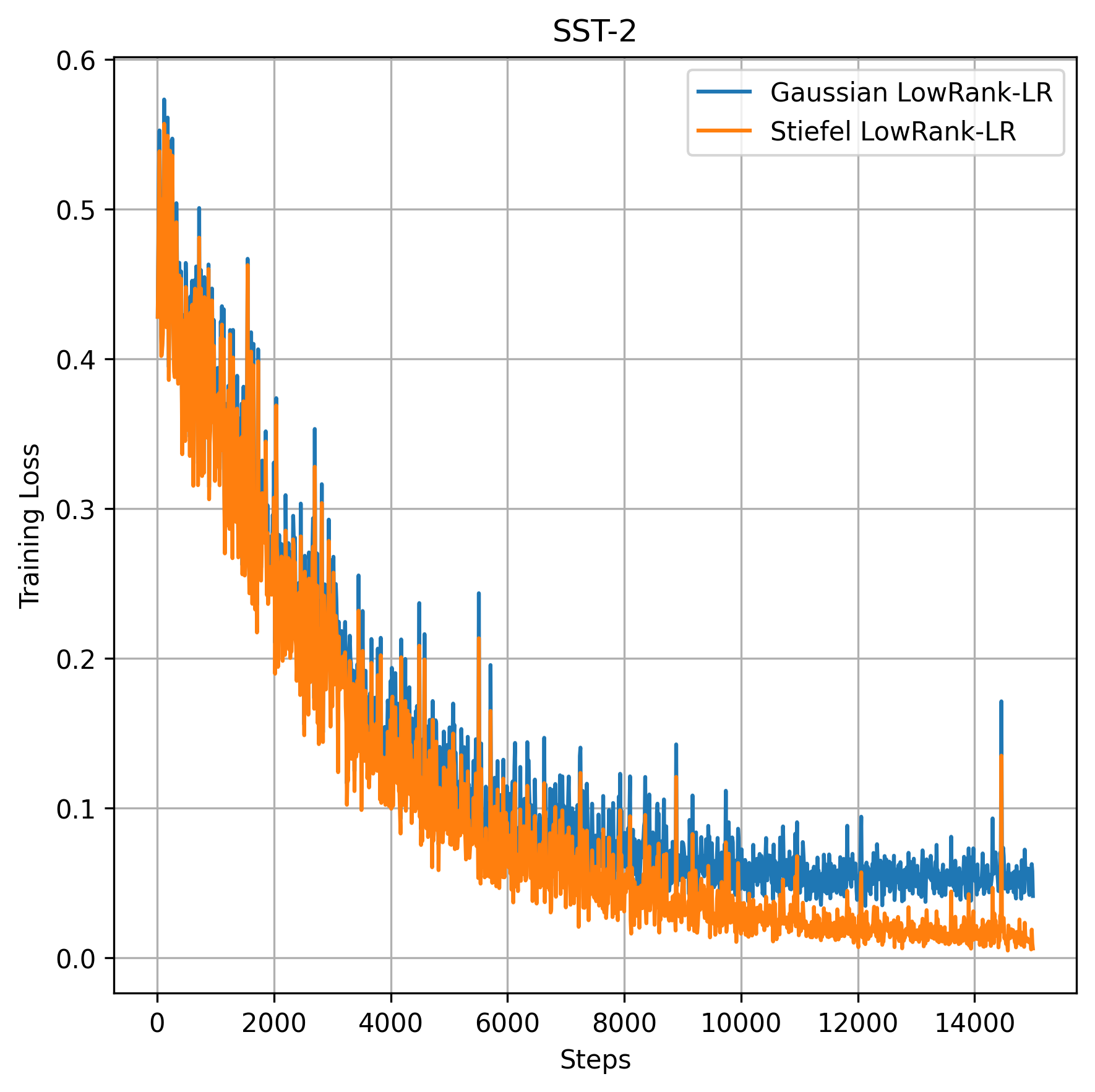}}
\subfloat[SST-5]{
\includegraphics[trim=0cm 0.5cm 0cm 0cm, width=0.3\textwidth]{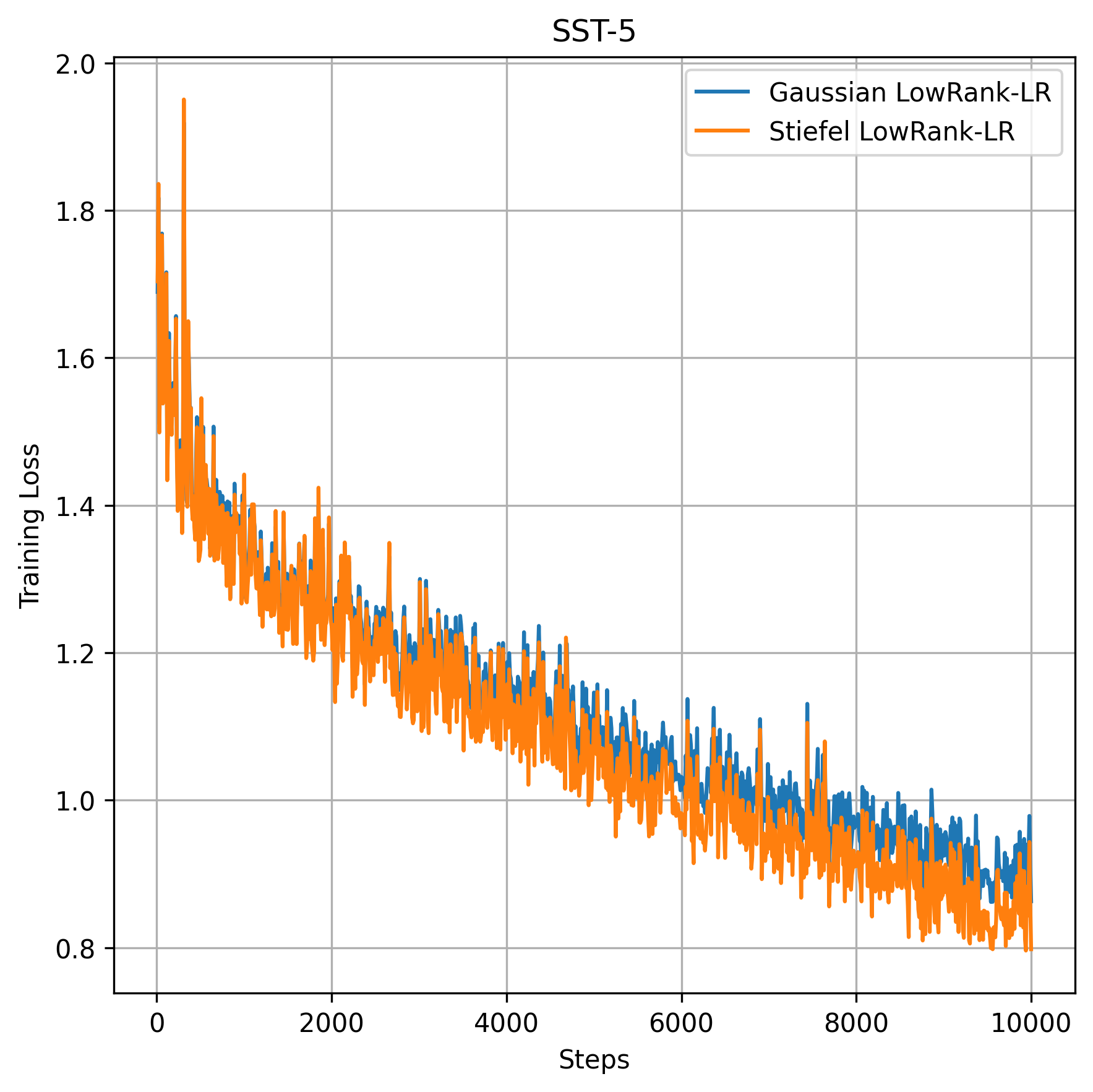}}
\subfloat[SNLI]{
\includegraphics[trim=0cm 0.5cm 0cm 0cm, width=0.3\textwidth]{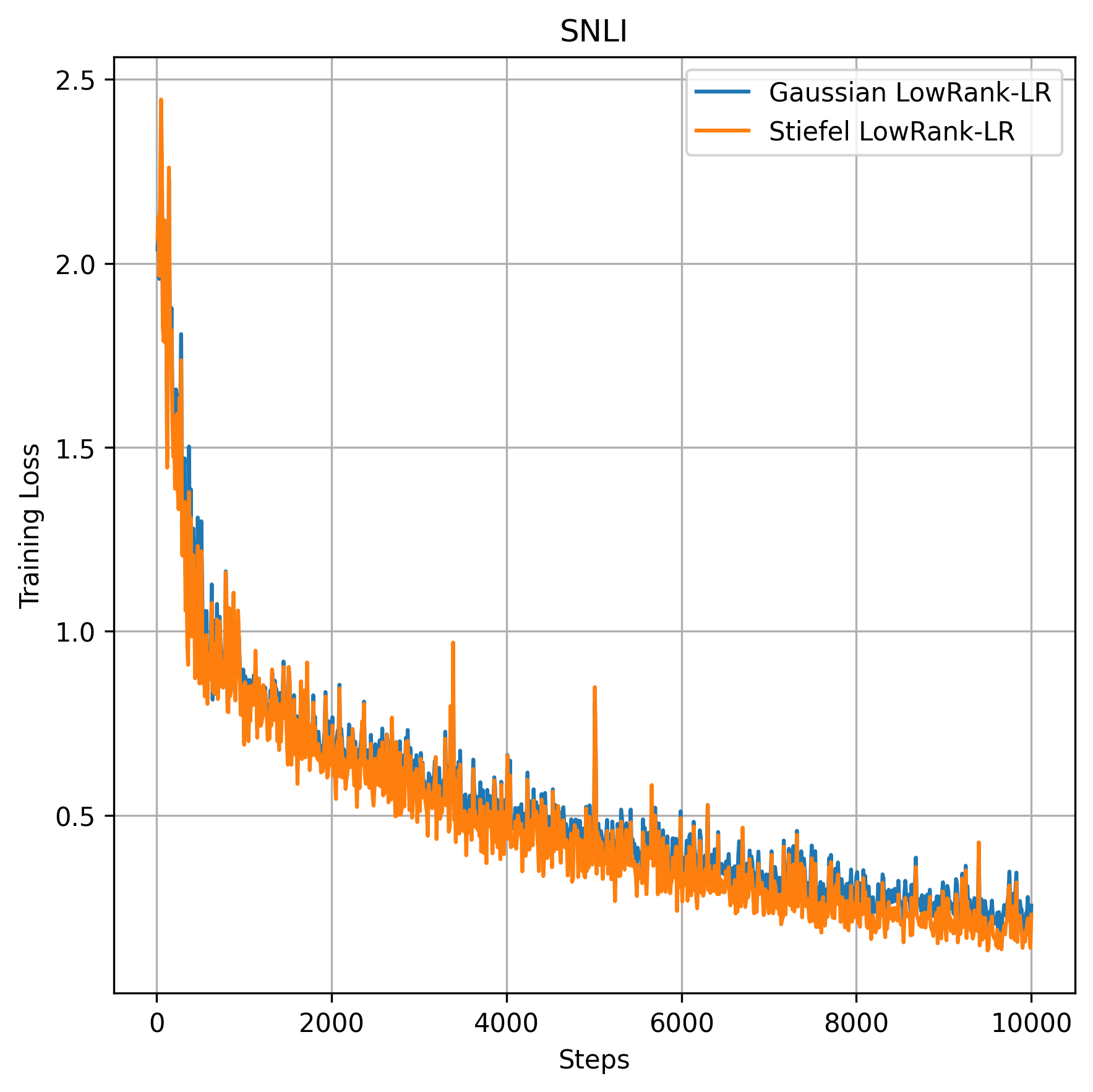}}

\subfloat[MNLI]{
\includegraphics[trim=0cm 0.5cm 0cm 0cm, width=0.3\textwidth]{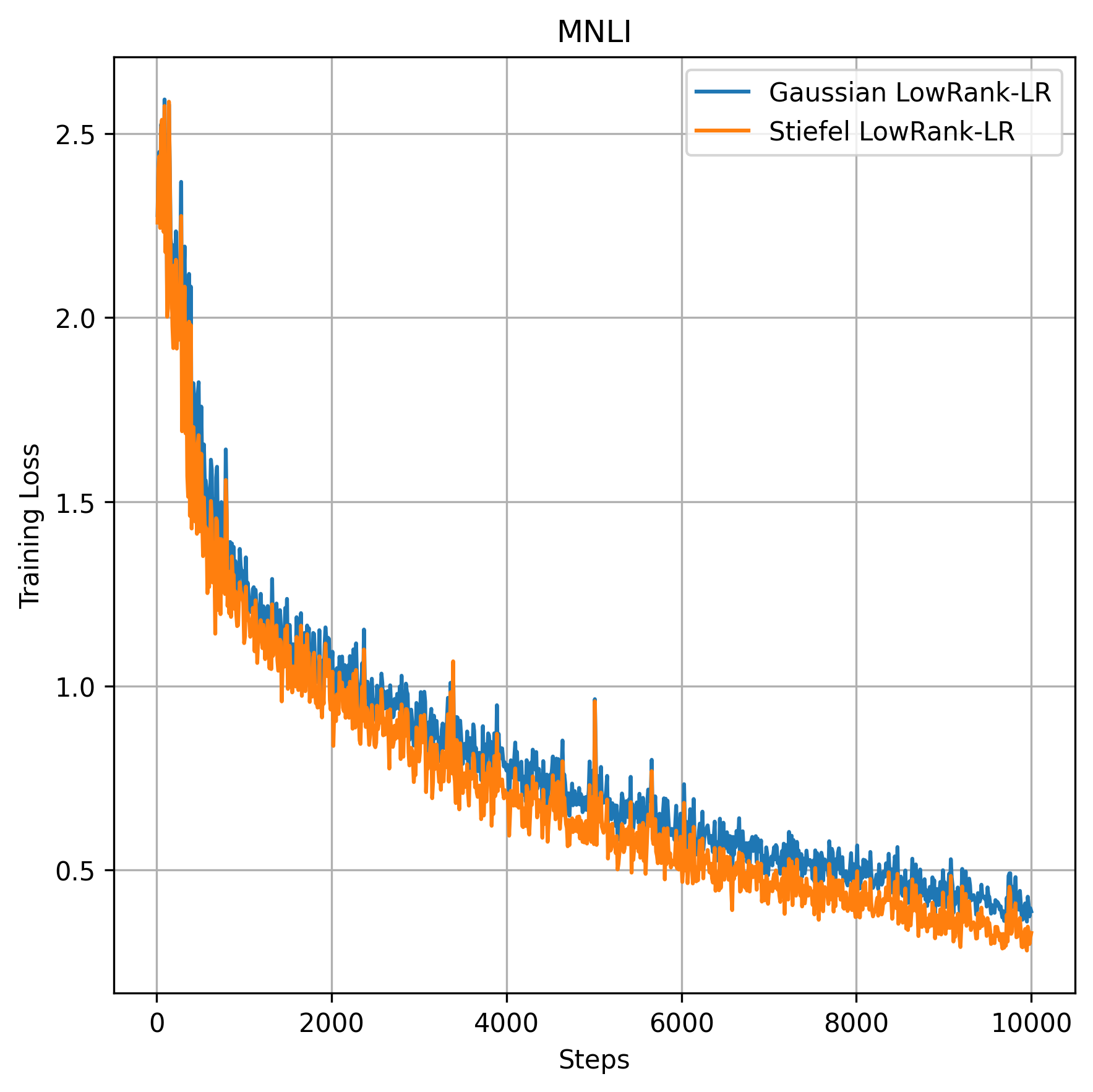}}
\subfloat[RTE]{
\includegraphics[trim=0cm 0.5cm 0cm 0cm, width=0.3\textwidth]{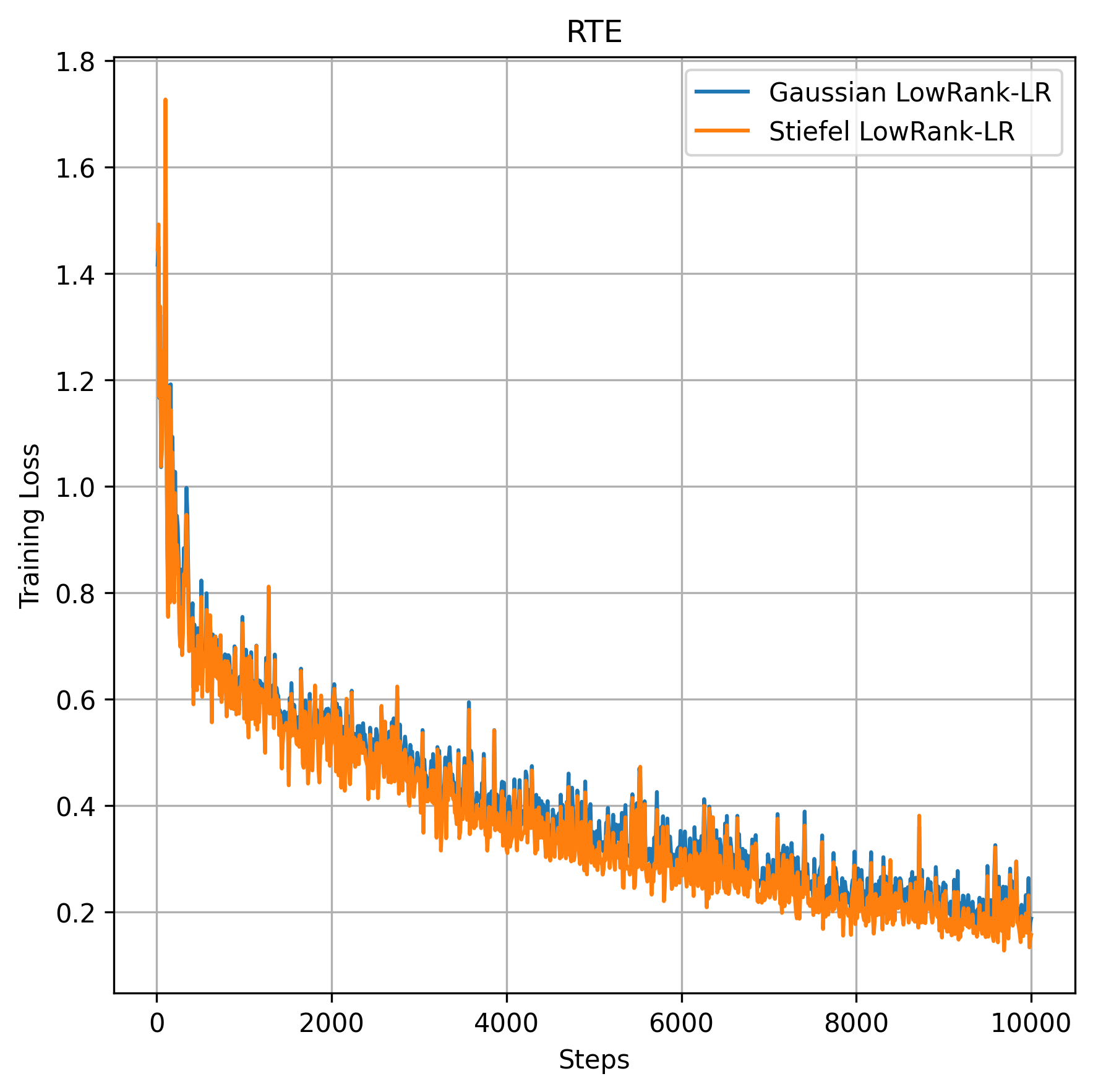}}
\subfloat[TREC]{
\includegraphics[trim=0cm 0.5cm 0cm 0cm, width=0.3\textwidth]{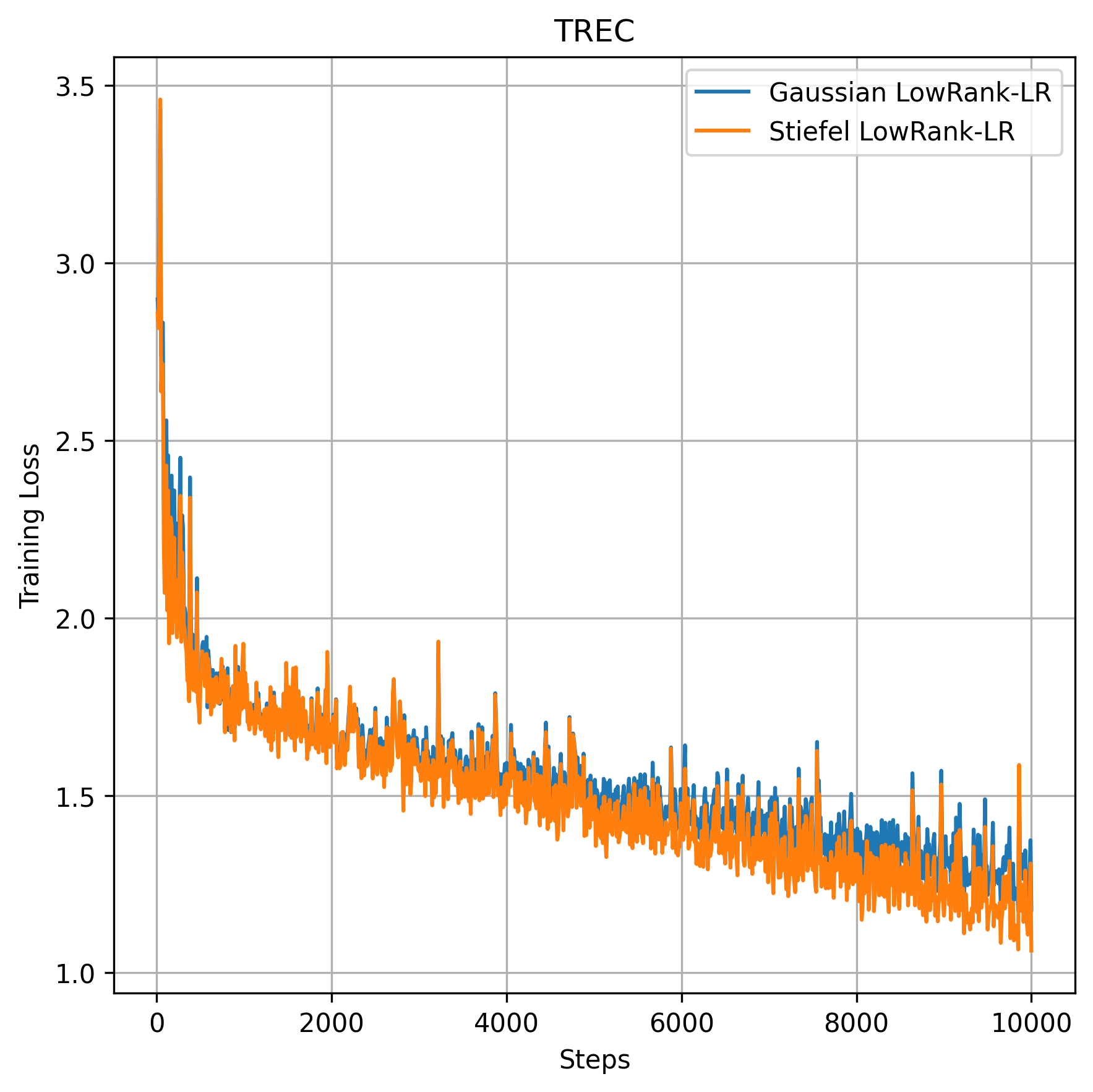}}
\vspace{-0.2cm}
\caption{\textcolor{black}{Training loss of Stiefel and Gaussian subspace sampling across six datasets}}
\vspace{-0.3cm}
\label{fig:convergence_zo}
\end{figure}

\begin{table}[h]
\centering
\small
\begin{tabular}{lcccccc}
\toprule
\textbf{Method} & \textbf{SST-2} & \textbf{SST-5} & \textbf{SNLI} & \textbf{MNLI} & \textbf{RTE} & \textbf{TREC} \\
\midrule
Zero-shot      & 79.0 & 35.5 & 50.2 & 48.8 & 51.4 & 32.0 \\
\midrule
Vanilla LR     & 86.3 & 40.8 & 68.5 & 56.7 & 58.6 & 62.4 \\
Gaussian LowRank-LR     & 88.0 & 41.1 & 73.4 & \textbf{61.6} & 61.2 & 77.9 \\
Stiefel LowRank-LR      & 90.2 & \textbf{43.2} & \textbf{74.3} & 59.4 & \textbf{63.7} & \textbf{80.9} \\
Coordinate LowRank-LR   & \textbf{91.3} & 42.6 & 71.0 & 57.2 & 62.2 & 79.4 \\
\midrule
Vanilla IPA   & 91.9 & 47.5 & 77.5 & 70.0 & 66.4 & 85.0 \\
\bottomrule
\end{tabular}
\caption{RoBERTa-large fine-tuning results on six classification datasets.}
\label{tab:zo_acc}
\end{table}

In addition to accuracy, we evaluate the memory consumption of different fine-tuning approaches to assess their scalability on resource-constrained devices. Table \ref{tab:mem} reports the peak GPU memory usage during RoBERTa-large fine-tuning under four different settings: Vanilla IPA, LowRank-IPA, Vanilla LR, and our LowRank-LR.
As expected, the Vanilla IPA, i.e., full backpropagation, consumes the most memory, due to the need to store gradients and intermediate activations. LowRank-IPA reduces this memory footprint moderately, confirming its benefit in parameter-efficient fine-tuning.

Most notably, LowRank-LR achieves the lowest memory footprint of only 3.83 GB, highlighting its strong memory efficiency. By performing optimization in a low-dimensional subspace and avoiding dense gradient estimation, the method greatly alleviates memory overhead. This makes it particularly suitable for deployment in memory-limited or edge environments where large-scale models like RoBERTa-large are otherwise infeasible to fine-tune.

Overall, these results show that LowRank-LR offers an attractive trade-off between effectiveness and efficiency: it remains competitive in predictive performance while delivering substantial memory savings, making it a practical alternative to conventional fine-tuning methods.
\begin{table}[h]
\centering
\small
\begin{tabular}{lcccc}
\toprule
 & \textbf{Vanilla IPA} & \textbf{LowRank-IPA} & \textbf{Vanilla LR} & \textbf{LowRank-LR} \\
\midrule
Memory (GB) & 16.7 & 14.3 & 5.49 & \textbf{3.83} \\
\bottomrule
\end{tabular}
\caption{Memory profile of different methods for fine-tuning the RoBERTa-large.}
\label{tab:mem}
\end{table}

We further report the per-step wall-clock time of different fine-tuning strategies in Table~\ref{tab:time} to assess their computational efficiency. As expected, backpropagation-based methods (Vanilla IPA and LowRank-IPA) incur the highest per-step costs due to gradient computation and backward pass overhead in large transformer models.

By contrast, LR-based methods substantially reduce runtime per step by eliminating gradient backpropagation. Our LowRank-LR method introduces only a slight increase in per-step time, primarily due to sampling from the optimized low-dimensional distribution and performing projection operations. Despite this modest overhead, the runtime remains significantly faster than backpropagation-based methods and well within acceptable bounds for practical deployment.

\begin{table}[h]
\centering
\small
\begin{tabular}{lccccc}
\toprule
 & \textbf{Vanilla IPA} & \textbf{LowRank-IPA} & \textbf{Vanilla LR} & \textbf{LowRank-LR} \\
\midrule
Time (second)   & 0.784 & 0.787 &  0.468 & 0.493 \\
\bottomrule
\end{tabular}
\caption{Per-step wall clock time.}
\label{tab:time}
\end{table}

\subsubsection{Pretraining LLM with infinitesimal perturbation analysis estimator.}\label{sec6.2.2}
We present the pretraining experiments on autoregressive language models with different sizes. This empirical study employs the IPA gradient estimator for the efficient pretraining of LLMs. Experiments are conducted on the \textbf{LLaMA model family}. Training is implemented via Distributed Data Parallel for multi-GPU acceleration. The \textbf{OpenWebText corpus} is used for training, and the T5-base tokenizer is employed with a fixed sequence length of 256.

\begin{figure}[h]
\centering
\subfloat[evaluation loss]{
\includegraphics[trim=0cm 0.5cm 0cm 0cm, width=0.3\textwidth]{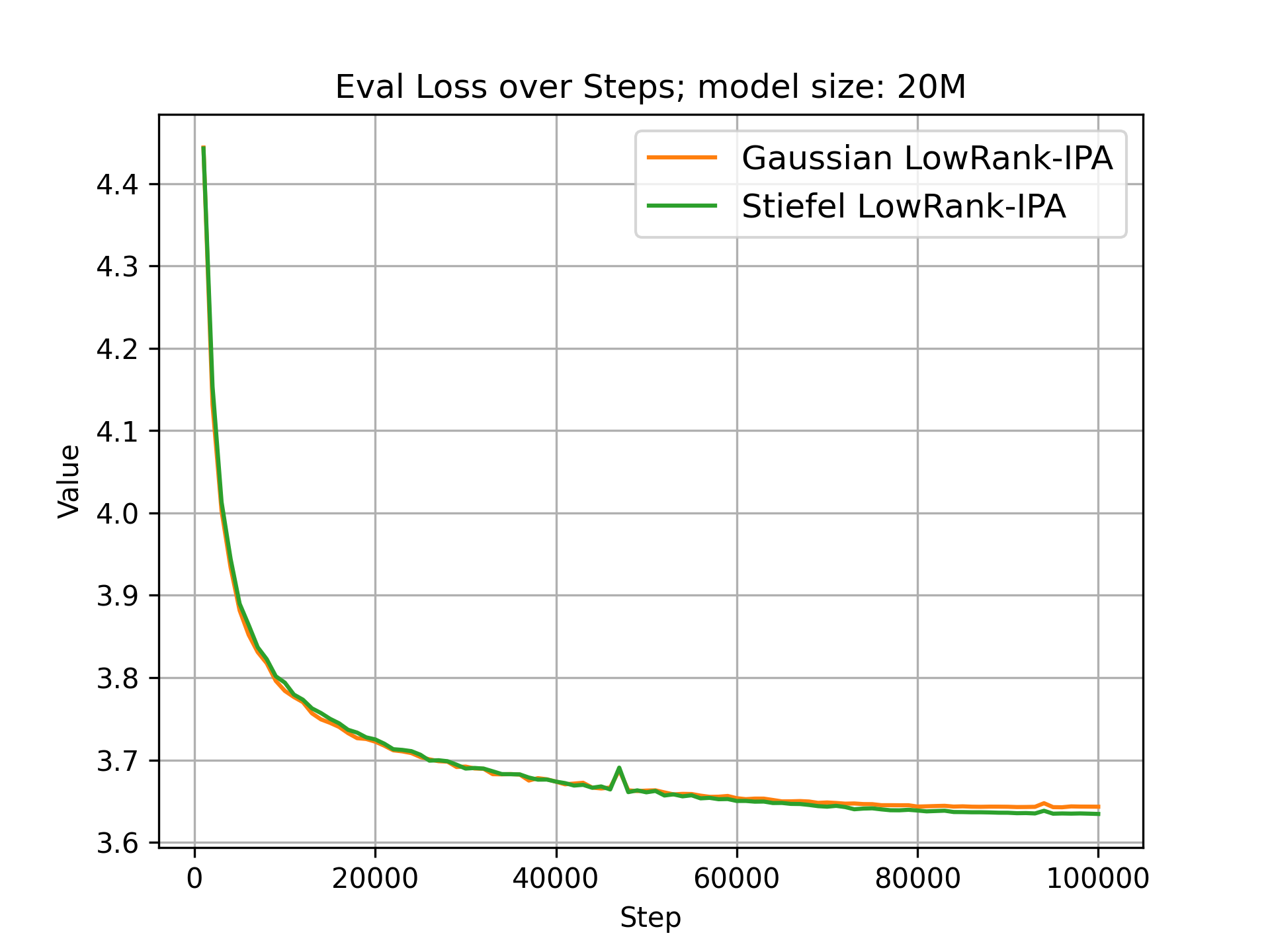}}
\subfloat[training loss]{
\includegraphics[trim=0cm 0.5cm 0cm 0cm, width=0.3\textwidth]{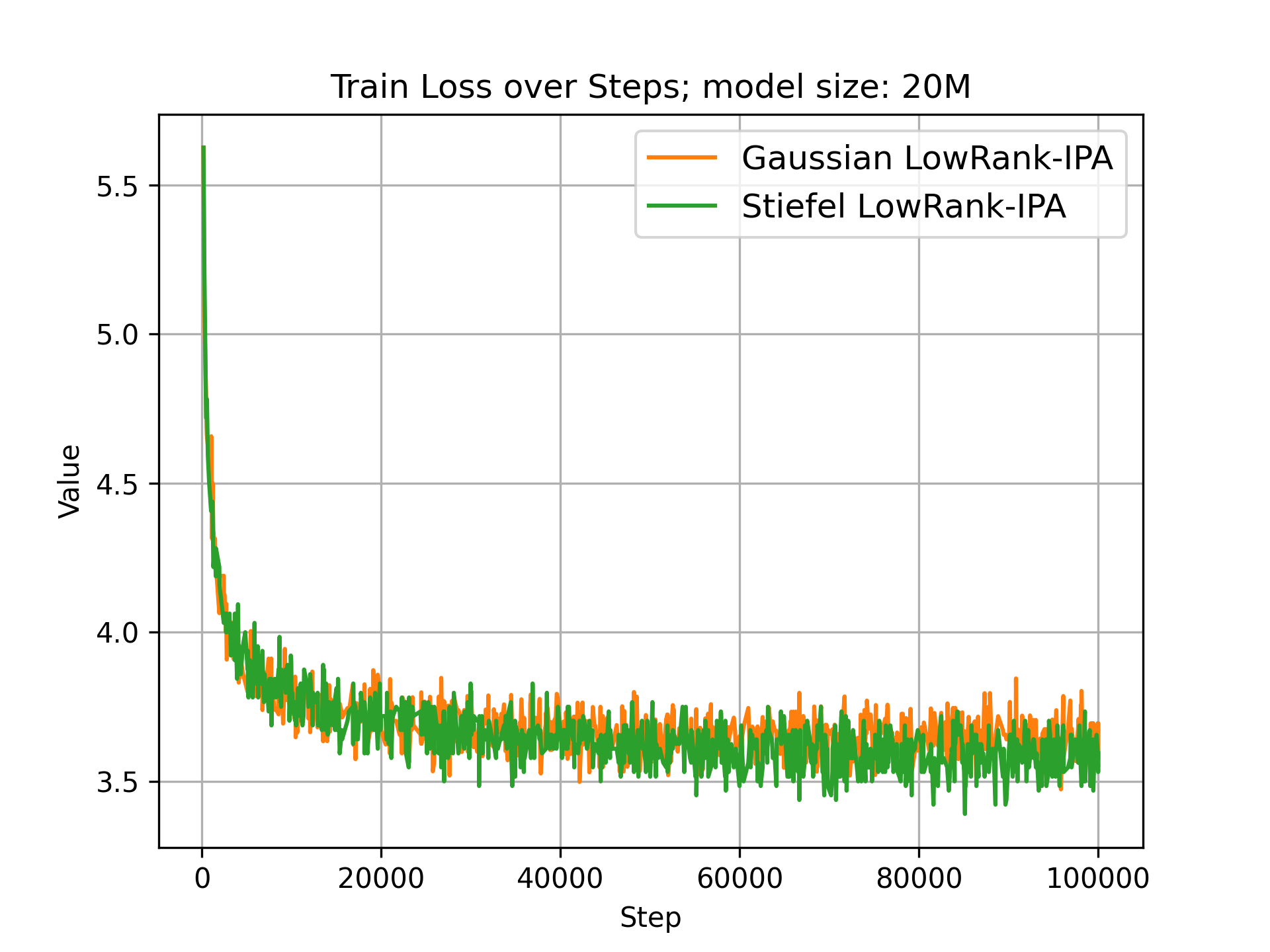}}
\vspace{-0.2cm}
\caption{Pretraining 20M language model using IPA gradient estimator}
\vspace{-0.3cm}
\label{fig:IPA_20M}
\end{figure}

\begin{figure}[h]
\centering
\subfloat[evaluation loss]{
\includegraphics[trim=0cm 0.5cm 0cm 0cm, width=0.3\textwidth]{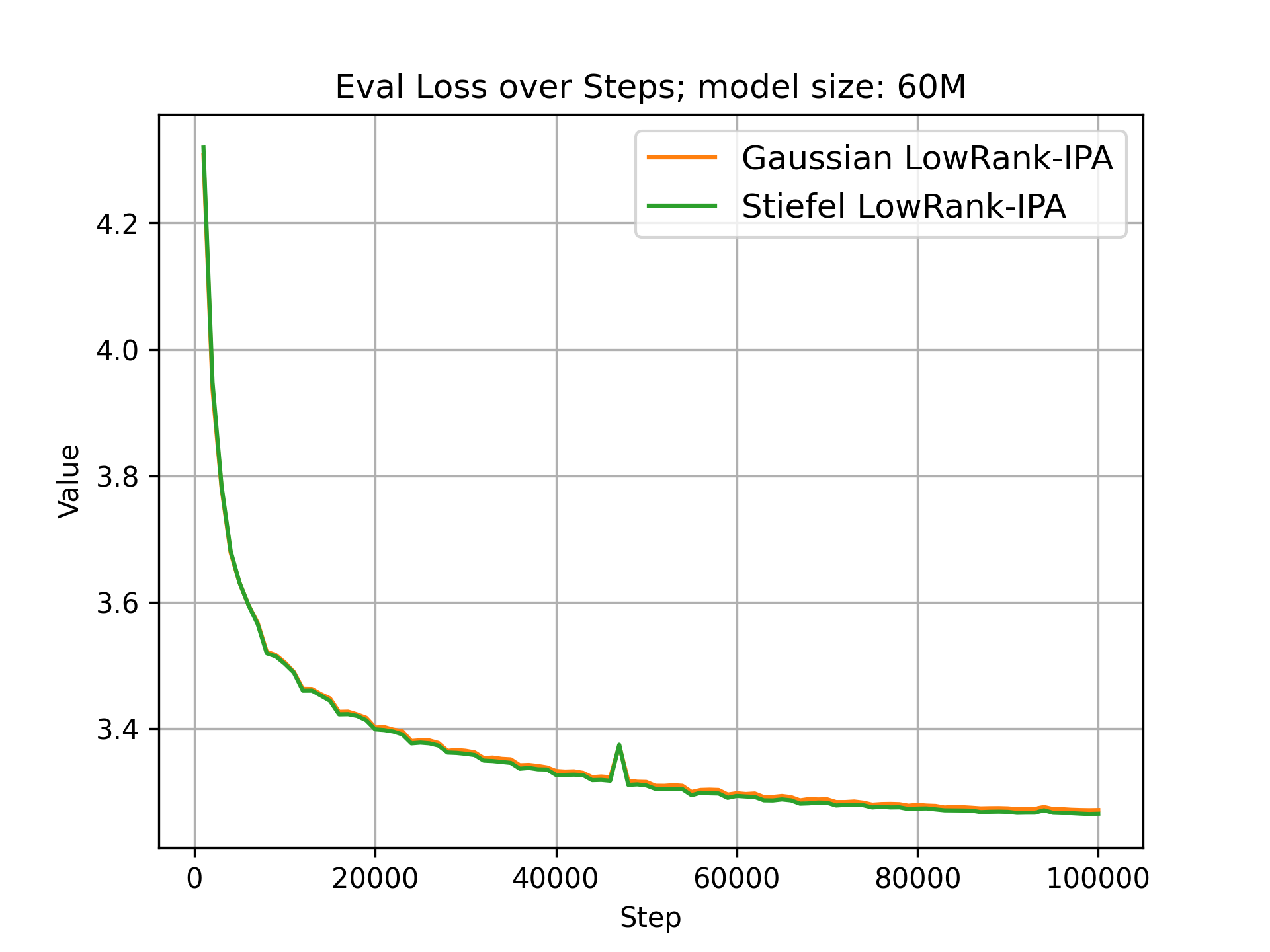}}
\subfloat[training loss]{
\includegraphics[trim=0cm 0.5cm 0cm 0cm, width=0.3\textwidth]{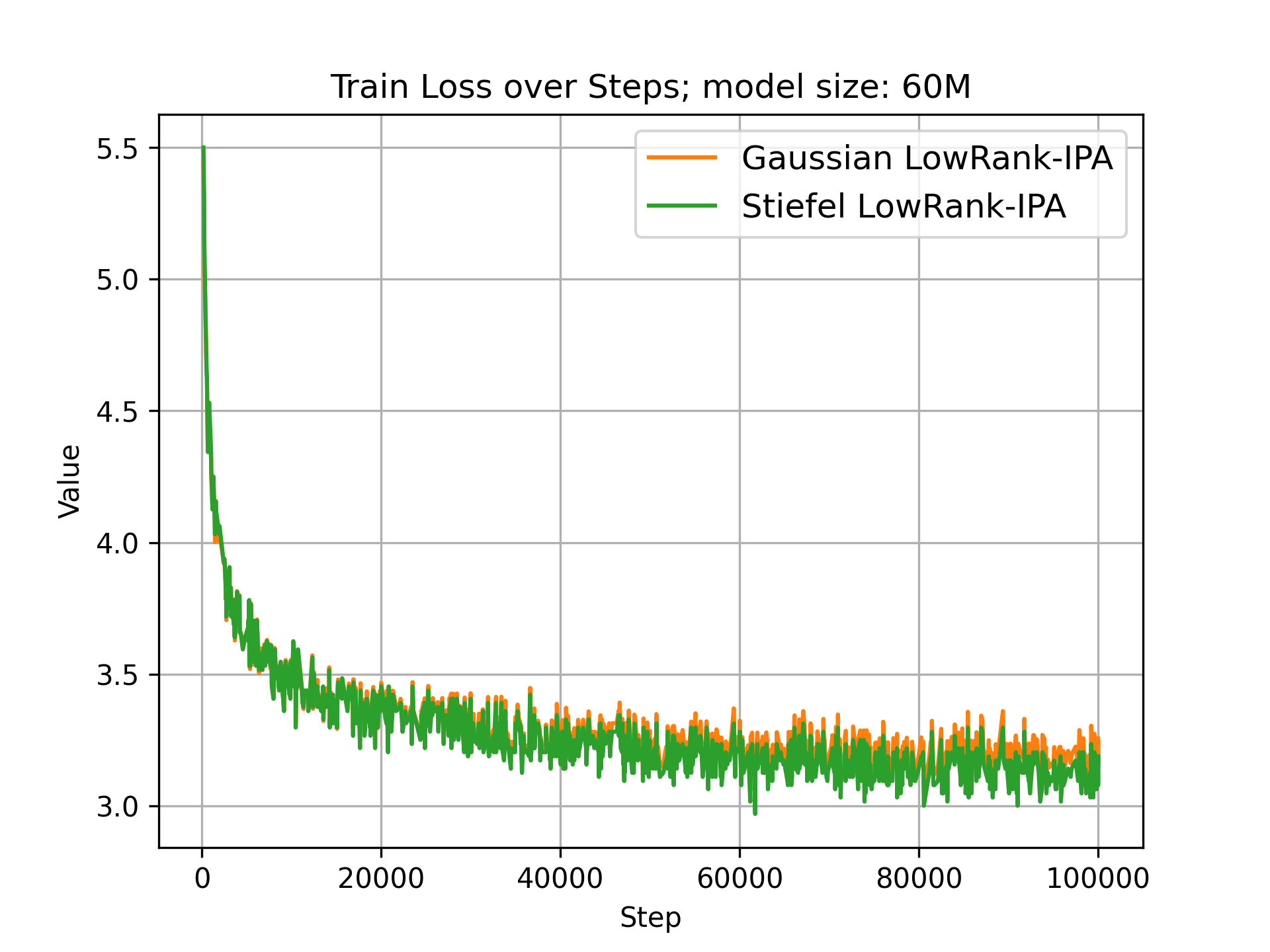}}
\vspace{-0.2cm}
\caption{Pretraining 60M language model using IPA gradient estimator}
\vspace{-0.3cm}
\label{fig:IPA_60M}
\end{figure}

\begin{figure}[h]
\centering
\subfloat[evaluation loss]{
\includegraphics[trim=0cm 0.5cm 0cm 0cm, width=0.3\textwidth]{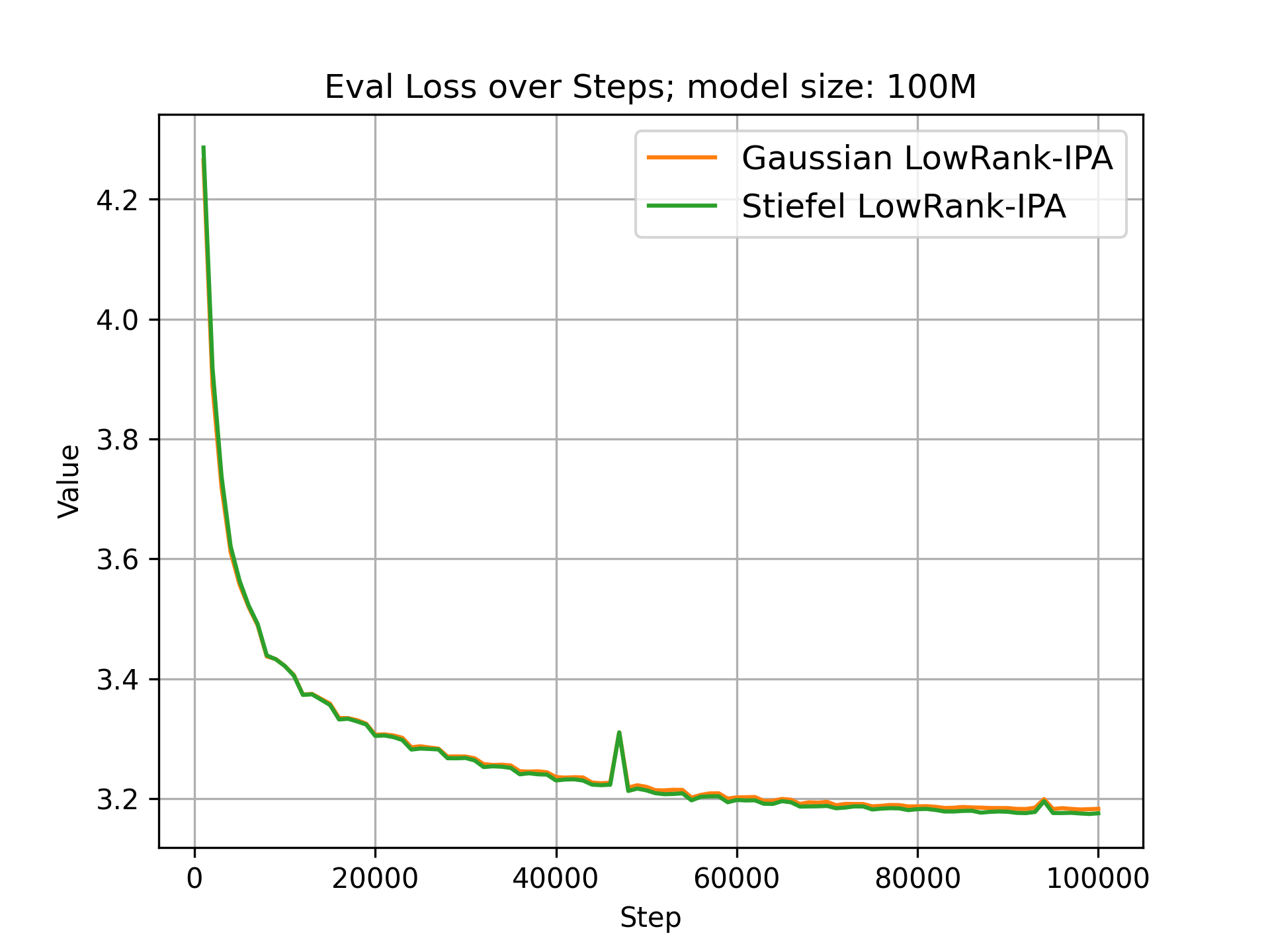}}
\subfloat[training loss]{
\includegraphics[trim=0cm 0.5cm 0cm 0cm, width=0.3\textwidth]{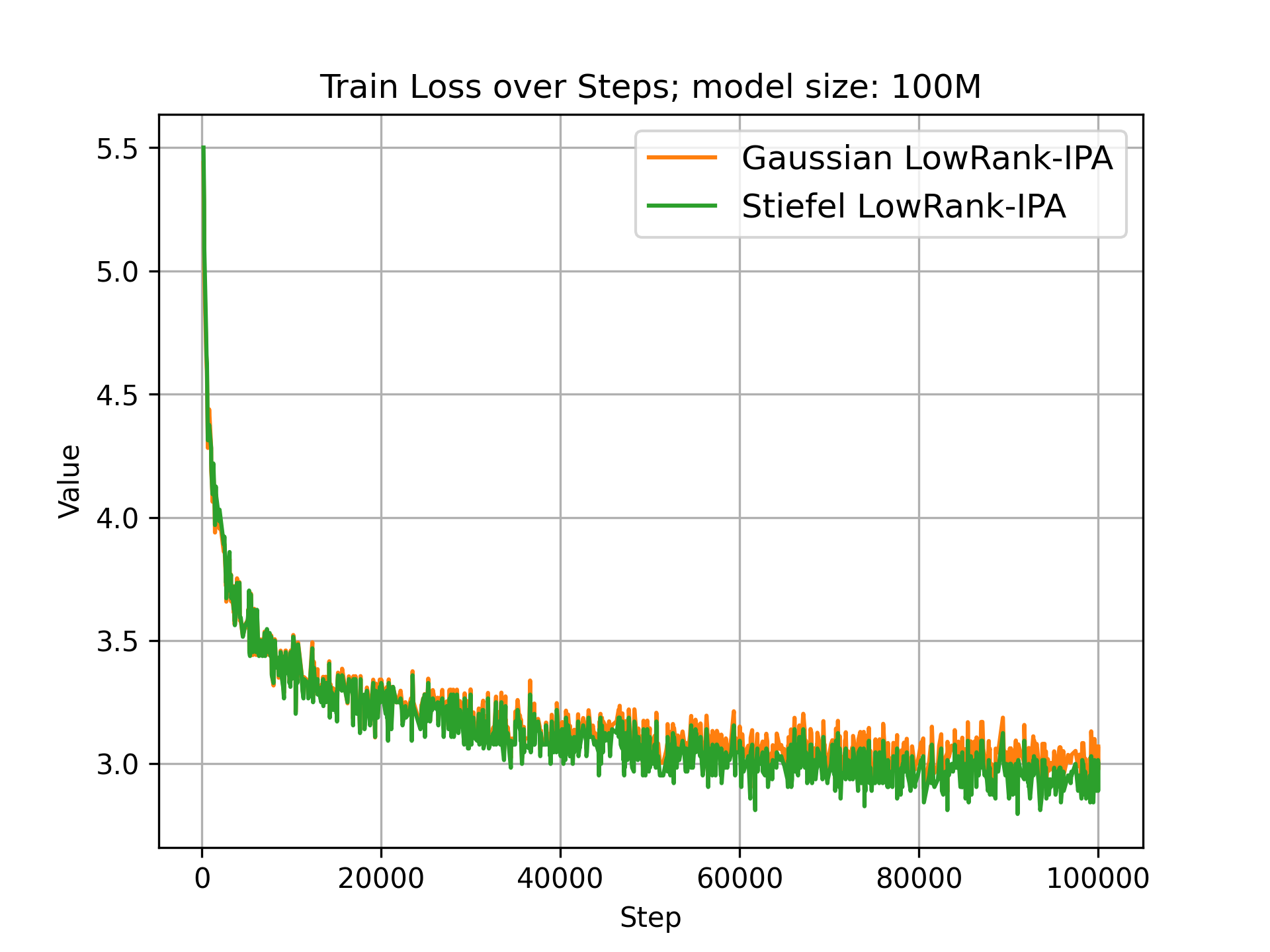}}
\vspace{-0.2cm}
\caption{Pretraining 100M language model using IPA gradient estimator}
\vspace{-0.3cm}
\label{fig:IPA_100M}
\end{figure}

All experiments are performed using bfloat16 precision. The optimizer is Adam, specifically adapted for subspace training, with Adam parameters $\beta_1=0.9$ and $\beta_2=0.999$. Gradient clipping is applied at a norm of 1.0, and the learning rate schedule follows a cosine annealing strategy with a cycle length of $100,000$ steps. Key hyperparameters are scaled across different model sizes as follows: For LLaMA-20M/60M/100M, the global batch size is uniformly set to $512$, and weight decay is fixed at $0.05$. The total training steps are $100,000$ with warmup steps adjusted to $1000$. The subproblem reset interval is set to $200$ steps. The subspace rank is set to $128$.

We compare the proposed Stiefel LowRank-IPA to the Gaussian LowRank-IPA method. The results are presented in Figure \ref{fig:IPA_20M}, \ref{fig:IPA_60M}, and \ref{fig:IPA_100M}. Across all model scales (20M, 60M, and 100M), the Stiefel LowRank-IPA consistently outperforms Gaussian LowRank-IPA in both training and evaluation loss reduction throughout the entire pretraining process, demonstrating the superiority of the proposed optimal projection strategy for the IPA gradient estimator. In Figure \ref{fig:IPA_20M}, Stiefel LowRank-IPA achieves a markedly lower evaluation loss and training loss from the early training stages, with the loss gap widening as training proceeds; the Gaussian  LowRank-IPA method exhibits slower loss decay and higher stable loss values, indicating less efficient gradient estimation and parameter update. For the 60M and 100M models (Figures \ref{fig:IPA_60M} and \ref{fig:IPA_100M}), Stiefel LowRank-IPA still outperforms Gaussian LowRank-IPA with faster training loss decay, a lower convergent loss and consistently lower evaluation loss, while Gaussian LowRank-IPA shows slow convergence and high residual loss.

\section{Conclusion}\label{sec7}

Training LLMs is fundamentally constrained by the tension between memory usage and the noise of stochastic gradients in extremely high-dimensional parameter spaces. Motivated by the empirically observed low-rank structure of neural gradient matrices, we proposed weakly unbiased low-rank matrix gradient estimators that apply to both IPA- and LR-family gradient estimation paradigms, together with a memory-efficient randomized subspace projection algorithm and a lazy-update mechanism. We analyzed the MSE of the projected estimators and showed how it decomposes into intrinsic IPA/LR variance, random-projection-induced variance, and a scalar bias term, which makes the role of projection design explicit. To optimally control this error, we formulated a constrained functional optimization problem over admissible random projectors and derived optimal projection distributions in both information-agnostic and information-aware settings, including Haar--Stiefel type constructions and instance-dependent sampling rules. Empirically, our method achieves substantial peak-memory reductions while remaining competitive in downstream performance, and it consistently improves training and evaluation losses in autoregressive pretraining across multiple model scales. Overall, these results suggest that principled randomized low-rank projections provide a unifying and practical route to scalable, stable, and memory-efficient gradient estimation for modern LLM training.

\bibliography{ref}

@article{mohamed2020monte,
  title={Monte carlo gradient estimation in machine learning},
  author={Mohamed, Shakir and Rosca, Mihaela and Figurnov, Michael and Mnih, Andriy},
  journal={Journal of Machine Learning Research},
  volume={21},
  number={132},
  pages={1--62},
  year={2020}
}

@book{fu2015stochastic,
  title={Stochastic gradient estimation},
  author={Fu, Michael C},
  year={2015},
  publisher={Springer}
}

@incollection{FU2006575,
title = {Chapter 19 Gradient Estimation},
editor = {Shane G. Henderson and Barry L. Nelson},
series = {Handbooks in Operations Research and Management Science},
publisher = {Elsevier},
volume = {13},
pages = {575-616},
year = {2006},
booktitle = {Simulation},
issn = {0927-0507},
author = {Michael C. Fu}
}

@article{ren2025zeroth,
  title={Zeroth-order Informed Fine-Tuning for Diffusion Model: A Recursive Likelihood Ratio Optimizer},
  author={Ren, Tao and Zhang, Zishi and Li, Zehao and Jiang, Jingyang and Qin, Shentao and Li, Guanghao and Li, Yan and Zheng, Yi and Li, Xinping and Zhan, Min and others},
  journal={arXiv preprint arXiv:2502.00639},
  year={2025}
}

@article{chen2024enhancing,
  title={Enhancing zeroth-order fine-tuning for language models with low-rank structures},
  author={Chen, Yiming and Zhang, Yuan and Cao, Liyuan and Yuan, Kun and Wen, Zaiwen},
  journal={arXiv preprint arXiv:2410.07698},
  year={2024}
}

@article{zhao2024galore,
  title={Galore: Memory-efficient llm training by gradient low-rank projection},
  author={Zhao, Jiawei and Zhang, Zhenyu and Chen, Beidi and Wang, Zhangyang and Anandkumar, Anima and Tian, Yuandong},
  journal={arXiv preprint arXiv:2403.03507},
  year={2024}
}

@article{hu2022lora,
  title={Lora: Low-rank adaptation of large language models.},
  author={Hu, Edward J and Shen, Yelong and Wallis, Phillip and Allen-Zhu, Zeyuan and Li, Yuanzhi and Wang, Shean and Wang, Lu and Chen, Weizhu and others},
  journal={ICLR},
  volume={1},
  number={2},
  pages={3},
  year={2022}
}

@article{he2024subspace,
  title={Subspace optimization for large language models with convergence guarantees},
  author={He, Yutong and Li, Pengrui and Hu, Yipeng and Chen, Chuyan and Yuan, Kun},
  journal={arXiv preprint arXiv:2410.11289},
  year={2024}
}

@article{chen2025memory,
  title={A Memory Efficient Randomized Subspace Optimization Method for Training Large Language Models},
  author={Chen, Yiming and Zhang, Yuan and Liu, Yin and Yuan, Kun and Wen, Zaiwen},
  journal={arXiv preprint arXiv:2502.07222},
  year={2025}
}

@article{peng2022new,
  title={A new likelihood ratio method for training artificial neural networks},
  author={Peng, Yijie and Xiao, Li and Heidergott, Bernd and Hong, L Jeff and Lam, Henry},
  journal={INFORMS Journal on Computing},
  volume={34},
  number={1},
  pages={638--655},
  year={2022},
  publisher={INFORMS}
}

@article{hajek1964asymptotic,
  title={Asymptotic theory of rejective sampling with varying probabilities from a finite population},
  author={H{\'a}jek, Jaroslav},
  journal={The Annals of Mathematical Statistics},
  volume={35},
  number={4},
  pages={1491--1523},
  year={1964},
  publisher={Institute of Mathematical Statistics}
}

@article{sampford1967sampling,
  title={On sampling without replacement with unequal probabilities of selection},
  author={Sampford, Michael R},
  journal={Biometrika},
  volume={54},
  number={3-4},
  pages={499--513},
  year={1967},
  publisher={Oxford University Press}
}

@article{deville1998unequal,
  title={Unequal probability sampling without replacement through a splitting method},
  author={Deville, Jean-Claude and Tille, Yves},
  journal={Biometrika},
  volume={85},
  number={1},
  pages={89--101},
  year={1998},
  publisher={Oxford University Press}
}

@article{bondesson2008list,
  title={A list sequential sampling method suitable for real-time sampling},
  author={Bondesson, Lennart and Thorburn, Daniel},
  journal={Scandinavian Journal of Statistics},
  volume={35},
  number={3},
  pages={466--483},
  year={2008},
  publisher={Wiley Online Library}
}

@article{kozak2023zeroth,
  title={Zeroth-order optimization with orthogonal random directions},
  author={Kozak, David and Molinari, Cesare and Rosasco, Lorenzo and Tenorio, Luis and Villa, Silvia},
  journal={Mathematical Programming},
  volume={199},
  number={1},
  pages={1179--1219},
  year={2023},
  publisher={Springer}
}

@article{fang1994inequalities,
  title={Inequalities for the trace of matrix product},
  author={Fang, Yuguang and Loparo, Kenneth A and Feng, Xiangbo},
  journal={IEEE Transactions on Automatic Control},
  volume={39},
  number={12},
  pages={2489--2490},
  year={1994},
  publisher={IEEE}
}

@article{heidelberger1988convergence,
  title={Convergence properties of infinitesimal perturbation analysis estimates},
  author={Heidelberger, Philip and Cao, Xi-Ren and Zazanis, Michael A and Suri, Rajan},
  journal={Management Science},
  volume={34},
  number={11},
  pages={1281--1302},
  year={1988},
  publisher={INFORMS}
}

@article{l1990unified,
  title={A unified view of the IPA, SF, and LR gradient estimation techniques},
  author={L'ecuyer, Pierre},
  journal={Management Science},
  volume={36},
  number={11},
  pages={1364--1383},
  year={1990},
  publisher={INFORMS}
}

@article{ho1983infinitesimal,
  title={Infinitesimal and finite perturbation analysis for queueing networks},
  author={Ho, Yu-Chi and Cao, Xiren and Cassandras, Christos},
  journal={Automatica},
  volume={19},
  number={4},
  pages={439--445},
  year={1983},
  publisher={Elsevier}
}

@article{glynn1990likelihood,
  title={Likelihood ratio gradient estimation for stochastic systems},
  author={Glynn, Peter W},
  journal={Communications of the ACM},
  volume={33},
  number={10},
  pages={75--84},
  year={1990},
  publisher={ACM New York, NY, USA}
}

@article{reiman1989sensitivity,
  title={Sensitivity Analysis for Simulations via Likelihood Ratios},
  author={Reiman, Martin I and Weiss, Alan},
  journal={Operations Research},
  volume={37},
  number={5},
  pages={830--844},
  year={1989},
  publisher={INFORMS}
}

@article{peng2018new,
  title={A new unbiased stochastic derivative estimator for discontinuous sample performances with structural parameters},
  author={Peng, Yijie and Fu, Michael C and Hu, Jian-Qiang and Heidergott, Bernd},
  journal={Operations Research},
  volume={66},
  number={2},
  pages={487--499},
  year={2018},
  publisher={INFORMS}
}

@article{hong2009estimating,
  title={Estimating quantile sensitivities},
  author={Hong, L Jeff},
  journal={Operations research},
  volume={57},
  number={1},
  pages={118--130},
  year={2009},
  publisher={INFORMS}
}

@article{kozak2021stochastic,
  title={A stochastic subspace approach to gradient-free optimization in high dimensions},
  author={Kozak, David and Becker, Stephen and Doostan, Alireza and Tenorio, Luis},
  journal={Computational Optimization and Applications},
  volume={79},
  number={2},
  pages={339--368},
  year={2021},
  publisher={Springer}
}

@article{nesterov2017random,
  title={Random gradient-free minimization of convex functions},
  author={Nesterov, Yurii and Spokoiny, Vladimir},
  journal={Foundations of Computational Mathematics},
  volume={17},
  number={2},
  pages={527--566},
  year={2017},
  publisher={Springer}
}

@article{duchi2015optimal,
  title={Optimal rates for zero-order convex optimization: The power of two function evaluations},
  author={Duchi, John C and Jordan, Michael I and Wainwright, Martin J and Wibisono, Andre},
  journal={IEEE Transactions on Information Theory},
  volume={61},
  number={5},
  pages={2788--2806},
  year={2015},
  publisher={IEEE}
}

@inproceedings{ji2019improved,
  title={Improved zeroth-order variance reduced algorithms and analysis for nonconvex optimization},
  author={Ji, Kaiyi and Wang, Zhe and Zhou, Yi and Liang, Yingbin},
  booktitle={International conference on machine learning},
  pages={3100--3109},
  year={2019},
  organization={PMLR}
}

@article{lam2024distributionally,
  title={Distributionally constrained black-box stochastic gradient estimation and optimization},
  author={Lam, Henry and Zhang, Junhui},
  journal={Operations Research},
  year={2024},
  publisher={INFORMS}
}

@article{xu2023gradient,
  title={Gradient-based simulation optimization algorithms via multi-resolution system approximations},
  author={Xu, Jingxu and Zheng, Zeyu},
  journal={INFORMS Journal on Computing},
  volume={35},
  number={3},
  pages={633--651},
  year={2023},
  publisher={INFORMS}
}

@article{peng2020maximum,
  title={Maximum likelihood estimation by Monte Carlo simulation: Toward data-driven stochastic modeling},
  author={Peng, Yijie and Fu, Michael C and Heidergott, Bernd and Lam, Henry},
  journal={Operations Research},
  volume={68},
  number={6},
  pages={1896--1912},
  year={2020},
  publisher={INFORMS}
}

@article{li2024beyond,
  title={Beyond likelihood ratio bias: Nested multi-time-scale stochastic approximation for likelihood-free parameter estimation},
  author={Li, Zehao and Lin, Zhouchen and Peng, Yijie},
  journal={arXiv preprint arXiv:2411.12995},
  year={2024}
}

@inproceedings{karimi2019non,
  title={Non-asymptotic analysis of biased stochastic approximation scheme},
  author={Karimi, Belhal and Miasojedow, Blazej and Moulines, Eric and Wai, Hoi-To},
  booktitle={Conference on Learning Theory},
  pages={1944--1974},
  year={2019},
  organization={PMLR}
}

@article{cai2022zeroth,
  title={Zeroth-order regularized optimization (zoro): Approximately sparse gradients and adaptive sampling},
  author={Cai, HanQin and McKenzie, Daniel and Yin, Wotao and Zhang, Zhenliang},
  journal={SIAM Journal on Optimization},
  volume={32},
  number={2},
  pages={687--714},
  year={2022},
  publisher={SIAM}
}

@article{hu2024quantile,
  title={Quantile optimization via multiple-timescale local search for black-box functions},
  author={Hu, Jiaqiao and Song, Meichen and Fu, Michael C},
  journal={Operations Research},
  year={2024},
  publisher={INFORMS}
}

@article{hu2022stochastic,
  title={A stochastic approximation method for simulation-based quantile optimization},
  author={Hu, Jiaqiao and Peng, Yijie and Zhang, Gongbo and Zhang, Qi},
  journal={INFORMS Journal on Computing},
  volume={34},
  number={6},
  pages={2889--2907},
  year={2022},
  publisher={INFORMS}
}

@article{kingma2014adam,
  title={Adam: A method for stochastic optimization},
  author={Kingma, Diederik P and Ba, Jimmy},
  journal={arXiv preprint arXiv:1412.6980},
  year={2014}
}

@article{zhang2023gradient,
  title={Gradient-based algorithms for convex discrete optimization via simulation},
  author={Zhang, Haixiang and Zheng, Zeyu and Lavaei, Javad},
  journal={Operations research},
  volume={71},
  number={5},
  pages={1815--1834},
  year={2023},
  publisher={INFORMS}
}

@article{wang2023large,
  title={Large-scale inventory optimization: A recurrent neural networks--inspired simulation approach},
  author={Wang, Tan and Hong, L Jeff},
  journal={INFORMS Journal on Computing},
  volume={35},
  number={1},
  pages={196--215},
  year={2023},
  publisher={INFORMS}
}

@article{heidergott2010gradient,
  title={Gradient estimation for discrete-event systems by measure-valued differentiation},
  author={Heidergott, Bernd and V{\'a}zquez--Abad, Felisa J and Pflug, Georg and Farenhorst-Yuan, Taoying},
  journal={ACM Transactions on Modeling and Computer Simulation (TOMACS)},
  volume={20},
  number={1},
  pages={1--28},
  year={2010},
  publisher={ACM New York, NY, USA}
}

@article{cui2020variance,
  title={On the variance of single-run unbiased stochastic derivative estimators},
  author={Cui, Zhenyu and Fu, Michael C and Hu, Jian-Qiang and Liu, Yanchu and Peng, Yijie and Zhu, Lingjiong},
  journal={INFORMS Journal on Computing},
  volume={32},
  number={2},
  pages={390--407},
  year={2020},
  publisher={INFORMS}
}

@article{rhee2015unbiased,
  title={Unbiased estimation with square root convergence for SDE models},
  author={Rhee, Chang-han and Glynn, Peter W},
  journal={Operations Research},
  volume={63},
  number={5},
  pages={1026--1043},
  year={2015},
  publisher={INFORMS}
}

@article{he2024adaptive,
  title={Adaptive importance sampling for efficient stochastic root finding and quantile estimation},
  author={He, Shengyi and Jiang, Guangxin and Lam, Henry and Fu, Michael C},
  journal={Operations Research},
  volume={72},
  number={6},
  pages={2612--2630},
  year={2024},
  publisher={INFORMS}
}

@article{aolaritei2025stochastic,
  title={Stochastic Optimization with Optimal Importance Sampling},
  author={Aolaritei, Liviu and Van Parys, Bart PG and Lam, Henry and Jordan, Michael I},
  journal={arXiv preprint arXiv:2504.03560},
  year={2025}
}

@article{pan2020adaptive,
  title={Adaptive importance sampling for extreme quantile estimation with stochastic black box computer models},
  author={Pan, Qiyun and Byon, Eunshin and Ko, Young Myoung and Lam, Henry},
  journal={Naval Research Logistics (NRL)},
  volume={67},
  number={7},
  pages={524--547},
  year={2020},
  publisher={Wiley Online Library}
}

@article{vihola2018unbiased,
  title={Unbiased estimators and multilevel Monte Carlo},
  author={Vihola, Matti},
  journal={Operations Research},
  volume={66},
  number={2},
  pages={448--462},
  year={2018},
  publisher={INFORMS}
}

@article{rosenbaum2017multilevel,
  title={Multilevel monte carlo metamodeling},
  author={Rosenbaum, Imry and Staum, Jeremy},
  journal={Operations Research},
  volume={65},
  number={4},
  pages={1062--1077},
  year={2017},
  publisher={INFORMS}
}

@article{giles2008multilevel,
  title={Multilevel monte carlo path simulation},
  author={Giles, Michael B},
  journal={Operations research},
  volume={56},
  number={3},
  pages={607--617},
  year={2008},
  publisher={INFORMS}
}

@article{vu2018random,
  title={Random projections for linear programming},
  author={Vu, Ky and Poirion, Pierre-Louis and Liberti, Leo},
  journal={Mathematics of Operations Research},
  volume={43},
  number={4},
  pages={1051--1071},
  year={2018},
  publisher={INFORMS}
}

@article{grishchenko2021proximal,
  title={Proximal gradient methods with adaptive subspace sampling},
  author={Grishchenko, Dmitry and Iutzeler, Franck and Malick, J{\'e}r{\^o}me},
  journal={Mathematics of Operations Research},
  volume={46},
  number={4},
  pages={1303--1323},
  year={2021},
  publisher={INFORMS}
}

@article{gutman2023coordinate,
  title={Coordinate descent without coordinates: Tangent subspace descent on Riemannian manifolds},
  author={Gutman, David H and Ho-Nguyen, Nam},
  journal={Mathematics of Operations Research},
  volume={48},
  number={1},
  pages={127--159},
  year={2023},
  publisher={INFORMS}
}

@article{harold1997stochastic,
  title={Stochastic approximation and recursive algorithm and applications},
  author={Kushner, Harold J and Yin, George G},
  journal={Application of Mathematics},
  volume={35},
  number={10},
  year={1997},
  publisher={Springer New York}
}

@article{ghadimi2013stochastic,
  title={Stochastic first-and zeroth-order methods for nonconvex stochastic programming},
  author={Ghadimi, Saeed and Lan, Guanghui},
  journal={SIAM journal on optimization},
  volume={23},
  number={4},
  pages={2341--2368},
  year={2013},
  publisher={SIAM}
}

@article{spall1997one,
  title={A one-measurement form of simultaneous perturbation stochastic approximation},
  author={Spall, James C},
  journal={Automatica},
  volume={33},
  number={1},
  pages={109--112},
  year={1997},
  publisher={Elsevier}
}

@article{spall1992multivariate,
  title={Multivariate stochastic approximation using a simultaneous perturbation gradient approximation},
  author={Spall, James C},
  journal={IEEE transactions on automatic control},
  volume={37},
  number={3},
  pages={332--341},
  year={1992},
  publisher={IEEE}
}

@article{hu2025convergence,
  title={On the convergence rate of stochastic approximation for gradient-based stochastic optimization},
  author={Hu, Jiaqiao and Fu, Michael C},
  journal={Operations research},
  volume={73},
  number={2},
  pages={1143--1150},
  year={2025},
  publisher={INFORMS}
}

@article{huang2025orlm,
  title={Orlm: A customizable framework in training large models for automated optimization modeling},
  author={Huang, Chenyu and Tang, Zhengyang and Hu, Shixi and Jiang, Ruoqing and Zheng, Xin and Ge, Dongdong and Wang, Benyou and Wang, Zizhuo},
  journal={Operations Research},
  year={2025},
  publisher={INFORMS}
}

@article{greensmith2004variance,
  title={Variance reduction techniques for gradient estimates in reinforcement learning},
  author={Greensmith, Evan and Bartlett, Peter L and Baxter, Jonathan},
  journal={Journal of Machine Learning Research},
  volume={5},
  number={Nov},
  pages={1471--1530},
  year={2004}
}

@article{ye2025unified,
  title={A unified zeroth-order optimization framework via oblivious randomized sketching},
  author={Ye, Haishan and Chang, Xiangyu and Chen, Xi},
  journal={arXiv preprint arXiv:2510.10945},
  year={2025}
}

@book{chikuse2003statistics,
  title={Statistics on special manifolds},
  author={Chikuse, Yasuko},
  volume={174},
  year={2003},
  publisher={Springer Science \& Business Media}
}

@article{stewart1980efficient,
  title={The efficient generation of random orthogonal matrices with an application to condition estimators},
  author={Stewart, Gilbert W},
  journal={SIAM Journal on Numerical Analysis},
  volume={17},
  number={3},
  pages={403--409},
  year={1980},
  publisher={SIAM}
}

@inproceedings{chungunbalanced,
  title={Unbalanced Optimal Total Variation Transport: A Theoretical Approach to Spatial Resource Allocation Problems},
  author={Chung, Nhan-Phu and Han, Jinhui and Li, Bohan and Li, Zehao},
  booktitle={The Thirty-ninth Annual Conference on Neural Information Processing Systems},
  year={2025}
}

@article{zhang2023adalora,
  title={Adalora: Adaptive budget allocation for parameter-efficient fine-tuning},
  author={Zhang, Qingru and Chen, Minshuo and Bukharin, Alexander and Karampatziakis, Nikos and He, Pengcheng and Cheng, Yu and Chen, Weizhu and Zhao, Tuo},
  journal={arXiv preprint arXiv:2303.10512},
  year={2023}
}

@article{dettmers2023qlora,
  title={Qlora: Efficient finetuning of quantized llms},
  author={Dettmers, Tim and Pagnoni, Artidoro and Holtzman, Ari and Zettlemoyer, Luke},
  journal={Advances in neural information processing systems},
  volume={36},
  pages={10088--10115},
  year={2023}
}

@article{xu2026parameter,
  title={Parameter-efficient fine-tuning methods for pretrained language models: A critical review and assessment},
  author={Xu, Lingling and Xie, Haoran and Qin, S Joe and Tao, Xiaohui and Wang, Fu Lee},
  journal={IEEE Transactions on Pattern Analysis and Machine Intelligence},
  year={2026},
  publisher={IEEE}
}

@article{ye2025lola,
  title={Lola: Llm-assisted online learning algorithm for content experiments},
  author={Ye, Zikun and Yoganarasimhan, Hema and Zheng, Yufeng},
  journal={Marketing Science},
  volume={44},
  number={5},
  pages={995--1016},
  year={2025},
  publisher={INFORMS}
}

@article{scheinberg2022finite,
  title={Finite difference gradient approximation: To randomize or not?},
  author={Scheinberg, Katya},
  journal={INFORMS Journal on Computing},
  volume={34},
  number={5},
  pages={2384--2388},
  year={2022},
  publisher={INFORMS}
}

@article{lin2025reusing,
  title={Reusing Historical Trajectories in Natural Policy Gradient via Importance Sampling: Convergence and Convergence Rate},
  author={Lin, Yifan and Wang, Yuhao and Zhou, Enlu},
  journal={Operations Research},
  volume={73},
  number={6},
  pages={3010--3026},
  year={2025},
  publisher={INFORMS}
}

@inproceedings{chen2020unbiased,
  title={Unbiased gradient simulation for zeroth-order optimization},
  author={Chen, Guanting},
  booktitle={2020 Winter Simulation Conference (WSC)},
  pages={2947--2959},
  year={2020},
  organization={IEEE}
}

@article{liu2024uncertainty,
  title={Uncertainty estimation and quantification for llms: A simple supervised approach},
  author={Liu, Linyu and Pan, Yu and Li, Xiaocheng and Chen, Guanting},
  journal={arXiv preprint arXiv:2404.15993},
  year={2024}
}

@article{dai2025assured,
  title={Assured Autonomy: How Operations Research Powers and Orchestrates Generative AI Systems},
  author={Dai, Tinglong and Simchi-Levi, David and Wu, Michelle Xiao and Xie, Yao},
  journal={arXiv preprint arXiv:2512.23978},
  year={2025}
}

@inproceedings{lu2018accelerating,
  title={Accelerating greedy coordinate descent methods},
  author={Lu, Haihao and Freund, Robert and Mirrokni, Vahab},
  booktitle={International Conference on Machine Learning},
  pages={3257--3266},
  year={2018},
  organization={PMLR}
}

@article{ding2025new,
  title={New understandings and computation on augmented lagrangian methods for low-rank semidefinite programming},
  author={Ding, Lijun and Lu, Haihao and Yang, Jinwen},
  journal={arXiv preprint arXiv:2505.15775},
  year={2025}
}

@article{liang2026llm,
  title={LLM for large-scale optimization model auto-formulation: A lightweight few-shot learning approach},
  author={Liang, Kuo and Lu, Yuhang and Mao, Jianming and Sun, Shuyi and Yang, Chunwei and Zeng, Congcong and Jin, Xiao and Qin, Hanzhang and Zhu, Ruihao and Teo, Chung-Piaw},
  journal={arXiv preprint arXiv:2601.09635},
  year={2026}
}

@article{wu2022joint,
  title={Joint Optimization and Statistical Inference for Zero-th Order Simulation Optimization},
  author={Wu, Yuhang and Zheng, Zeyu and Wang, Yingfei and Zhang, Guangyu and Zhang, Zuohua and Wang, Chu},
  journal={arXiv e-prints},
  pages={arXiv--2210},
  year={2022}
}

@article{li2025agentgit,
  title={AgentGit: A Version Control Framework for Reliable and Scalable LLM-Powered Multi-Agent Systems},
  author={Li, Yang and Ping, Siqi and Chen, Xiyu and Qi, Xiaojian and Wang, Zigan and Luo, Ye and Zhang, Xiaowei},
  journal={arXiv preprint arXiv:2511.00628},
  year={2025}
}

@article{che2026stochastic,
  title={Stochastic gradient descent with adaptive data},
  author={Che, Ethan and Dong, Jing and Tong, Xin T},
  journal={Operations Research},
  year={2026},
  publisher={INFORMS}
}

@inproceedings{li2025new,
  title={A new stochastic approximation method for gradient-based simulated parameter estimation},
  author={Li, Zehao and Peng, Yijie},
  booktitle={2025 Winter Simulation Conference (WSC)},
  pages={3298--3309},
  year={2025},
  organization={IEEE}
}

@article{hu2023contextual,
  title={Contextual stochastic bilevel optimization},
  author={Hu, Yifan and Wang, Jie and Xie, Yao and Krause, Andreas and Kuhn, Daniel},
  journal={Advances in Neural Information Processing Systems},
  volume={36},
  pages={78412--78434},
  year={2023}
}

@article{wang2025sinkhorn,
  title={Sinkhorn distributionally robust optimization},
  author={Wang, Jie and Gao, Rui and Xie, Yao},
  journal={Operations Research},
  year={2025},
  publisher={Informs}
}

@article{fan2020distributionally,
  title={Distributionally robust selection of the best},
  author={Fan, Weiwei and Hong, L Jeff and Zhang, Xiaowei},
  journal={Management Science},
  volume={66},
  number={1},
  pages={190--208},
  year={2020},
  publisher={INFORMS}
}

@article{wang2024understanding,
  title={Understanding the training and generalization of pretrained transformer for sequential decision making},
  author={Wang, Hanzhao and Pan, Yu and Sun, Fupeng and Liu, Shang and Talluri, Kalyan and Chen, Guanting and Li, Xiaocheng},
  journal={arXiv preprint arXiv:2405.14219},
  year={2024}
}
\section{Appendix}\label{app:proof}

\proof{Proof of Theorem \ref{thm:unbiasedness}.}
For LowRank-IPA, fix $\xi$ and consider the map $B\mapsto F(\xi,\Theta+BV^\top)$, where $B\in\mathbb R^{m\times r}$.
By the chain rule for matrix derivatives,
\begin{equation}\label{eq:chain_rule_ipa}
\nabla_B F(\xi,\Theta+BV^\top)
=
\nabla_\Theta F(\xi,\Theta+BV^\top)\,V,
\end{equation}
where $\nabla_\Theta F(\xi,\cdot)\in\mathbb R^{m\times n}$ and the product is matrix multiplication.
Evaluating \eqref{eq:chain_rule_ipa} at $B=0$ gives
\begin{equation*}
\nabla_B F(\xi,\Theta+BV^\top)\big|_{B=0}
=
\nabla_\Theta F(\xi,\Theta)\,V.
\end{equation*}
Therefore,
\begin{equation*}
\hat g_{\mathrm{LowRank\text{-}IPA}}(\xi,V,\Theta)
=
\big(\nabla_\Theta F(\xi,\Theta)\,V\big)V^\top
=
\nabla_\Theta F(\xi,\Theta)\,P.
\end{equation*}
Taking expectation and using independence of $\xi$ and $V$, we have
\[
\mathbb E_{\xi,V}\big[\hat g_{\mathrm{LowRank\text{-}IPA}}(\xi,V,\Theta)\big]
=
\mathbb E_{\xi}\Big[\nabla_\Theta F(\xi,\Theta)\,\mathbb E_V[P]\Big]
=
\mathbb E_{\xi}\big[\nabla_\Theta F(\xi,\Theta)\big]\,(cI_n)
=
c\, g(\Theta),
\]
where the interchange $g(\Theta)=\mathbb E[\nabla_\Theta F]$ is justified by the dominated convergence theorem. When $c=1$, it reduces to strong
unbiasedness.

For LowRank-LR, consider the differentiable map $B\mapsto \log p(\xi;\Theta+BV^\top)$.
By the chain rule,
\begin{equation*}\label{eq:chain_rule_lr}
\nabla_B \log p(\xi;\Theta+BV^\top)
=
\nabla_\Theta \log p(\xi;\Theta+BV^\top)\,V.
\end{equation*}
Evaluating at $B=0$ yields
\begin{equation*}\label{eq:lr_form}
\nabla_B \log p(\xi;\Theta+BV^\top)\big|_{B=0}
=
\nabla_\Theta \log p(\xi;\Theta)\,V.
\end{equation*}
Hence the estimator can be rewritten as
\begin{equation*}\label{eq:lr_proj}
\hat g_{\mathrm{LowRank\text{-}LR}}(\xi,V,\Theta)
=
F(\xi)\,\big(\nabla_\Theta \log p(\xi;\Theta)\,V\big)V^\top
=
F(\xi)\,\nabla_\Theta \log p(\xi;\Theta)\,P.
\end{equation*}
Taking expectation and using independence of $\xi$ and $V$ (so $\mathbb E_V[P]=cI_n$ can be pulled out),
\[
\mathbb E_{\xi,V}\big[\hat g_{\mathrm{LowRank\text{-}LR}}(\xi,V,\Theta)\big]
=
\mathbb E_{\xi}\Big[F(\xi)\,\nabla_\Theta \log p(\xi;\Theta)\,\mathbb E_V[P]\Big]
=
c\,\mathbb E_{\xi}\big[F(\xi)\,\nabla_\Theta \log p(\xi;\Theta)\big].
\]
By the assumption that the standard interchange condition that allows differentiation under the integral sign holds, we have
\[
\nabla_\Theta \mathbb E_{\xi\sim p(\cdot;\Theta)}[F(\xi)]=g(\Theta)
=\nabla_\Theta\int F(\xi)\,p(\xi;\Theta)\,d\xi
=\int F(\xi)\,\nabla_\Theta p(\xi;\Theta)\,d\xi.
\]
Using $\nabla_\Theta p(\xi;\Theta)=p(\xi;\Theta)\,\nabla_\Theta\log p(\xi;\Theta)$,
we obtain
\[
 g(\Theta)
=\int F(\xi)\,p(\xi;\Theta)\,\nabla_\Theta\log p(\xi;\Theta)\,d\xi
=\mathbb E_{\xi\sim p(\cdot;\Theta)}\big[F(\xi)\,\nabla_\Theta\log p(\xi;\Theta)\big].
\]
Therefore the last expectation equals $g(\Theta)$, and hence
\[
\mathbb E_{\xi,V}\big[\hat g_{\mathrm{LowRank\text{-}LR}}(\xi,V,\Theta)\big]
=
c\, g(\Theta).
\]
When $c=1$, this becomes strong unbiasedness. \Halmos
\endproof

\medskip
\proof{Proof of Proposition~\ref{propMSE}.}
Let $\mu:=g(\Theta)$ and write $\delta(\xi):=\hat g_{\mathrm{IPA/LR}}-\mu$ so that
$\hat g_{\mathrm{LowRank\text{-}IPA/LR}}-\mu=\delta(\xi)P+\mu(P-I_n)$. Expanding the Frobenius norm and using
$\|A\|_F^2=\text{tr}(A^\top A)$ gives
\[
\mathbb{E}_{V,\xi}\|\hat g_{\mathrm{LowRank\text{-}IPA/LR}}-\mu\|_F^2
=\mathbb{E}_{V,\xi}\|\delta(\xi)P\|_F^2
+\mathbb{E}_{V}\|\mu(P-I_n)\|_F^2
+2\,\mathbb{E}_{V,\xi}\text{tr}\!\big(P^\top\delta(\xi)^\top\mu(P-I_n)\big).
\]
The cross term vanishes because $\mathbb{E}_{\xi}\delta(\xi)=0$ (unbiasedness of $\hat g_{\mathrm{IPA/LR}}$),
and $P$ is independent of $\xi$, hence
$\mathbb{E}_{V,\xi}\text{tr}(P^\top\delta(\xi)^\top\mu(P-I_n))
=\mathbb{E}_{V}\text{tr}\!\big(P^\top (\mathbb{E}_{\xi}\delta(\xi))^\top \mu(P-I_n)\big)=0$.

For the first term, by trace cyclicity and $P^\top=P$,
\[
\mathbb{E}_{V,\xi}\|\delta(\xi)P\|_F^2
=\mathbb{E}_{V,\xi}\text{tr}\!\big(P\,\delta(\xi)^\top\delta(\xi)\,P\big)
=\text{tr}\!\Big(\underbrace{\mathbb{E}_{\xi}[\delta(\xi)^\top\delta(\xi)]}_{\Sigma_\xi}\,
\mathbb{E}_{V}[P^2]\Big).
\]
For the second term, we have
\[
\mathbb{E}_{V}\|\mu(P-I_n)\|_F^2
=\mathbb{E}_{V}\text{tr}\!\big((P-I_n)\mu^\top\mu(P-I_n)\big)
=\text{tr}\!\Big(\underbrace{\mu^\top\mu}_{\Sigma_\Theta}\,\mathbb{E}_{V}[(P-I_n)^2]\Big)
=\text{tr}\!\Big(\Sigma_\Theta\,\mathbb{E}_{V}[P^2-2P+I_n]\Big).
\]
Combining the two parts yields
\[
\mathbb{E}_{V,\xi}\|\hat g_{\mathrm{LowRank\text{-}IPA/LR}}-g(\Theta)\|_F^2
=\text{tr}\!\big(\Sigma_\xi\,\mathbb{E}[P^2]\big)
+\text{tr}\!\big(\Sigma_\Theta\,\mathbb{E}[P^2-2P+I_n]\big).
\]
Finally, using $\mathbb{E}_V[P]=cI_n$ we rewrite
$\mathbb{E}[P^2-2P+I_n]=\mathbb{E}[P^2-c^2I_n]+(1-c)^2I_n$, which gives
\[
\text{MSE}
=\text{tr}\!\big(\Sigma_\xi\,\mathbb{E}[P^2]\big)
+\text{tr}\!\big(\Sigma_\Theta\,\mathbb{E}[P^2-c^2I_n]\big)
+(1-c)^2\text{tr}(\Sigma_\Theta),
\]
as claimed. \Halmos
\endproof

\medskip
\proof{Proof of Theorem \ref{thm2}.}
  Let $V\in\mathbb R^{n\times r}$ and define the random projector $P:=VV^\top\succeq 0$.
Then $\mathrm{rank}(P)\le r$ almost surely. Let $\lambda_1(P)\ge \cdots \ge \lambda_n(P)\ge 0$
be the eigenvalues of $P$, and note that $\lambda_{r+1}(P)=\cdots=\lambda_n(P)=0$.
Hence
\[
\operatorname{tr}(P)=\sum_{i=1}^r \lambda_i(P),
\qquad
\operatorname{tr}(P^2)=\sum_{i=1}^r \lambda_i(P)^2.
\]

By Cauchy--Schwarz,
\[
\sum_{i=1}^r \lambda_i(P)^2
\;\ge\;
\frac{1}{r}\Big(\sum_{i=1}^r \lambda_i(P)\Big)^2
=
\frac{1}{r}\,\operatorname{tr}(P)^2.
\]
Equivalently,
\begin{equation}\label{eq:rank_trace}
\operatorname{tr}(P^2)\ \ge\ \frac{1}{r}\,\operatorname{tr}(P)^2.
\end{equation}
Moreover, equality in \eqref{eq:rank_trace} holds if and only if
$\lambda_1(P)=\cdots=\lambda_r(P)$, which forces $\mathrm{rank}(P)=r$ and all $r$ nonzero eigenvalues
are identical.

Taking expectation in \eqref{eq:rank_trace} and using Jensen's inequality gives
\[
\mathbb E[\operatorname{tr}(P^2)]
\ \ge\
\frac{1}{r}\,\mathbb E[\operatorname{tr}(P)^2]
\ \ge\
\frac{1}{r}\,\big(\mathbb E[\operatorname{tr}(P)]\big)^2.
\]
The second inequality is tight if and only if $\operatorname{tr}(P)$ is constant almost surely.

Since $\mathbb E[VV^\top]=\mathbb E[P]=cI_n$, we have
\[
\mathbb E[\operatorname{tr}(P)]
=
\operatorname{tr}(\mathbb E[P])
=
\operatorname{tr}(cI_n)
=
cn.
\]
Therefore,
\[
\operatorname{tr}(\mathbb E[P^2])
=
\mathbb E[\operatorname{tr}(P^2)]
\ \ge\
\frac{1}{r}(cn)^2
=
\frac{n^2c^2}{r},
\]
which proves the lower bound.

If the equality holds, then both inequalities above must be tight, hence:
(i) $\operatorname{tr}(P)$ is constant almost surely and equals $cn$;
(ii) $\lambda_1(P)=\cdots=\lambda_r(P)$ almost surely.
Combining (i)--(ii) yields
\[
\lambda_1(P)=\cdots=\lambda_r(P)=\frac{cn}{r}
\quad\text{and}\quad
\lambda_{r+1}(P)=\cdots=\lambda_n(P)=0
\qquad\text{a.s.}
\]
Thus, all nonzero singular values of $V$ equal $\sqrt{cn/r}$, and therefore
\[
V^\top V \;=\; \frac{cn}{r}\,I_r
\qquad\text{a.s.}
\]
Conversely, if $V^\top V=\frac{cn}{r}I_r$ almost surely, then $V$ has singular values
$\sqrt{cn/r}$ (multiplicity $r$), so $P=VV^\top$ has eigenvalues $\frac{cn}{r}$ (multiplicity $r$)
and $0$ (multiplicity $n-r$), which makes both inequalities tight and achieves
$\operatorname{tr}(\mathbb E[P^2])=\frac{n^2c^2}{r}$.

This completes the proof.
 \Halmos
\endproof

\medskip

\proof{Proof of Proposition \ref{prop:isotropic}.}
We verify the above conditions for each construction.

(i) Haar–Stiefel case.
Since $U^{\top}U=I_r$ by definition of the Stiefel manifold,
\[
      \bigl(V^{(S)}\bigr)^{\!\top}V^{(S)}
      =\alpha^{2}U^{\top}U
      =\alpha^{2}I_r
      =\frac{c\,n}{r}\,I_r.
\]
Then, due to the invariance of the Haar measure under orthogonal transformation, $U$ has the same distribution as $PU$ for any orthogonal matrix $P$. Hence, 
\begin{equation*}
    P\mathbb{E}[UU^{\top}]=P\mathbb{E}[UU^{\top}]P^{\top}P = \mathbb{E}[PU(PU)^{\top}]P = \mathbb{E}[UU^{\top}]P,
\end{equation*}
which means $\mathbb{E}[UU^{\top}]$ must commute with every orthogonal matrix, hence is a scalar multiple of $I_n$. Taking traces yields that multiple: $\text{tr}(\mathbb{E}[UU^{\top}]) = \mathbb{E}[\text{tr}(U^{\top}U)] = \text{tr}(I_r) = r$, so $\mathbb{E}[UU^{\top}]=\frac{r}{n}I_n$. Therefore, 
\begin{equation*}
    \mathbb{E}\!\bigl[P^{(S)}\bigr]
      =\alpha^{2}\,\mathbb{E}[UU^{\top}]
      =\frac{c\,n}{r}\cdot\frac{r}{n}\,I_n
      =c\,I_n.
\end{equation*}

(ii) Coordinate–axis case.
For every realization $J$, $U^{(C)}$ has orthonormal columns, hence
\[
      \bigl(V^{(C)}\bigr)^{\!\top}V^{(C)}
      =\alpha^{2}I_r
      =\frac{c\,n}{r}\,I_r.
\]
To compute $\mathbb{E}[U^{(C)}U^{(C)\!\top}]$, note that each coordinate
appears in $J$ with probability $r/n$, so $\mathbb{E}[U^{(C)}U^{(C)\!\top}]
      =\frac{r}{n}I_n.$
Multiplying by $\alpha^{2}$ yields the results exactly as above.\Halmos
\endproof

\medskip

\proof{Proof of Theorem \ref{thm:SigmaP2}.}

Let $\{q_i\}_{i=1}^n$ be the known orthonormal eigenbasis of $\Sigma$, so that
\[
\Sigma=Q\,\mathrm{diag}(\sigma_1,\ldots,\sigma_n)\,Q^\top,
\qquad
\sigma_1\ge \cdots \ge \sigma_n\ge 0.
\]
Take any feasible random matrix $P$ for \eqref{instanceind}. By feasibility, $P=VV^\top$ almost surely, hence $P\succeq 0$ almost surely, and $\operatorname{rank}(P)\le r$ almost surely. For each $i=1,\ldots,n$, define
\[
d_i:=q_i^\top Pq_i,
\qquad
s_i:=q_i^\top P^2q_i.
\]
Since $P\succeq 0$, we have $d_i\ge 0$. Also,
\[
s_i=q_i^\top P^2q_i=(Pq_i)^\top(Pq_i)=\|Pq_i\|^2\ge 0.
\]
In particular, if $s_i=0$, then $Pq_i=0$, and therefore
$d_i=q_i^\top Pq_i=\langle q_i,Pq_i\rangle=0.$
Thus the quantity
\[
\rho_i:=
\begin{cases}
d_i^2/s_i, & s_i>0,\\[0.3ex]
0, & s_i=0,
\end{cases}
\]
is well defined.

We first show that $0\le \rho_i\le 1$ almost surely. The lower bound is immediate from the definition. For the upper bound, note that
$d_i=q_i^\top Pq_i=\langle q_i,Pq_i\rangle=\langle Pq_i,q_i\rangle.$
Applying the Cauchy--Schwarz inequality, we obtain
\[
d_i^2=\langle Pq_i,q_i\rangle^2
\le \|Pq_i\|^2\,\|q_i\|^2.
\]
Because $\|q_i\|=1$ and $\|Pq_i\|^2=s_i$, this becomes
\[
d_i^2\le s_i.
\]
Hence $0\le \rho_i\le 1$ almost surely. Define
\[
\pi_i:=\mathbb E[\rho_i],\qquad i=1,\ldots,n.
\]
Then automatically $0\le \pi_i\le 1$.

Next we show that in fact $\pi_i>0$. Since $\mathbb E[P]=cI_n$, we have
\[
\mathbb E[d_i]
=
q_i^\top \mathbb E[P]q_i
=
q_i^\top(cI_n)q_i
=
c.
\]
Because $d_i\ge 0$ almost surely and its expectation is strictly positive, $d_i$ cannot vanish almost surely. Thus $\Pr(d_i>0)>0$. But if $d_i>0$, then necessarily $s_i>0$ and hence $\rho_i=d_i^2/s_i>0$. Therefore $\Pr(\rho_i>0)>0$, which implies
\[
0<\pi_i\le 1,\qquad i=1,\ldots,n.
\]

We now prove that the vector $\pi=(\pi_1,\ldots,\pi_n)$ also satisfies the budget constraint $\sum_{i=1}^n\pi_i\le r$. Fix one realization of $P$. Since $P\succeq 0$, it admits a spectral decomposition
\[
P=U\Lambda U^\top,
\]
where $U\in\mathbb R^{n\times k}$ has orthonormal columns, $\Lambda=\mathrm{diag}(\lambda_1,\ldots,\lambda_k)$ with $\lambda_j\ge 0$, and
\[
k=\operatorname{rank}(P)\le r.
\]
For each $i$, define $u_i:=U^\top q_i\in\mathbb R^k$. Then
\[
d_i=q_i^\top U\Lambda U^\top q_i=u_i^\top \Lambda u_i,
\qquad
s_i=q_i^\top U\Lambda^2U^\top q_i=u_i^\top \Lambda^2u_i.
\]

We claim that
\[
(u_i^\top \Lambda u_i)^2\le (u_i^\top \Lambda^2u_i)(u_i^\top u_i).
\]
This is again Cauchy--Schwarz, now applied in $\mathbb R^k$ to the vectors $\Lambda u_i$ and $u_i$. We have
\[
(u_i^\top \Lambda u_i)^2
=
\langle \Lambda u_i,u_i\rangle^2
\le
\|\Lambda u_i\|^2\,\|u_i\|^2.
\]
Since $\|\Lambda u_i\|^2=u_i^\top\Lambda^2u_i$ and $\|u_i\|^2=u_i^\top u_i$, the claimed inequality follows.

If $s_i>0$, dividing the above inequality by $s_i=u_i^\top \Lambda^2u_i$ gives
\[
\rho_i
=
\frac{(u_i^\top \Lambda u_i)^2}{u_i^\top \Lambda^2u_i}
\le
u_i^\top u_i
=
\|U^\top q_i\|^2.
\]
If $s_i=0$, then $\rho_i=0$, so the same inequality remains true. Therefore, for every realization of $P$,
\[
\sum_{i=1}^n \rho_i
\le
\sum_{i=1}^n \|U^\top q_i\|^2.
\]

We now simplify the right-hand side. Since $\|U^\top q_i\|^2=(U^\top q_i)^\top(U^\top q_i)=q_i^\top UU^\top q_i$, we have
\[
\sum_{i=1}^n \|U^\top q_i\|^2
=
\sum_{i=1}^n q_i^\top UU^\top q_i.
\]
Because $\{q_i\}_{i=1}^n$ is an orthonormal basis, for any square matrix $A$ one has the identity
\[
\sum_{i=1}^n q_i^\top A q_i=\operatorname{tr}(A).
\]
Applying this with $A=UU^\top$ yields
\[
\sum_{i=1}^n q_i^\top UU^\top q_i=\operatorname{tr}(UU^\top).
\]
Finally, $U$ has $k$ orthonormal columns, so $U^\top U=I_k$. Hence
$\operatorname{tr}(UU^\top)=\operatorname{tr}(U^\top U)=\operatorname{tr}(I_k)=k.$
Combining the previous displays, we obtain
$\sum_{i=1}^n \rho_i\le k\le r.$
Taking expectations on both sides gives
\[
\sum_{i=1}^n \pi_i
=
\sum_{i=1}^n \mathbb E[\rho_i]
\le r.
\]

We next derive a lower bound for the objective. Since
\[
\Phi(P)=\operatorname{tr}\!\bigl(\Sigma\,\mathbb E[P^2]\bigr),
\]
and $\Sigma$ is diagonal in the basis $\{q_i\}$, we may write
\[
\Phi(P)
=
\sum_{i=1}^n \sigma_i\, q_i^\top \mathbb E[P^2]q_i
=
\sum_{i=1}^n \sigma_i\,\mathbb E[q_i^\top P^2q_i]
=
\sum_{i=1}^n \sigma_i\,\mathbb E[s_i].
\]
Thus the objective is a weighted sum of the quantities $\mathbb E[s_i]$.

We now bound each $\mathbb E[s_i]$ from below. Since $\mathbb E[d_i]=c$, it is enough to compare $d_i$ and $s_i$. On the event $\{s_i>0\}$, we can write
$d_i=\sqrt{s_i}\cdot \frac{d_i}{\sqrt{s_i}}.$
On the event $\{s_i=0\}$, we already showed that $d_i=0$. Therefore
\[
d_i=\sqrt{s_i}\cdot \frac{d_i}{\sqrt{s_i}}\,\mathbf 1_{\{s_i>0\}}
\]
almost surely. Taking expectations and then applying Cauchy--Schwarz to the random variables $\sqrt{s_i}$ and $\frac{d_i}{\sqrt{s_i}}\mathbf 1_{\{s_i>0\}}$, we obtain
\[
(\mathbb E[d_i])^2
\le
\mathbb E[s_i]\,
\mathbb E\!\left[\frac{d_i^2}{s_i}\mathbf 1_{\{s_i>0\}}\right].
\]
By definition of $\rho_i$, the second expectation is exactly $\mathbb E[\rho_i]=\pi_i$. Since $\mathbb E[d_i]=c$, we conclude that
\[
c^2\le \mathbb E[s_i]\pi_i,
\qquad\text{that is,}\qquad
\mathbb E[s_i]\ge \frac{c^2}{\pi_i}.
\]

Substituting this estimate into the expression for $\Phi(P)$ gives a lower bound of the objective functional
\[
\Phi(P)
=
\sum_{i=1}^n \sigma_i\,\mathbb E[s_i]
\ge
\sum_{i=1}^n \sigma_i\,\frac{c^2}{\pi_i}
=
c^2\sum_{i=1}^n \frac{\sigma_i}{\pi_i}.
\]
Here the inequality holds term by term, because each coefficient $\sigma_i$ is nonnegative and each $\mathbb E[s_i]$ is bounded below by $c^2/\pi_i$. Since $0<\pi_i\le 1$ and $\sum_i\pi_i\le r$, every feasible $P$ satisfies
\[
\Phi(P)\ge
c^2\inf\Bigl\{
\sum_{i=1}^n \sigma_i/\pi_i:\ 0<\pi_i\le 1,\ \sum_{i=1}^n\pi_i\le r
\Bigr\}.
\]

This reduces the original matrix optimization problem to a finite-dimensional convex program. For the purpose of minimizing the objective, one may take an optimizer with $\sum_{i=1}^n\pi_i=r$; indeed, if $\sum_i\pi_i<r$ and there exists an index with $\sigma_i>0$ and $\pi_i<1$, then increasing that coordinate preserves feasibility and strictly decreases the objective. Hence the relevant optimization problem is
\[
\min\Bigl\{\sum_{i=1}^n \sigma_i/\pi_i:\ 0<\pi_i\le 1,\ \sum_{i=1}^n\pi_i=r\Bigr\}.
\]
Its KKT conditions give
\[
\pi_i^\star=\min\{1,\sqrt{\sigma_i/\mu}\},\qquad i=1,\ldots,n,
\]
for some $\mu>0$ chosen so that $\sum_{i=1}^n\pi_i^\star=r$. If we define
\[
t:=\#\{i:\pi_i^\star=1\},
\]
then
\[
r=t+\sum_{j:\pi_j^\star<1}\sqrt{\sigma_j/\mu},
\]
which is equivalent to
\[
\sqrt{\mu}
=
\frac{\sum_{j:\pi_j^\star<1}\sqrt{\sigma_j}}{r-t}.
\]
Substituting this into the formula for $\pi_i^\star$ yields
\[
\pi_i^{\star}
=
\min\!\Bigl\{
1,\ 
(r-t)\sqrt{\sigma_i}\Big/\!\!\sum_{j:\pi_j^{\star}<1}\sqrt{\sigma_j}
\Bigr\},
\qquad i=1,\ldots,n.
\]
Evaluating the objective at $\pi^\star$, we obtain
\[
\Phi_{\min}
=
c^2\sum_{i=1}^n \frac{\sigma_i}{\pi_i^\star}
=
c^{2}\!\sum_{i:\pi_i^{\star}=1}\sigma_i
+
c^2\sum_{i:\pi_i^{\star}<1}\frac{\sigma_i}{\pi_i^\star}.
\]
For indices with $\pi_i^\star<1$, we have $\pi_i^\star=\sqrt{\sigma_i/\mu}$, so
$\frac{\sigma_i}{\pi_i^\star}=\sqrt{\sigma_i\mu}.$
Therefore
\[
\Phi_{\min}
=
c^{2}\!\sum_{i:\pi_i^{\star}=1}\sigma_i
+
c^2\sum_{i:\pi_i^{\star}<1}\sqrt{\sigma_i\mu}.
\]
Using the expression for $\sqrt{\mu}$ found above, this becomes
\[
\Phi_{\min}
=
c^{2}\!\sum_{i:\pi_i^{\star}=1}\sigma_i
+
\frac{c^{2}}{r-t}
\Bigl(\sum_{i:\pi_i^{\star}<1}\sqrt{\sigma_i}\Bigr)^{2}.
\]
This proves \eqref{eq:Phi-star-general-v2} and \eqref{eq:pi-star-v2}.

It remains to show that this lower bound is attained. Let $J\subseteq[n]$ be a random subset such that $|J|=r$ almost surely and
\[
\Pr(i\in J)=\pi_i^\star,\qquad i=1,\ldots,n.
\]
Define
\[
P=\sum_{i\in J}\frac{c}{\pi_i^\star}\,q_iq_i^\top.
\]
Each matrix $q_iq_i^\top$ is a rank-one orthogonal projector, and different terms are mutually orthogonal because $q_i^\top q_j=0$ whenever $i\neq j$. Since exactly $r$ indices are selected, $P$ is positive semi-definite and has rank exactly $r$ almost surely. In particular, it can be factorized as $P=VV^\top$ for some random matrix $V$, so it indeed defines an element of $\mathcal D$.

We now verify the moment conditions. By linearity of expectation,
\[
\mathbb E[P]
=
\sum_{i=1}^n \Pr(i\in J)\frac{c}{\pi_i^\star}\,q_iq_i^\top
=
\sum_{i=1}^n \pi_i^\star\frac{c}{\pi_i^\star}\,q_iq_i^\top
=
c\sum_{i=1}^n q_iq_i^\top
=
cI_n.
\]
Thus $P$ is feasible. To compute $P^2$, note that
\[
(q_iq_i^\top)(q_jq_j^\top)=0
\qquad\text{for }i\neq j,
\]
because $q_i^\top q_j=0$, and also
\[
(q_iq_i^\top)^2=q_iq_i^\top.
\]
Hence all cross terms vanish, and
\[
P^2
=
\sum_{i\in J}\Bigl(\frac{c}{\pi_i^\star}\Bigr)^2 q_iq_i^\top.
\]
Taking expectations once again,
\[
\mathbb E[P^2]
=
\sum_{i=1}^n \pi_i^\star\Bigl(\frac{c}{\pi_i^\star}\Bigr)^2 q_iq_i^\top
=
c^2\sum_{i=1}^n \frac{1}{\pi_i^\star}\,q_iq_i^\top.
\]
Multiplying by $Q^\top$ on the left and $Q$ on the right simply expresses this matrix in the eigenbasis of $\Sigma$, so
\[
\mathbb E[Q^\top P^2Q]
=
c^2\,\mathrm{diag}\!\Bigl(\frac{1}{\pi_1^\star},\ldots,\frac{1}{\pi_n^\star}\Bigr).
\]
Therefore
\[
\Phi(P)
=
\operatorname{tr}\!\bigl(\Sigma\,\mathbb E[P^2]\bigr)
=
c^2\sum_{i=1}^n \frac{\sigma_i}{\pi_i^\star}
=
\Phi_{\min}.
\]
So this construction is feasible and optimal in $\mathcal D$.

Finally, suppose a feasible distribution in $\mathcal D$ satisfies \eqref{eq:opt-cond-v2}. Then
\[
\Phi(P)
=
\operatorname{tr}\!\bigl(\Sigma\,\mathbb E[P^2]\bigr)
=
\operatorname{tr}\!\Bigl(
\mathrm{diag}(\sigma_1,\ldots,\sigma_n)\,\mathbb E[Q^\top P^2Q]
\Bigr)
=
c^2\sum_{i=1}^n \frac{\sigma_i}{\pi_i^\star}.
\]
The first part of the proof has already shown that no feasible distribution can achieve a value below $\Phi_{\min}$, while the right-hand side above is exactly the value of $\Phi_{\min}$ given by \eqref{eq:Phi-star-general-v2}. Hence any feasible distribution satisfying \eqref{eq:opt-cond-v2} is optimal. \Halmos
\endproof

\medskip
\proof{Proof of Proposition~\ref{prop:sampler_optimal}.}
Algorithm~\ref{alg:sigma_dependent_sampler} samples a subset $J\subset\{1,\ldots,n\}$ with $|J|=r$
and marginal inclusion probabilities $\Pr(i\in J)=\pi_i^\star$.
It then sets
\[
V = Q_J\,\mathrm{diag}\!\big(\sqrt{\mu_i}\big)_{i\in J},
\qquad
\mu_i:=\frac{c}{\pi_i^\star},
\qquad
P:=VV^\top
=
\sum_{i\in J}\mu_i\,q_i q_i^\top.
\]
Since $|J|=r$ and the vectors $\{q_i\}$ are orthonormal, $P\succeq 0$ and $\mathrm{rank}(P)=r$ almost surely.

\paragraph{First moment.}
Write $P=\sum_{i=1}^n \mathbf 1_{\{i\in J\}}\mu_i\,q_i q_i^\top$. Taking expectation gives
\[
\mathbb E[P]
=
\sum_{i=1}^n \mathbb E[\mathbf 1_{\{i\in J\}}]\mu_i\,q_i q_i^\top
=
\sum_{i=1}^n \pi_i^\star \frac{c}{\pi_i^\star}\,q_i q_i^\top
=
c\sum_{i=1}^n q_i q_i^\top
=
cI_n.
\]

\paragraph{Second moment in the eigenbasis.}
Because $q_i q_i^\top q_j q_j^\top =0$ for $i\neq j$ and $(q_i q_i^\top)^2=q_i q_i^\top$, we have
\[
P^2=\sum_{i\in J}\mu_i^2\,q_i q_i^\top,
\qquad
\mathbb E[P^2]
=
\sum_{i=1}^n \Pr(i\in J)\,\mu_i^2\,q_i q_i^\top
=
\sum_{i=1}^n \pi_i^\star\Big(\frac{c}{\pi_i^\star}\Big)^2 q_i q_i^\top
=
c^2\sum_{i=1}^n \frac{1}{\pi_i^\star}\,q_i q_i^\top.
\]
Multiplying by $Q^\top(\cdot)Q$ yields
\[
\mathbb E[Q^\top P^2 Q]
=
c^2\,\mathrm{diag}\!\Big(\frac{1}{\pi_1^\star},\ldots,\frac{1}{\pi_n^\star}\Big).
\]
Together with $\mathbb E[P]=cI_n$ and $|J|=r$, this shows that the law produced by the algorithm
satisfies \eqref{eq:opt-cond-v2}. Hence it attains $\Phi_{\min}$ in Theorem~\ref{thm:SigmaP2}.
\Halmos
\endproof

\medskip
\proof{Proof of Proposition \ref{thm3}.}
    Throughout the proof we assume strong unbiasedness, i.e., $c=1$.
Recall that $\Sigma=\Sigma_\xi+\Sigma_\Theta\succeq 0$ and that, under $c=1$,
the MSE decomposition gives
\begin{equation}\label{eq:MSE_c1}
\mathrm{MSE}_{\min}
=
\Phi_{\min} + (1-2c)\operatorname{tr}(\Sigma_\Theta)
=
\Phi_{\min}-\operatorname{tr}(\Sigma_\Theta).
\end{equation}
Therefore, it suffices to show $\Phi_{\min}\le \operatorname{tr}(\Sigma)$.

Let $\Sigma=Q\mathrm{diag}(\sigma_1,\ldots,\sigma_n)Q^\top$ with $\sigma_1\ge\cdots\ge\sigma_n\ge 0$
and let $s:=\mathrm{rank}(\Sigma)=\#\{i:\sigma_i>0\}$.
By assumption, $s\le r$.

Consider the optimal inclusion probabilities $\pi^\star$ in Theorem~\ref{thm:SigmaP2}.
Since directions with $\sigma_i=0$ do not contribute to the objective
$\Phi(P)=\operatorname{tr}(\Sigma\,\mathbb E[P^2])$, we may always use them to satisfy the fixed-size
constraint $\sum_{i=1}^n\pi_i=r$ without changing the objective value.
In particular, we can choose an optimal solution such that
\[
\pi_i^\star=1 \quad \text{for all } i\le s \text{ (i.e., for all }\sigma_i>0\text{)},
\]
and distribute the remaining probability mass $r-s$ arbitrarily over indices with $\sigma_i=0$ so that
$\sum_i\pi_i^\star=r$ (this is feasible because $s\le r$ and each $\pi_i^\star\le 1$).

With this choice, the ``uncapped'' set $\{i:\pi_i^\star<1\}$ is contained in $\{i:\sigma_i=0\}$.
Hence $\sum_{i:\pi_i^\star<1}\sqrt{\sigma_i}=0$, and the second term in the optimal value formula
\eqref{eq:Phi-star-general-v2} vanishes. Therefore,
\[
\Phi_{\min}
=
\sum_{i:\pi_i^\star=1}\sigma_i
=
\sum_{i=1}^{s}\sigma_i
=
\operatorname{tr}(\Sigma).
\]
Substituting into \eqref{eq:MSE_c1} yields
\[
\mathrm{MSE}_{\min}
=
\operatorname{tr}(\Sigma)-\operatorname{tr}(\Sigma_\Theta)
=
\operatorname{tr}(\Sigma_\xi+\Sigma_\Theta)-\operatorname{tr}(\Sigma_\Theta)
=
\operatorname{tr}(\Sigma_\xi).
\]
In particular, $\mathrm{MSE}_{\min}\le \operatorname{tr}(\Sigma_\xi)$, and the inequality holds with equality
when $\mathrm{rank}(\Sigma)\le r$.\Halmos
\endproof

\end{document}